\documentclass[12pt,draftclsnofoot,onecolumn]{IEEEtran}

\usepackage[utf8]{inputenc} 
\usepackage[T1]{fontenc}
\usepackage{url}
\usepackage{ifthen}
\usepackage[cmex10]{amsmath} 
\usepackage{cite}
\usepackage{bbm}
\usepackage{xcolor}
\usepackage{comment}

\usepackage{cleveref}

\usepackage{listings}
\definecolor{backcolour}{rgb}{0.95,0.95,0.92}
\usepackage{amssymb}
\usepackage{amsmath}
\usepackage{amsthm}
\usepackage [english]{babel}
\usepackage [autostyle, english = american]{csquotes}
\usepackage{mathtools}
\usepackage[hang,flushmargin]{footmisc}

\DeclareMathOperator*{\argmin}{arg\,min}
\MakeOuterQuote{"}

\usepackage[short]{optidef}

\usepackage{stfloats}
\usepackage{boldline,multirow}
\usepackage{tikz}
\usepackage{pgfplots}
\usetikzlibrary{positioning,quotes,angles,patterns,intersections,pgfplots.fillbetween}
\usepackage{tkz-euclide}

\usepackage{tcolorbox}
\definecolor{mycolor}{rgb}{0.122, 0.435, 0.698}
\makeatletter
\newcommand{\mybox}[1]{%
  \setbox0=\hbox{#1}%
  \setlength{\@tempdima}{\dimexpr\wd0+13pt}%
  \begin{tcolorbox}[colframe=mycolor,boxrule=0.5pt,arc=4pt,
      left=6pt,right=6pt,top=6pt,bottom=6pt,boxsep=0pt,width=\@tempdima]
    #1
  \end{tcolorbox}
}
\makeatother

\makeatletter
\newtheorem*{rep@theorem}{\rep@title}
\newcommand{\newreptheorem}[2]{%
\newenvironment{rep#1}[1]{%
 \def\rep@title{#2 \ref{##1}}%
 \begin{rep@theorem}}%
 {\end{rep@theorem}}}
\makeatother
\newtheorem{theorem}{Theorem}
\newreptheorem{theorem}{Theorem}
\newtheorem{lemma}{Lemma}
\newreptheorem{lemma}{Lemma}
\newtheorem{prop}{Proposition}
\newreptheorem{prop}{Proposition}
\newtheorem{corollary}{Corollary}
\newreptheorem{corollary}{Corollary}

\newtheorem{definition}{Definition}
\newreptheorem{definition}{Definition}
\newtheorem{remark}{Remark}

\DeclareMathOperator{\arcsinh}{arcsinh}

\usepackage{soul}

\usepackage{algorithm}
\usepackage[noend]{algpseudocode}
\usepackage{tikz}
\usepackage{pgfplots}
\usetikzlibrary{positioning,quotes,angles,patterns,intersections, pgfplots.fillbetween}

\usepackage{caption}
\usepackage{subcaption}
\newcolumntype{?}[1]{!{\vrule width #1}}

\pgfplotsset{compat=1.16}

\DeclareMathOperator{\sign}{sign}
\begin{document}

\title{A Tunable Loss Function for Robust Classification: Calibration, Landscape, and Generalization\footnote{This paper was presented in part at the $2019$ and $2020$ IEEE International Symposium on Information Theory~\cite{sypherd2019tunable,sypherd2020alpha}.}}
\author{%
  Tyler Sypherd*, Mario Diaz$\dagger$, John Kevin Cava*, \\ Gautam Dasarathy*, Peter Kairouz$\ddagger$, and Lalitha Sankar*  \\
*Arizona State University, \texttt{\{tsypherd,jcava,gautamd,lsankar\}@asu.edu} \\
$\dagger$ Universidad Nacional Aut\'{o}noma de M\'{e}xico, \texttt{mario.diaz@sigma.iimas.unam.mx} \\
$\ddagger$ Google AI, \texttt{kairouz@google.com}\\
}

\date{}
\maketitle


\begin{abstract}
We introduce a tunable loss function called $\alpha$-loss, parameterized by $\alpha \in (0,\infty]$, which interpolates between the exponential loss ($\alpha = 1/2$), the log-loss ($\alpha = 1$), and the 0-1 loss ($\alpha = \infty$),
for the machine learning setting of classification.
Theoretically, we illustrate a fundamental connection between $\alpha$-loss and Arimoto conditional entropy, verify the classification-calibration of $\alpha$-loss in order to demonstrate asymptotic optimality via Rademacher complexity generalization techniques, and build-upon a notion called strictly local quasi-convexity in order to quantitatively characterize the optimization landscape of $\alpha$-loss.
Practically, we perform class imbalance, robustness, and classification experiments on benchmark image datasets using convolutional-neural-networks. 
Our main practical conclusion is that certain tasks may benefit from tuning $\alpha$-loss away from log-loss ($\alpha = 1$), and to this end we provide simple heuristics for the practitioner.
In particular, navigating the $\alpha$ hyperparameter can readily provide superior model robustness to label flips ($\alpha > 1$) and sensitivity to imbalanced classes ($\alpha < 1$).
\end{abstract}
\newpage

\section{Introduction}
In the context of machine learning, the performance of a classification algorithm, in terms of accuracy, tractability, and convergence guarantees crucially depends on the choice of the loss function during training~\cite{friedman2001elements,shalev2014understanding}. Consider a feature vector $X \in \mathcal{X}$, an unknown finite-valued label $Y \in \mathcal{Y}$, and a hypothesis $h:\mathcal{X} \rightarrow \mathcal{Y} $. The canonical $0$-$1$ loss, given by $\mathbbm{1}[h(X) \neq Y]$, is considered an ideal loss function in the classification setting that captures the probability of incorrectly guessing the true label $Y$ using $h(X)$. However, since the $0$-$1$ loss is neither continuous nor differentiable, its applicability in state-of-the-art learning algorithms is highly restricted \cite{ben2003difficulty}.
As a consequence, \textit{surrogate} loss functions that approximate the $0$-$1$ loss such as log-loss, exponential loss, sigmoid loss, etc. have garnered much interest \cite{bartlett2006convexity,masnadi2009design,lin2004note,nguyen2009,rosasco2004loss,nguyen2013algorithms,singh2010loss,tewari2007consistency,zhao2010convex,barron2019general,lin2017focal}. 

In the field of information-theoretic privacy, Liao \textit{et al.} recently introduced a tunable loss function called $\alpha$-loss for $\alpha \in [1,\infty]$ to model the inferential capacity of an adversary to obtain private attributes~\cite{liao2018tunable,liao2019robustness,liao2020maximal}.
For $\alpha = 1$, $\alpha$-loss reduces to log-loss which models a belief-refining adversary; for $\alpha = \infty$, $\alpha$-loss reduces to the probability of error which models an adversary that makes hard decisions. 
Using $\alpha$-loss, Liao \textit{et al.} in~\cite{liao2018tunable} derived a new privacy measure called \textit{$\alpha$-leakage} which continuously interpolates between Shannon's mutual information~\cite{shannon2001mathematical} and maximal leakage introduced by Isaa \textit{et al.}~\cite{issa2019operational}; indeed, Liao \textit{et al.} showed that $\alpha$-leakage is equivalent to the Arimoto mutual information~\cite{alphamutualinfo}.
%
In this paper, we extend $\alpha$-loss to the range $\alpha \in (0,\infty]$ and propose it as a tunable \textit{surrogate} loss function for the ideal 0-1 loss in the machine learning setting of classification. 
Through our extensive analysis, we argue that: 1)~since $\alpha$-loss continuously interpolates between the exponential ($\alpha = 1/2$), log ($\alpha = 1$), and $0$-$1$ ($\alpha = \infty$) losses and is intimately related to the Arimoto conditional entropy, it is theoretically an object of interest in its own right; 2)~navigating the convexity/robustness trade-offs inherent in the $\alpha$ hyperparameter offers significant practical improvements over log-loss, which is a canonical loss function in machine learning, and can be done quickly and effectively.


\subsection{Related Work}
The study and implementation of tunable utility (or loss) metrics which continuously interpolate between useful quantities is a persistent theme in information theory, networking, and machine learning. 
In information theory, Rényi entropy generalized the Shannon entropy~\cite{renyi1961}, and Arimoto extended the Rényi entropy to conditional distributions~\cite{ARIMOTO1971181}. 
This led to the $\alpha$-mutual information~\cite{alphamutualinfo,sason2017arimoto}, which is directly related to a recently introduced privacy measure called $\alpha$-leakage~\cite{liao2018tunable}. 
More recently in networking, Mo~\textit{et al.} introduced $\alpha$-fairness in~\cite{mo2000fair}, which is a tunable utility metric that alters the value of different edge users; similar ideas have recently been studied in the federated learning setting~\cite{li2019fair}.
Even more recently in machine learning, Barron in~\cite{barron2019general} presented a tunable extension of the $l_{2}$ loss function, which interpolates between several known $l_{2}$-type losses and has similar convexity$/$robustness themes as this work.
Presently, there is a need in the machine learning setting of \textit{classification} for alternative losses to the cross-entropy loss (one-hot encoded log-loss)~\cite{janocha2017loss}. 
We propose $\alpha$-loss, which continuously interpolates between the exponential, log, and $0$-$1$ losses, as a viable solution. 

In order to evaluate the statistical efficacy of loss functions in the learning setting of classification, Bartlett \textit{et al.} proposed the notion of \textit{classification-calibration} in a seminal paper~\cite{bartlett2006convexity}. 
\textit{Classification-calibration} is analogous to point-wise Fisher consistency in that it requires that the minimizer of the conditional expectation of a loss function agrees in sign with the Bayes predictor for every value of the feature vector.
A more restrictive notion called \textit{properness} requires that the minimizer of the conditional expectation of a loss function exactly replicates the true posterior~\cite{nock2020supervised,walder2020all,reid2010composite}.
\textit{Properness} of a loss function is a necessary condition for efficacy in the class probability estimation setting (see, e.g.,~\cite{reid2010composite}), but for the classification setting which is the focus of this work, the notion of \textit{classification-calibration} is sufficient. 
In the sequel, we find that the margin-based form of $\alpha$-loss is classification-calibrated for all $\alpha \in (0,\infty]$ and thus satisfies this necessary condition for efficacy in binary classification.

%
%
%

While early research was predominantly focused on convex losses~\cite{bartlett2006convexity,rosasco2004loss,nguyen2009,lin2004note}, more recent works propose the use of non-convex losses as a means to moderate the behavior of an algorithm~\cite{mei2018landscape,nguyen2013algorithms,masnadi2009design,barron2019general}. This is due to the increased robustness non-convex losses offer over convex losses \cite{long2010random,mei2018landscape,barron2019general} and the fact that modern learning models (e.g., deep learning) are inherently non-convex as they involve vast functional compositions~\cite{goodfellow2016deep}. 
There have been numerous theoretical attempts to capture the non-convexity of the optimization landscape
which is the loss surface induced by the learning model, underlying distribution, and the surrogate loss function itself~\cite{mei2018landscape,hazan2015beyond,li2018visualizing,nguyen2017loss,fu2018guaranteed,liang2018understanding,engstrom2019exploring,chaudhari2018deep}.
To this end, Hazan \textit{et al.}~\cite{hazan2015beyond} introduce the notion of \textit{strictly local quasi-convexity} (SLQC) to parametrically quantify approximately quasi-convex functions, and provide convergence guarantees for the Normalized Gradient Descent (NGD) algorithm (originally introduced in~\cite{nesterov1984minimization}) for such functions.
Through a quantification of the SLQC parameters of the expected $\alpha$-loss, we provide some estimates that strongly suggest that the degree of convexity increases as $\alpha$ decreases less than $1$ (log-loss); conversely, the degree of convexity decreases as $\alpha$ increases greater than $1$.
Thus, we find that there exists a trade-off inherent in the choice of $\alpha \in (0,\infty]$, i.e., trade convexity (and hence optimization speed) for robustness and vice-versa. 
Since increasing the degree of convexity of the optimization landscape is conducive to faster optimization, our approach could serve as an alternative to other approaches whose objective is to accelerate the optimization process, e.g., the activation function tuning in~\cite{benigni2019eigenvalue,xiao2018dynamical,pennington2017nonlinear} and references therein.

Understanding the generalization capabilities of learning algorithms stands as one of the key problems in theoretical machine learning. A classical approach to this problem consists in deriving algorithm independent generalization bounds, mainly relying on the notion of Rademacher complexity \cite[Ch.~26]{shalev2014understanding}. A recent line of research, initiated by the works of Russo and Zou \cite{russo2020how} and Xu and Raginsky \cite{xu2017information}, aims to improve generalization bounds by considering the statistical dependency between the input and the output of a given learning algorithm. While there are many extensions and refinements, e.g., \cite{lopez2018generalization,wang2019information,bu2020tightening,steinke2020reasoning,esposito2021generalization,galvez2021tighter,neu2021information}, these results are inherently algorithm dependent which makes them hard to instantiate and obfuscates the role of the loss function. Hence, in this work we rely on classical Rademacher complexity tools to provide algorithm independent generalization bounds that lead to the asymptotic optimality of $\alpha$-loss w.r.t.\ the 0-1 loss.

There are a few proposed tunable loss functions for the classification setting in the literature~\cite{wang2019symmetric,amid2019robust,nguyen2013algorithms,li2021tilted}. 
Notably, the symmetric cross entropy loss introduced by Wang \textit{et al.} in~\cite{wang2019symmetric} proposes the tunable linear combination of the usual cross entropy loss with the so-called reverse cross entropy loss, which essentially reverses the roles of the one-hot encoded labels and soft prediction of the model. 
Wang \textit{et al.} report gains under symmetric and asymmetric noisy labels, particularly in the very high noise regime. 
Another approach introduced by Amid \textit{et al.} in~\cite{amid2019robust} is a bi-tempered logistic loss, which is based on Bregman divergences. 
As the name suggests, the bi-tempered logistic loss depends on two temperature hyperparameters, which Amid \textit{et al.} show improvements over vanilla cross-entropy loss again on noisy data. 
Recently, Li \textit{et al.} introduced tilted empirical risk minimization~\cite{li2021tilted}, a framework which parametrically generalizes empirical risk minimization using a log-exponential transformation to induce fairness or robustness in the model.
Contrasting with this work, we note that our study is exclusively focused on $\alpha$-loss acting within empirical risk minimization.

Summing up, the main distinctions that differentiate this work from related work are that $\alpha$-loss has a fundamental relationship to the Arimoto conditional entropy, continuously interpolates between the exponential, log, and $0$-$1$ losses, and provides robustness to noisy labels \textit{and} sensitivity to imbalanced classes.
Lastly, we remark that $\alpha$-loss has also been recently investigated in the context of generative adversarial networks~\cite{kurri2021realizing} and boosting~\cite{nock2021properly}.

%
%

%

\subsection{Contributions}
%
%
%
The following are the main contributions of this paper: 
\begin{itemize}
    \item We formulate $\alpha$-loss in the classification setting, extending it to $\alpha \in (0,1)$, and we thereby extend the result of Liao \textit{et al.} in~\cite{liao2018tunable} which characterizes the relationship between $\alpha$-loss and the Arimoto conditional entropy.
    \item For binary classification, we define a margin-based form of $\alpha$-loss and demonstrate its equivalence to $\alpha$-loss for all $\alpha \in (0,\infty]$. 
    We then characterize convexity and verify statistical calibration of the margin-based $\alpha$-loss for $\alpha \in (0,\infty]$. 
    We next derive the minimum conditional risk of the margin-based $\alpha$-loss, which we show recovers the relationship between $\alpha$-loss and the Arimoto conditional entropy for all $\alpha \in (0,\infty]$. 
    Lastly, we provide synthetic experiments on a two-dimensional Gaussian mixture model with asymmetric label flips and class imbalances, where we train linear predictors with $\alpha$-loss for several values of $\alpha$. 
    \item For the logistic model in binary classification, we show that the expected $\alpha$-loss is convex in the logistic parameter for $\alpha \leq 1$ (strongly-convex when the covariance matrix is positive definite), and we show that it retains convexity as $\alpha$ increases greater than $1$ provided that the radius of the parameter space is small enough. We provide a point-wise extension of \textit{strictly local quasi-convexity} (SLQC) by Hazan \textit{et al.}, and we reformulate SLQC into a more tractable inequality using a geometric inequality which may be of independent interest. 
    Using a bootstrapping technique which also may be of independent interest, 
    we provide bounds in order to quantify the evolution of the SLQC parameters as $\alpha$ increases.
    \item Also for the logistic model in binary classification, we characterize the generalization capabilities of $\alpha$-loss. To this end, we employ standard Rademacher complexity generalization techniques to derive a uniform generalization bound for the logistic model trained with $\alpha$-loss for $\alpha \in (0,\infty]$. We then combine a result by Bartlett \textit{et al.} and our uniform generalization bound to show (under standard distributional assumptions) that the minimizer of the empirical $\alpha$-loss is asymptotically optimal with respect to the expected $0$-$1$ loss (probability of error), which is the ideal metric in classification problems. 
    \item Finally, we perform symmetric noisy label and class imbalance experiments on MNIST, FMNIST, and CIFAR-10 using convolutional-neural-networks. We show that models trained with $\alpha$-loss can either be more robust or sensitive to outliers (depending on the application) over models trained with log-loss ($\alpha = 1$).
    Following some of our theoretical intuitions, we demonstrate the "Goldilocks zone" of $\alpha \in (0,\infty]$, i.e., for most applications $\alpha^{*} \in [.8,8]$.
    Thus, we argue that $\alpha$-loss is an effective generalization of log-loss (cross-entropy loss) for classification problems in modern machine learning. 
\end{itemize}

Different subsets of the authors published portions of this paper as conference proceedings in~\cite{sypherd2019tunable} and~\cite{sypherd2020alpha}. 
Specifically, results provided in~\cite{sypherd2019tunable} primarily comprise a subset of the second bullet in the list above, however, this work extends those published results to $\alpha \in (0,1)$, clarifies the relationship to Arimoto conditional entropy, and provides synthetic experiments; in addition, results in~\cite{sypherd2020alpha} primarily comprise a subset of the third bullet in the list above, however, this work provides a new convexity result for $\alpha > 1$, provides SLQC background material including a point-wise statement and proof of Lemma~\ref{lemma:SLQCreformulation}, and utilizes a bootstrapping argument which significantly improves the bounds in~\cite{sypherd2020alpha}.
%
The remaining three bullets are all comprised of unpublished work.
\section{Information-Theoretic Motivations of $\alpha$-loss} \label{sec:Prelim}
Consider a pair of discrete random variables $(X,Y) \sim P_{X,Y}$. Observing $X$, one can construct an estimate $\hat{Y}$ of $Y$ such that $Y - X - \hat{Y}$ form a Markov chain. In this context, it is possible to evaluate the fitness of a given estimate $\hat{Y}$ using a loss function $\ell:\mathcal{Y}\times\mathcal{P}(\mathcal{Y})\to\mathbb{R}_{+}$ via the expectation
\begin{equation}
    \mathbb{E}_{X,Y}\left[\ell(Y,P_{\hat{Y} \vert X})\right],
\end{equation}
where $\hat{Y}|X \sim P_{\hat{Y} \vert X}$ is the \textit{learner's} posterior estimate of $Y$ given knowledge of $X$; for simplicity we sometimes abbreviate $P_{\hat{Y} \vert X=x}$ as $\hat{P}$ when the context is clear. 
In \cite{liao2018tunable}, Liao \textit{et al.} proposed the definition of $\alpha$-loss for $\alpha \in [1,\infty]$ in order to quantify adversarial action in the information leakage context. We adapt and extend the definition of $\alpha$-loss to $\alpha \in (0,\infty]$ in order to study the efficacy of the loss function in the machine learning setting. 
\begin{definition}
\label{def:alphaloss}
Let $\mathcal{P}(\mathcal{Y})$ be the set of probability distributions over $\mathcal{Y}$. We define $\alpha$-loss for $\alpha \in (0,1) \cup (1,\infty)$, $l^{\alpha}:\mathcal{Y} \times \mathcal{P}(\mathcal{Y}) \rightarrow \mathbb{R}_{+}$ as
\begin{equation} \label{eq:alphaloss_prob}
l^{\alpha}(y,\hat{P}) := \frac{\alpha}{\alpha - 1}\left(1 - \hat{P}(y)^{1 - 1/\alpha}\right),
\end{equation} 
and, by continuous extension, $l^{1}(y,\hat{P}) := -\log{\hat{P}(y)}$ and $l^{\infty}(y,\hat{P}) := 1 - \hat{P}(y)$.
\end{definition}
Observe that for $(y,\hat{P})$ fixed, $l^{\alpha}(y,\hat{P})$ is continuous and monotonically decreasing in $\alpha$. 
Also note that $l^{1}$ recovers log-loss, and plugging in $\alpha = 1/2$ yields $l^{1/2}(y,\hat{P}) := \hat{P}^{-1}(y) - 1$.
\begin{figure}[h] 
    \centering
    \centerline{\includegraphics[scale=.25]{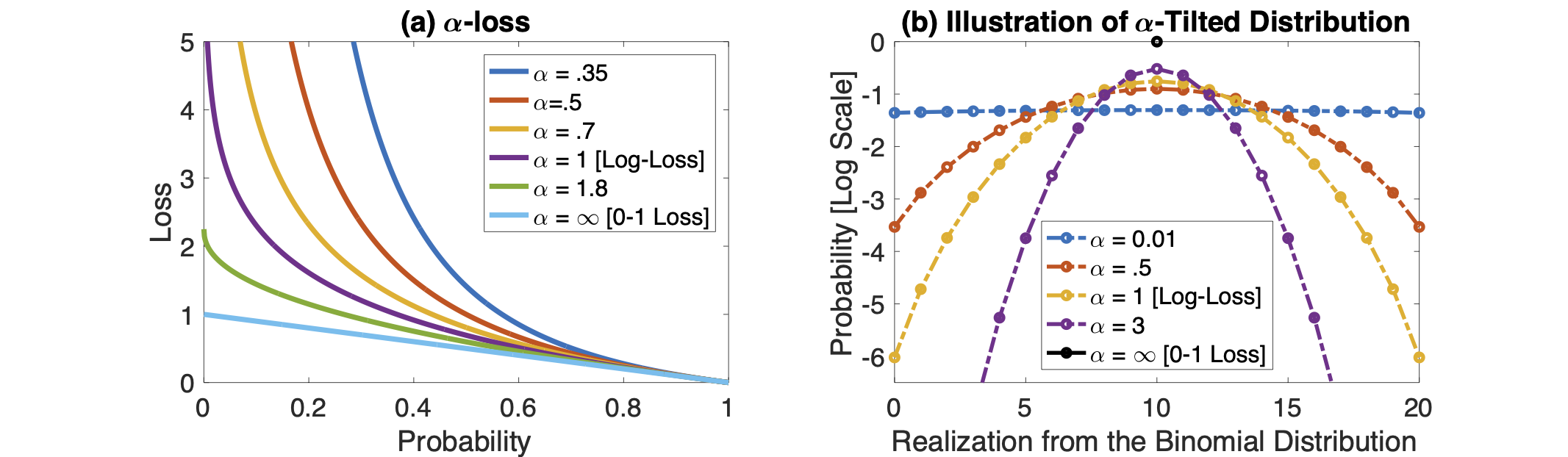}}
    \caption{\textbf{(a)} $\alpha$-loss \eqref{eq:alphaloss_prob} as a function of the probability for several values of $\alpha$; \textbf{(b)} $\alpha$-tilted posterior \eqref{eq:tilteddistribution} for several values of $\alpha$ where the true underlying distribution is the (20,0.5)-binomial distribution.}
    \label{fig:probandtilt}
\end{figure}
One can use expected $\alpha$-loss $\mathbb{E}_{X,Y}[l^{\alpha}(Y,P_{\hat{Y}|X})]$, hence called $\alpha$-risk, to quantify the effectiveness of the estimated posterior $P_{\hat{Y}|X}$. In particular, 
\begin{equation} \label{inf1}
\mathbb{E}_{X,Y}\left[l^{1}(Y,P_{\hat{Y}|X})\right] = \mathbb{E}_{X}\left[H(P_{Y|X=x},P_{\hat{Y}|X=x})\right],
\end{equation}
where $H(P,Q) := H(P) + D_{\textnormal{KL}}(P\|Q)$ is the cross-entropy between $P$ and $Q$. Similarly,
\begin{equation} \label{inf2}
\mathbb{E}_{X,Y}[l^{\infty}(Y,P_{\hat{Y}|X})] = \mathbb{P}[Y \neq \hat{Y}],
\end{equation}
i.e., the expected $\alpha$-loss for $\alpha = \infty$ equals the probability of error. 
Recall that the expectation of the canonical 0-1 loss, $\mathbb{E}_{X,Y}[\mathbbm{1}[Y \neq \hat{Y}]]$, also recovers the probability of error~\cite{shalev2014understanding}.
For this reason, we sometimes refer to $l^{\infty}$ as the $0$-$1$ loss.

Observe that $\alpha$-loss presents a tunable class of loss functions that value the probabilistic estimate of the label differently as a function of $\alpha$; see Figure~\ref{fig:probandtilt}(a).
In the sequel, we find that, when composed with a sigmoid, $l^{1/2},l^{1},l^{\infty}$ become the exponential, logistic, and sigmoid (smooth $0$-$1$) losses, respectively.
While we note that there may be infinitely many ways to continuously interpolate between the exponential, log, and $0$-$1$ losses, we observe that the interpolation introduced by \textit{$\alpha$-loss} is monotonic in $\alpha$, seems to provide an information-theoretic interpretation (Proposition \ref{Prop:Liao}), and also appears to be apt for the classification setting which will be further elaborated in the sequel. 
The following result was shown by Liao \textit{et al.} in \cite{liao2018tunable} for $\alpha \in [1,\infty]$ and provides an explicit characterization of the optimal risk-minimizing posterior under $\alpha$-loss. We extend the result to $\alpha \in (0,1)$.
\begin{prop} \label{Prop:Liao}
For each $\alpha\in (0,\infty]$, the minimal $\alpha$-risk is
\begin{equation} \label{eq:minriskarimotorelation}
 \min_{P_{\hat{Y}|X}} \mathbb{E}_{X,Y}\big[ l^{\alpha}(Y,P_{\hat{Y}|X})\big] = \frac{\alpha}{\alpha -1}\left(1 - e^{\frac{1-\alpha}{\alpha}H_{\alpha}^{A}(Y|X)}\right).
 \end{equation}
 where $H_{\alpha}^{A}(Y|X) := \dfrac{\alpha}{1 - \alpha} \log{\sum\limits_{x}\Big(\sum\limits_{y} P_{X,Y}(x,y)^{\alpha}\Big)^{1/\alpha}}$ is the Arimoto conditional entropy of order $\alpha$ \cite{arimoto1977information}. The resulting unique minimizer, $\hat{P}^{*}_{\alpha}$, is the $\alpha$-tilted true posterior
 \begin{equation} \label{eq:tilteddistribution}
 \hat{P}^{*}_{\alpha}(y|x) 
 = \dfrac{P_{Y|X}(y|x)^{\alpha}}{\sum\limits_{y} P_{Y|X} (y|x)^{\alpha}}.
\end{equation}
\end{prop}
The proof of Proposition~\ref{Prop:Liao} for $\alpha \in [1,\infty]$ can be found in \cite{liao2018tunable} and is readily extended to the case where $\alpha \in (0,1)$ with similar techniques.
Through Proposition~\ref{Prop:Liao}, we note that $\alpha$-loss exhibits different operating conditions through the choice of $\alpha$.
Observe that the minimizer of~\eqref{eq:minriskarimotorelation} given by the $\alpha$-tilted distribution in~\eqref{eq:tilteddistribution} recovers the true posterior only if $\alpha = 1$, i.e., for log-loss.
Further, as $\alpha$ decreases from 1 towards 0, $\alpha$-loss places increasingly higher weights on the low probability outcomes;
on the other hand as $\alpha$ increases from 1 to $\infty$, $\alpha$-loss increasingly limits the effect of the low probability outcomes.
Ultimately, we find that for $\alpha=\infty$, minimizing the corresponding risk leads to making a single guess on the most likely label, i.e., MAP decoding.
See Figure~\ref{fig:probandtilt}(b) for an illustration of the $\alpha$-tilted distribution on a (20,0.5)-Binomial distribution. 
Intuitively, empirically minimizing $\alpha$-loss for $\alpha \neq 1$ could be a boon for learning the minority class ($\alpha < 1$) or ignoring label noise ($\alpha > 1$); see Section~\ref{sec:Experiments} for experimental consideration of such class imbalance and noisy label trade-offs.

With the information-theoretic motivations of $\alpha$-loss behind us, we now consider the setting of binary classification, where we study the optimization, statistical, and robustness properties of $\alpha$-loss. 
\section{$\alpha$-loss in Binary Classification} \label{sec:alphalossbinaryclassification}
In this section, we study the role of $\alpha$-loss in binary classification. First, we provide its margin-based form, which we show is intimately related to the original $\alpha$-loss formulation in Definition \ref{def:alphaloss}; next, we analyze the optimization characteristics and statistical properties of the margin-based $\alpha$-loss where we notably recover the relationship between $\alpha$-loss and the Arimoto conditional entropy in the margin setting; finally, we comment on the robustness and sensitivity trade-offs which are inherent in the choice of $\alpha$ through theoretical discussion and experimental considerations. First, however, we formally discuss the binary classification setting through the role of classification functions and surrogate loss functions.

In binary classification, the learner ideally wants to obtain a classifier $h: \mathcal{X} \rightarrow \{-1,+1\}$ that  minimizes the probability of error, or the risk (expectation) of the $0$-$1$ loss, given by 
\begin{equation} \label{eq:probabilityoferror}
R(h) = \mathbb{P}[h(X) \neq Y], 
\end{equation}
where the true $0$-$1$ loss given by $\mathbbm{1}[h(X) \neq Y]$.
Unfortunately, this optimization problem is NP-hard~\cite{ben2003difficulty}. Therefore, the problem is typically relaxed by imposing restrictions on the space of possible classifiers and by choosing surrogate loss functions with desirable properties. 
%
Thus during the training phase, it is common to optimize a surrogate loss function over classification functions of the form $f:\mathcal{X}\rightarrow \overline{\mathbb{R}}$, $\overline{\mathbb{R}} = \mathbb{R} \cup \{\pm{\infty}\}$, whose output captures the certainty of a model's prediction of the true underlying binary label $Y \in \{-1,1\}$ associated with $X$~\cite{bartlett2006convexity,lin2004note,nguyen2009,masnadi2009design,sypherd2019tunable,schapire2013boosting,shalev2014understanding,friedman2001elements}. 
Once a suitable classification function has been chosen, the classifier is obtained by making a hard decision, i.e., the model outputs the classification $h(X) = \text{sign}({f(X)})$, in order to predict the true underlying binary label $Y \in \{-1,1\}$ associated with the feature vector $X \in \mathcal{X}$. 
Examples of learning algorithms which optimize surrogate losses over classification functions include SVM (hinge loss), logistic regression (logistic loss), and AdaBoost (exponential loss), to name a few \cite{friedman2001elements}.
With the notions of classification functions and surrogate loss functions in hand, we now turn our attention to an important family of surrogate loss functions in binary classification.
%

%

\subsection{Margin-based $\alpha$-loss}
Here, we provide the definition of $\alpha$-loss in binary classification and characterize its relationship to the form presented in Definition \ref{def:alphaloss}. 
First, we discuss an important family of loss functions in binary classification called \emph{margin-based} loss functions. 

%
A loss function is said to be margin-based if, for all $x\in\mathcal{X}$ and $y\in \{-1,+1\}$, the loss associated to a pair $(y,f(x))$ is given by $\tilde{l}(yf(x))$ for some function $\tilde{l}: \overline{\mathbb{R}}\to\mathbb{R}_{+}$~\cite{bartlett2006convexity,lin2004note,masnadi2009design,nguyen2009,janocha2017loss}. In this case, the loss of the pair $(y,f(x))$ only depends on the product $z:=yf(x)$, the (unnormalized) margin \cite{schapire2013boosting}. Observe that a negative margin corresponds to a mismatch between the signs of $f(x)$ and $y$, i.e., a classification error by $f$. Similarly, a positive margin corresponds to a match between the signs of $f(x)$ and $y$, i.e., a correct classification by $f$.  
We now provide the margin-based form of $\alpha$-loss, which is illustrated in Figure~\ref{fig:margin_condrisk}(a).
\begin{figure}[h] 
    \centering
    \centerline{\includegraphics[scale=.25]{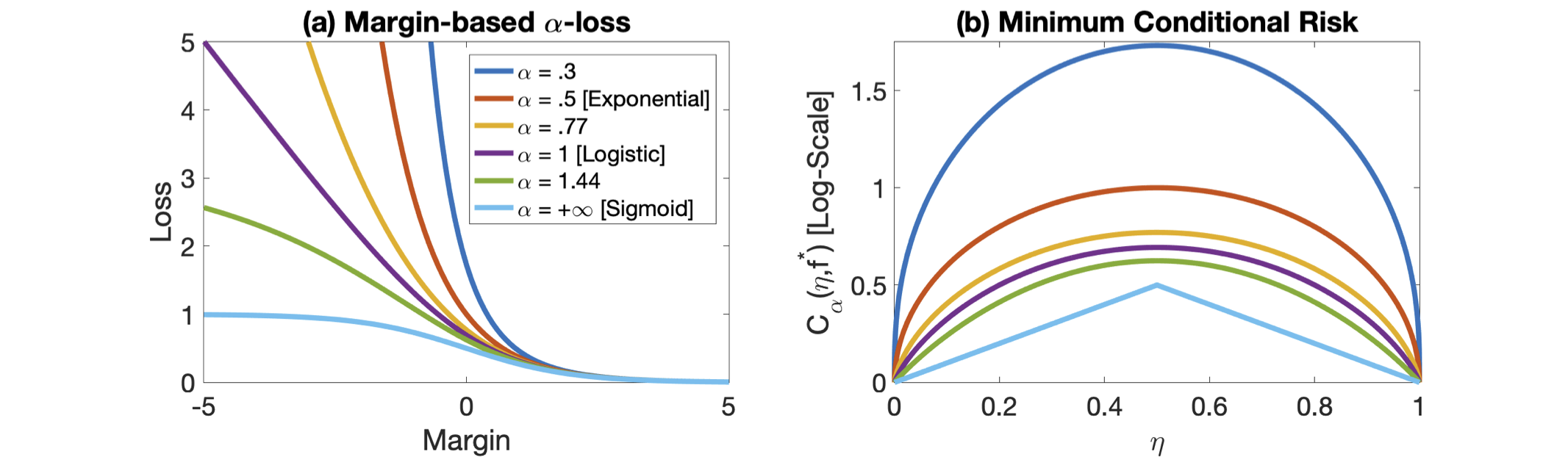}}
    \caption{\textbf{(a)} Margin-based $\alpha$-loss \eqref{eq:marginalphaloss} as a function of the margin ($z := yf(x)$) for $\alpha \in \{0.3,0.5,0.77,1,1.44,\infty\}$; \textbf{(b)} Minimum conditional risk \eqref{eq:mincondrisk} for the same values of $\alpha$.}
    \label{fig:margin_condrisk}
\end{figure}

\begin{definition} \label{def:marginbasedalphaloss}
We define the margin-based $\alpha$-loss for $\alpha \in (0,1) \cup (1,\infty)$, $\tilde{l}^{\alpha}: \overline{\mathbb{R}} \to \mathbb{R}_{+}$ as 
\begin{equation}
\label{eq:marginalphaloss}
\tilde{l}^{\alpha}(z) := \frac{\alpha}{\alpha - 1}\left(1 - \left(1+e^{-z} \right)^{1/\alpha-1}\right),
\end{equation}
and, by continuous extension, $\tilde{l}^{1}(z) = \log(1+ e^{-z})$ and $\tilde{l}^{\infty}(z) = (1+e^{z})^{-1}$.
\end{definition}
Note that $\tilde{l}^{1/2}(z) = e^{-z}$. Thus, $\tilde{l}^{1/2}$, $\tilde{l}^{1}$, and $\tilde{l}^{\infty}$ recover the exponential, logistic, and sigmoid losses, respectively. Navigating the various regimes of $\alpha$ induces different optimization, statistical, and robustness characteristics for the margin-based $\alpha$-loss; this is elaborated in the sequel. First, we discuss its relationship to the original form in Definition \ref{def:alphaloss}, which requires alternative prediction functions to classification functions called soft classifiers.

In binary classification, it is also common to use soft classifiers $g: \mathcal{X} \rightarrow [0,1]$ which encode the conditional distribution, namely, $g(x) := P_{\hat{Y}|X}(1|x)$. 
In essence, soft classifiers capture a model's \textit{belief} of $Y|X$~\cite{shalev2014understanding,goodfellow2016deep,sypherd2019tunable}. 
Similar to the classification function setting, the hard decision of a soft classifier is obtained by $h(x) = \text{sign}({g(x)-1/2})$. Log-loss, and by extension $\alpha$-loss as given in Definition \ref{def:alphaloss}, are examples of loss functions which act on soft classifiers. In practice, a soft classifier can be obtained by composing a classification function with the logistic sigmoid function $\sigma:\overline{\mathbb{R}}\to[0,1]$ given by
\begin{equation}
\label{eq:DefSigmoid}
    \sigma(z) = \frac{1}{1+e^{-z}},
\end{equation}
which is generalized by the softmax function in the multiclass setting \cite{goodfellow2016deep}. Observe that $\sigma$ is invertible and $\sigma^{-1}:[0,1]\to\overline{\mathbb{R}}$ is given by
\begin{equation} \label{eq:DefInverseSigmoid}
    \sigma^{-1}(z) = \log\left(\frac{z}{1-z}\right),
\end{equation}
which is often referred to as the logistic link \cite{reid2010composite}.

With these two transformations, one is able to map classification functions to soft classifiers and vice-versa. Thus, a loss function in one domain is readily transformed into a loss function in the other domain. In particular, we are now in a position to derive the correspondence between $\alpha$-loss in Defintion~\ref{def:alphaloss} and the margin-based $\alpha$-loss in Definition~\ref{def:marginbasedalphaloss}, which generalizes our previous proof in \cite{sypherd2019tunable}.
\begin{prop} \label{Prop:relationhardtosoft}
Consider a soft classifier $g(x)=P_{\hat{Y}|X}(1 | x)$. If $f(x) = \sigma^{-1}(g(x))$, then, for every $\alpha\in (0,\infty]$,
\begin{equation}
\label{eq:lglf}
    l^{\alpha}(y,g(x)) = \tilde{l}^{\alpha}(yf(x)).
\end{equation}
Conversely, if $f$ is a classification function, then the soft classifier $g(x) := \sigma(f(x))$ satisfies~\eqref{eq:lglf}. In particular, for every $\alpha\in(0,\infty]$,
\begin{equation}
\min_{g} \mathbb{E}_{X,Y}(l^\alpha(Y,g(x))) = \min_f \mathbb{E}_{X,Y}(\tilde{l}^\alpha(Yf(X))).
\end{equation}
\end{prop}

Thus, there is a direct correspondence between $\alpha$-loss in Definition \ref{def:alphaloss} and the margin-based $\alpha$-loss used in binary classification.

\begin{remark}
Instead of the fixed inverse link function~\eqref{eq:DefSigmoid}, it is also possible to use any other fixed inverse link function, or even inverse link functions dependent on $\alpha$; indeed, it is possible to derive many such tunable margin-based losses this way.
However, 
the margin-based $\alpha$-loss as given in Definition~\ref{def:marginbasedalphaloss}
allows for continuous interpolation between the exponential, logistic, and sigmoid losses, and thus motivates our choice of the fixed sigmoid in~\eqref{eq:DefSigmoid} as the inverse link.
\end{remark}

The following result, which quantifies the convexity of the margin-based $\alpha$-loss, will be useful in characterizing the convexity of the average loss, or \textit{landscape}, in the sequel.
\begin{prop} \label{Prop:alpha-loss-convex}
As a function of the margin, $\tilde{l}^{\alpha}:\overline{\mathbb{R}}\to\mathbb{R}_+$ is convex for $\alpha \leq 1$ and quasi-convex for $\alpha > 1$.
\end{prop}
Recall that a real-valued function $f: \mathbb{R} \rightarrow \mathbb{R}$ is quasi-convex if, for all $x,y \in \mathbb{R}$ and $\lambda \in [0,1]$, we have that $f(\lambda x + (1-\lambda)y) \leq \max{\{f(x),f(y)\}}$, and also recall that any monotonic function is quasi-convex (see e.g.,~\cite{boyd2004convex}). 
Intuitively through Figure~\ref{fig:margin_condrisk}(a), we find that the quasi-convexity of the margin-based $\alpha$-loss for $\alpha > 1$ reduces the penalty induced during training for samples which have a negative margin; this has implications for robustness that will also be investigated in the sequel.
%
%
\subsection{Calibration of Margin-based $\alpha$-loss}
With the definition and basic properties of the margin-based $\alpha$-loss in hand, we now discuss a statistical property of the margin-based $\alpha$-loss that highlights its suitability in binary classification. 
%
Bartlett \textit{et al.} in \cite{bartlett2006convexity} introduce \textit{classification-calibration} as a means to compare the performance of a margin-based loss function relative to the 0-1 loss by inspecting the minimizer of its conditional risk. 
Formally, let $\phi: \overline{\mathbb{R}}\to\mathbb{R}_+$ denote a margin-based loss function and let $C_{\phi}(\eta(x),f(x)) = \mathbb{E}[\phi(Yf(X))|X=x]$ denote its conditional expectation (risk), where $\eta(x) = P_{Y|X}(1|x)$ is the true posterior and $f: \mathcal{X} \rightarrow \overline{\mathbb{R}}$ is a classification function. 
Thus, the conditional risk of the margin-based $\alpha$-loss for $\alpha \in (0,\infty]$ is given by 
\begin{align} \label{eq:condriskalphaloss}
C_{\alpha}(\eta(x),f(x)) = \mathbb{E}_{Y}[\tilde{l}^{\alpha}(Yf(X))|X=x].
\end{align}
We say that $\phi: \overline{\mathbb{R}}\to\mathbb{R}_+$ is classification-calibrated if, for all $x \in \mathcal{X}$, its minimum conditional risk 
\begin{align} \label{eq:mincondrisk}
\inf\limits_{f: \mathcal{X} \rightarrow \mathbb{R}} C_{\phi}(\eta(x),f(x)) = \inf\limits_{f: \mathcal{X} \rightarrow \mathbb{R}} \eta(x) \phi(f(x)) + (1-\eta(x))\phi(-f(x))
\end{align}
is attained by a $f^{*}:\mathcal{X} \rightarrow \overline{\mathbb{R}}$ such that $\textnormal{sign}(f^{*}(x)) = \textnormal{sign}(2\eta(x)-1)$. 
In words, a margin-based loss function is classification-calibrated if for each feature vector, the minimizer of its minimum conditional risk agrees in sign with the Bayes optimal predictor. Note that this is a pointwise form of Fisher consistency \cite{lin2004note,bartlett2006convexity}.

The expectation of the loss function $\phi$, or the $\phi$-risk, is denoted $R_{\phi}(f) = \mathbb{E}_{X}[C_{\phi}(\eta(X),f(X))]$; this notation will be useful in the sequel when we quantify the asymptotic behavior of $\alpha$-loss.
Finally, as is common in the literature, we omit the dependence of $\eta$ and $f$ on $x$, and we also let $C_{\phi}^{*}(\eta) = C_{\phi}(\eta,f^{*})$ for notional convenience~\cite{masnadi2009design,bartlett2006convexity}. 
With the necessary background on classification-calibrated loss functions in hand, we are now in a position to 
show that $\tilde{l}^{\alpha}$ is classification-calibrated for all $\alpha \in (0,\infty]$.
\begin{theorem}
\label{thm:alphalossclassificationcalibration}
For $\alpha\in (0,\infty]$, the margin-based $\alpha$-loss $\tilde{l}^{\alpha}$ is classification-calibrated. In addition, its optimal classification function is given by
\begin{equation}
\label{eq:optimalclassifier}
    f^{*}_{\alpha}(\eta) = \alpha \cdot \sigma^{-1}(\eta) = \alpha \log\left(\frac{\eta}{1-\eta}\right).
\end{equation}
\end{theorem}
See Appendix~\ref{appen:binaryclass} for full proof details.
Examining the optimal classification function in~\eqref{eq:optimalclassifier} more closely, we observe that this expression is readily derived from the $\alpha$-tilted distribution for a binary label set in Proposition~\ref{Prop:relationhardtosoft}. 
Thus, analogous to the intuitions regarding the $\alpha$-tilted distribution in~\eqref{eq:tilteddistribution}, the optimal classification function in~\eqref{eq:optimalclassifier} suggests that $\alpha > 1$ is more robust to slight fluctuations in $\eta$ and $\alpha < 1$ is more sensitive to slight fluctuations in $\eta$. In the sequel, we find that this has practical implications for noisy labels and class imbalances.

Upon plugging~\eqref{eq:optimalclassifier} into~\eqref{eq:condriskalphaloss}, we obtain the following result which specifies the minimum conditional risk of $\tilde{l}^{\alpha}$ for $\alpha \in (0,\infty]$.
\begin{corollary} \label{cor:condrisk}
For $\alpha \in (0,\infty]$, the minimum conditional risk of $\tilde{l}^{\alpha}$ is given by
\begin{equation} \label{eq:mincondalpharisk}
C_{\alpha}^{*}(\eta) =
\begin{cases} 
    \frac{\alpha}{\alpha - 1} \left(1 - (\eta^{\alpha} + (1-\eta)^{\alpha})^{1/\alpha} \right) & \alpha \in (0,1) \cup (1,+\infty), \\
    -\eta \log{\eta} - (1-\eta) \log{(1-\eta)} & \alpha = 1, \\
    \min\{\eta, 1- \eta \} & \alpha \rightarrow +\infty.
\end{cases}
\end{equation}
\end{corollary}
\begin{remark}
Observe that in \eqref{eq:mincondalpharisk} for $\alpha = 1$, the minimum conditional risk can be rewritten as 
\begin{align}
C_{1}^{*}(\eta) = -\eta \log{\eta} - (1-\eta) \log{(1-\eta)} = H(Y|X=x),
\end{align}
where $H(Y|X=x)$ is the Shannon conditional entropy for a $Y$ given $X = x$~\cite{cover1999elements}.
Also note that in \eqref{eq:mincondalpharisk} for $\alpha \in (0,1) \cup (1,+\infty)$, the minimum conditional risk can be rewritten as 
\begin{align}
C_{\alpha}^{*}(\eta) = \frac{\alpha}{\alpha - 1} \left[1 - (\eta^{\alpha} + (1-\eta)^{\alpha})^{1/\alpha} \right] = \frac{\alpha}{\alpha - 1} \left[1 - e^{\frac{1-\alpha}{\alpha} H_{\alpha}^{A}(Y|X=x)} \right],
\end{align}
where $H_{\alpha}^{A}(Y|X=x) = \frac{1}{1-\alpha} \log{\left(\sum_{y} P_{Y|X}(y|x)^{\alpha} \right)}$ is the Arimoto conditional entropy of order $\alpha$~\cite{arimoto1977information}.
Finally, observe that $\mathbb{E}_{X}[C_{\alpha}^{*}(\eta(X))]$ recovers \eqref{eq:minriskarimotorelation} in Proposition~\ref{Prop:Liao}.
%
\end{remark}

Finally, note that the minimum conditional risk of the margin-based $\alpha$-loss is concave for all $\alpha \in (0,\infty]$ (see Figure~\ref{fig:margin_condrisk}(b)); indeed, this is known to be a useful property for classification problems \cite{masnadi2009design}.
Therefore, since the margin-based $\alpha$-loss is classification-calibrated and its minimum conditional risk is concave for all $\alpha \in (0,\infty]$, it seems to have reasonable statistical behavior for binary classification problems.
We now turn our attention to the robustness and sensitivity tradeoffs induced by traversing the different regimes of $\alpha$ for the margin-based $\alpha$-loss.

\subsection{Robustness and Sensitivity of Margin-based $\alpha$-loss} \label{sec:syntheticexp}
\begin{figure}[h] 
    \hspace{-.3cm}
    \centerline{\includegraphics[width=.9\textwidth,trim=10 10 10 10,clip]{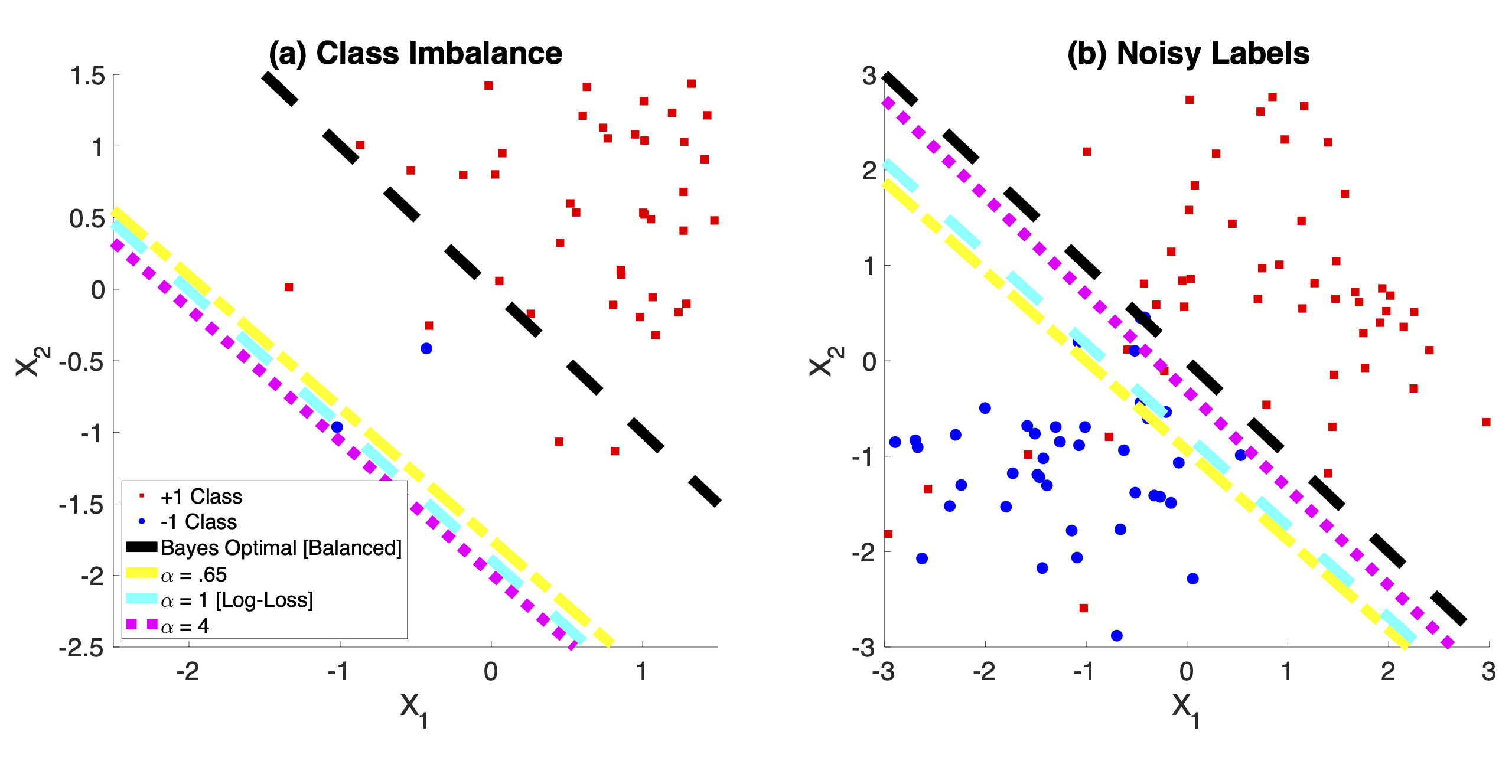}}
    \caption{Two synthetic experiments each averaged over 100 runs highlighting the differences in trained linear predictors of $\alpha$-loss for $\alpha \in \{0.65, 1, 4\}$ on imbalanced and noisy data, which are compared with the Bayes optimal predictor for the clean, balanced distribution.
    Training data present in both figures is obtained from the last run in each experiment, respectively.
    \textbf{(a)}~Averaged linear predictors trained using $\alpha$-loss on imbalanced data with 2 samples from $Y = -1$ class per run. Averaged linear predictors for smaller values of $\alpha$ are closer to the Bayes predictor for the balanced distribution, which highlights the sensitivity of $\alpha$-loss to the minority class for $\alpha < 1$. \textbf{(b)}~Averaged linear predictors trained using $\alpha$-loss on noisy data, which is obtained by flipping the labels of the $Y=-1$ class with probability $0.2$. Averaged linear predictor for $\alpha=4$ is closer to the Bayes predictor for the balanced distribution, which highlights the robustness of $\alpha$-loss to noise for $\alpha > 1$.}
    \label{fig:planeplots}
\end{figure}

Despite the advantages of convex losses in terms of numerical optimization and theoretical tractability, non-convex loss functions often provide superior model robustness and classification accuracy \cite{mei2018landscape,nguyen2013algorithms,barron2019general,sypherd2019tunable,schapire2013boosting,wu2007robust,chapelle2009tighter,long2010random,masnadi2009design}.
In essence, non-convex loss functions tend to assign less weight to misclassified training examples\footnote{Convex losses grow at least linearly with respect to the negative margin which results in an increased sensitivity to outliers.
See Figure~\ref{fig:margin_condrisk}(a) for $\alpha = 1$ as an example of this phenomenon.} and therefore algorithms optimizing such losses are often less perturbed by outliers, i.e., samples which induce large negative margins. 
More concretely, consider Figure~\ref{fig:margin_condrisk}(a) for $\alpha = 1/2$ (convex) and $\alpha = 1.44$ (quasi-convex), and suppose that $z_{1} = -1$ and $z_{2} = -5$. 
Plugging these parameters into Definition~\ref{def:marginbasedalphaloss}, we find that $\tilde{l}^{1/2}(z_{1}) = e^{1} \approx 2.7$,  $\tilde{l}^{1/2}(z_{2})=e^{5}\approx148.4$, $\tilde{l}^{1.44}(z_{1}) \approx 1.1$, and $\tilde{l}^{1.44}(z_{2}) \approx 2.6$.
In words, the difference in these loss evaluations for a negative value of the margin, which is representative of a misclassified training example, is approximately exponential versus sub-linear. Indeed, this difference in loss function behavior appears to be most salient with outliers (e.g., noisy or imbalanced training examples)~\cite{masnadi2009design,schapire2013boosting}.

We explore these ideas with the synthetic experiment presented in Figure~\ref{fig:planeplots}.
We assume the practitioner has access to modified training data which approximates the true underlying distribution given by a two-dimensional Gaussian Mixture Model (2D-GMM) with equal mixing probability $\mathbb{P}[Y=-1] = \mathbb{P}[Y=+1]$, symmetric means $\mu_{X|Y=-1} = (-1, -1)^{\intercal} = -\mu_{X|Y=1}$, and shared identity covariance matrix $\Sigma = \mathbb{I}_{2}$. 
The first experiment considers the scenario where the training data suffers from a class imbalance; specifically, the number of training samples for the $Y = -1$ class is $2$ and the number of training samples for the $Y = +1$ class is $98$ for every run.
The second experiment considers the scenario where the training data suffers from noisy labels; specifically, the labels of the $Y = -1$ class are flipped with probability $0.2$ and the labels of the $Y = +1$ class are kept fixed.
For both experiments we train $\alpha$-loss on the logistic model, which is the generalization of logistic regression with $\alpha$-loss and is formally described in the next section.
Specifically, we minimize $\alpha$-loss using gradient descent with the fixed $\textit{learning rate} = 0.01$ for each $\alpha \in \{0.65,1,4\}$.
Note that $\alpha = 0.65$ (lower limit) and $\alpha = 4$ (upper limit) were both chosen for computational feasibility in the logistic model; in practice, the range of $\alpha \in (0,\infty]$, while usually contracted as in this experiment, is dependent on the model - this is elaborated in the sequel.
Training is allowed to progress until convergence as specified by the $\textit{optimality parameter} = 10^{-4}$. 
The linear predictors presented in Figure~\ref{fig:planeplots} are averaged over $100$ runs of randomly generated data according to the parameters for each experiment.

Ideally, the practitioner would like to generate a linear predictor which is invariant to noisy or imbalanced training data and tends to align with the Bayes optimal predictor for the balanced distribution. Indeed, when the training data is balanced (and clean), all averaged linear predictors generated by $\alpha$-loss collapse to the Bayes predictor; see
Figure~\ref{fig:planeplotsclean} in Appendix~\ref{sec:extrasyntheticexp}.
However, training on noisy or imbalanced data affects the linear predictors of $\alpha$-loss in different ways. 
In the class imbalance experiment in Figure~\ref{fig:planeplots}(a), we find that the averaged linear predictor for the smaller values of $\alpha$ more closely approximate the Bayes predictor for the balanced distribution, which suggests that the smaller values of $\alpha$ are more sensitive to the minority class.
Similarly in the class noise experiment in Figure~\ref{fig:planeplots}(b), we find that the averaged linear predictor for $\alpha = 4$ more closely approximates the Bayes predictor for the balanced distribution, which suggests that the larger values of $\alpha$ are less sensitive to noise in the training data.
Both results suggest that $\alpha = 1$ (log-loss) can be improved with the use of $\alpha$-loss in these scenarios.
For quantitative results of this experiment, including a wider range of $\alpha$'s, additional imbalances and noise levels, and results using the $\text{F}_{1}$ score, see Tables~\ref{table:syntheticclassimbalanceaccuracy},~\ref{table:syntheticimbalancef1}, and~\ref{table:syntheticnoisyaccuracy} in Appendix~\ref{sec:extrasyntheticexp}.

In summary, we find that navigating the convexity regimes of $\alpha$-loss induces different robustness and sensitivity characteristics. We explore these themes in more detail on canonical image datasets in Section \ref{sec:Experiments}. We now turn our attention to theoretically characterizing the optimization complexity of $\alpha$-loss for the different regimes of $\alpha$ in the logistic model.

\section{Optimization Guarantees for $\alpha$-loss in the Logistic Model} \label{sec:landscapelogisticmodel}


In this section, we analyze the optimization complexity of $\alpha$-loss in the logistic model as we vary $\alpha$ by quantifying the convexity of the optimization landscape. First, we show that the $\alpha$-risk is convex (indeed, strongly-convex if a certain correlation matrix is positive definite) in the logistic model for $\alpha \leq 1$; next, we provide a brief summary of a notion known as \textit{strictly local quasi-convexity} (SLQC); then, we provide a more tractable reformulation of SLQC which is instrumental for our theory; finally, we study the convexity of the $\alpha$-risk in the logistic model through SLQC for a range of $\alpha > 1$, which we argue is sufficient due to the rapid saturation effect of $\alpha$-loss as $\alpha \rightarrow \infty$.
Notably, our main result depends on a bootstrapping argument that might be of independent interest.
Our main conclusion of this section is that there exists a "Goldilocks zone" of $\alpha \in (0,\infty]$ which drastically reduces the hyperparameter search induced by $\alpha$ for the practitioner.
Finally, note that all proofs and background material can be found in Appendix~\ref{appen:logisticmodel}.
%
%
\subsection{$\alpha$-loss in the Logistic Model}
Prior to stating our main results, we clarify the setting and provide necessary definitions. 
Let $X \in [0,1]^{d}$ be the normalized feature where $d \in \mathbb{N}$ is the number of dimensions, $Y \in \{-1,+1\}$ the label and we assume that the pair is distributed according to an unknown distribution $P_{X,Y}$, i.e., $(X,Y) \sim P_{X,Y}$.
For $\tilde{\theta}\in\mathbb{R}^d$ and $r>0$, we let $\mathbb{B}_{d}(\tilde{\theta},r) := \{\theta\in\mathbb{R}^d : \|\theta-\tilde{\theta}\| \leq r\}$. 
For simplicity, we let $\mathbb{B}_{d}(r)= \mathbb{B}_{d}(\mathbf{0},r)$ when $\tilde{\theta} = \mathbf{0}$; also note that all norms are Euclidean.
Given $r>0$, we consider the logistic model and its associated hypothesis class $\mathcal{G} = \{g_\theta:\theta\in\mathbb{B}_{d}(r)\}$, composed of parameterized soft classifiers $g_{\theta}$ such that
%
\begin{equation}
g_{\theta}(x) = \sigma(\langle \theta,x \rangle),
\end{equation}
with $\sigma: \mathbb{R} \rightarrow [0,1]$ being the sigmoid function given by \eqref{eq:DefSigmoid}.
For convenience, we present the following short form of $\alpha$-loss in the logistic model which is equivalent to the expanded expression in \cite{sypherd2019tunable}.
For $\alpha \in (0,\infty]$, $\alpha$-loss is given by
\begin{equation} \label{eq:alphadefLR}
l^{\alpha}(y,g_{\theta}(x)) = \frac{\alpha}{\alpha-1} \left[1 - g_{\theta}(yx)^{1-1/\alpha} \right].
\end{equation}
For $\alpha = 1$, $l^{1}$ is the logistic loss and we recover logistic regression by optimizing this loss. 
Note that in this setting $\langle yx, \theta \rangle$ is the margin, and recall from Proposition~\ref{Prop:alpha-loss-convex} that \eqref{eq:alphadefLR} is convex for $\alpha \in (0,1]$ and quasi-convex for $\alpha > 1$ in $\langle yx, \theta \rangle$. 
For $\theta \in \mathbb{B}_{d}(r)$, we define the $\alpha$-risk $R_\alpha$ as the risk of the loss in \eqref{eq:alphadefLR},  
\begin{equation} \label{eq:alphariskdefLR}
R_{\alpha}({\theta}) := \mathbb{E}_{X,Y}[l^{\alpha}(Y,g_{\theta}(X))].
\end{equation}
The $\alpha$-risk \eqref{eq:alphariskdefLR} is plotted for several values of $\alpha$ in a two-dimensional Gaussian Mixture Model (GMM) in Figure~\ref{fig:fourlandscapes}.
Further, observe that, for all $\theta\in\mathbb{B}_d(r)$,
\begin{equation}
R_\infty(\theta) := \mathbb{E}_{X,Y}[l^{\infty}(Y,g_{\theta}(X))] = \mathbb{P}[Y \neq \hat{Y}_\theta],
\end{equation}
where $\hat{Y}_\theta$ is a random variable such that for all $x\in\mathbb{B}_d(1)$, $\mathbb{P}[\hat{Y}_\theta = 1|X=x] = g_\theta(x)$. 

In order to study the landscape of the $\alpha$-risk, we compute the gradient and Hessian of \eqref{eq:alphadefLR}, by employing the following useful properties of the sigmoid 
\begin{align} \label{eq:sigprop1}
\sigma(-z) = 1-\sigma(z) \quad \quad \textnormal{and} \quad \quad \dfrac{d}{dz} \sigma(z) = \sigma(z)(1-\sigma(z)).
\end{align}
Indeed, a straightforward computation shows that
\begin{equation} \label{eq:alphaderLR}
\frac{\partial}{\partial \theta^{j}} l^{\alpha}(y,g_{\theta}(x)) = \left[- y g_\theta(yx)^{1-1/\alpha}(1-g_\theta(yx))\right]x^{j}, 
\end{equation} 
where $\theta^{j}, x^{j}$ denote the $j$-th components of $\theta$ and $x$, respectively. 
Thus, the gradient of $\alpha$-loss in \eqref{eq:alphadefLR} is
\begin{equation} \label{eq:alphagradLR}
\nabla_\theta l^{\alpha}(Y,g_\theta(X)) = F_{1}(\alpha,\theta,X,Y)X,
\end{equation}
where $F_{1}(\alpha,\theta,x,y)$ is defined as the expression within brackets in \eqref{eq:alphaderLR}.
Another straightforward computation yields
\begin{equation} \label{eq:alphader2LR}
\nabla^{2}_{\theta}l^{\alpha}(Y,g_{\theta}(X)) = F_{2}(\alpha,\theta,X,Y)XX^{\intercal},
\end{equation}
where $F_{2}$ is defined as
\begin{align} \label{eq:F2}
F_2(\alpha,\theta,x,y) &:= g_{\theta}(yx)^{1 - \alpha^{-1}}g_{\theta}(-yx) \left(g_{\theta}(yx) - \left(1-\alpha^{-1}\right) g_{\theta}(-yx) \right).
\end{align}

\subsection{Convexity of the $\alpha$-risk}
\begin{figure}[h] 
    \centering
    \centerline{\includegraphics[width=.75\linewidth]{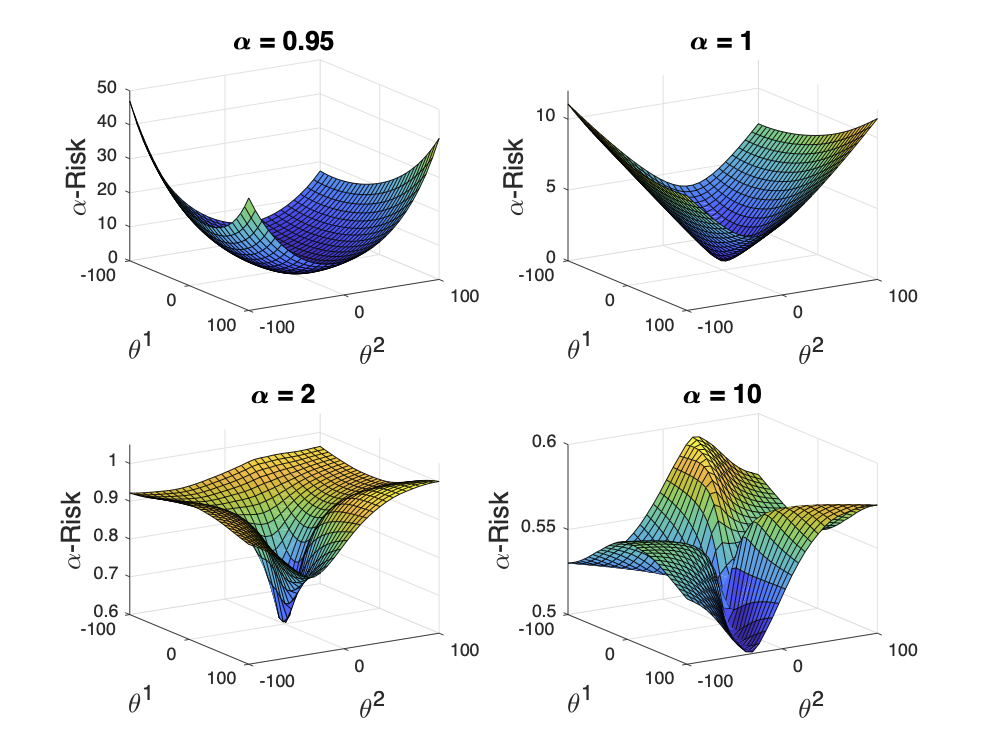}}
    \caption{The landscape of $\alpha$-loss ($R_{\alpha}$ for $\alpha = 0.95, 1, 2, 10$) in the logistic model, where features are normalized, for a 2D-GMM with $\mathbb{P}[Y=-1] = 0.12 = 1 - \mathbb{P}[Y=1]$, $\mu_{X|Y=-1} =(-0.18, 1.49)^{\intercal}$, $\mu_{X|Y=1} = (-0.01,0.16)^{\intercal}$, $\Sigma_{-1} = [3.20, -2.02; -2.02, 2.71]$, and $\Sigma_{1} = [4.19, 1.27; 1.27, 0.90]$.}
    \label{fig:fourlandscapes}
\end{figure}


We now turn our attention to the case where $\alpha \in (0,1]$; we find that for this regime, $R_{\alpha}$ is strongly convex; see Figure~\ref{fig:fourlandscapes} for an example. 
Prior to stating the result, for two matrices $A,B \in \mathbb{R}^{d\times d}$, we let $\succeq$ denote the Loewner (partial) order in the positive semi-definite cone. That is, we write $A \succeq B$ when $A-B$ is a positive semi-definite matrix. 
For a matrix $A \in \mathbb{R}^{d\times d}$, let $\lambda_1(A),\ldots,\lambda_d(A)$ be its eigenvalues.
Finally, we recall that a function is $m$-strongly convex if and only if its Hessian has minimum eigenvalue $m \geq 0$ \cite{boyd2004convex}. 
\begin{theorem} \label{Thm:SLQClessthan1}
Let $\Sigma := \mathbb{E}[XX^{\intercal}]$. If $\alpha \in (0,1]$, then $R_{\alpha}(\theta)$ 
is $\Lambda(\alpha,r\sqrt{d}) \min_{i \in [d]}\lambda_{i}\left(\Sigma\right)$-strongly convex in $\theta \in \mathbb{B}_{d}(r)$,
where 
\begin{align}
\Lambda(\alpha,r\sqrt{d}) := \sigma(r\sqrt{d})^{1-1/\alpha}\left(\sigma'(r\sqrt{d}) - \left(1 - \alpha^{-1} \right)\sigma(-r\sqrt{d})^{2}\right).
\end{align}
\end{theorem}
Observe that if $\min_{i \in [d]} \lambda_{i}(\Sigma) = 0$, then the $\alpha$-risk is merely convex for $\alpha \leq 1$. 
Also observe that for $r \sqrt{d} > 0$ fixed, $\Lambda(\alpha,r\sqrt{d})$ is monotonically decreasing in $\alpha$. 
Thus, $R_{\alpha}$ becomes more strongly convex as $\alpha$ approaches zero.

While Theorem \ref{Thm:SLQClessthan1} states that the $\alpha$-risk is strongly-convex for all $\alpha \leq 1$ and for any $r\sqrt{d}>0$, the following corollary, which is proved with similar techniques as Theorem \ref{Thm:SLQClessthan1}, states that the $\alpha$-risk is strongly-convex for some range of $\alpha > 1$, provided that $r\sqrt{d}>0$ is small enough. 
\begin{corollary} \label{cor:convexforalphagreaterthan1}
Let $\Sigma := \mathbb{E}[XX^{T}]$.
If $r\sqrt{d} \leq \arcsinh{(1/2)}$, then $R_{\alpha}(\theta)$ is $\tilde{\Lambda}(\alpha,r\sqrt{d}) \min_{i \in [d]}\lambda_{i}\left(\Sigma\right)$-strongly convex in $\theta \in \mathbb{B}_{d}(r)$ for $\alpha \in \left(0,(e^{2r\sqrt{d}}-e^{r\sqrt{d}})^{-1}\right]$, where 
\begin{align}
\tilde{\Lambda}(\alpha,r\sqrt{d}) := \sigma(-r\sqrt{d})^{2-1/\alpha}\sigma(r\sqrt{d})\left(1 - e^{r\sqrt{d}} + \alpha^{-1} e^{-r\sqrt{d}} \right).
\end{align}
\end{corollary}
Observe that if $r\sqrt{d} < \arcsinh{(1/2)}$, then $(e^{2r\sqrt{d}}-e^{r\sqrt{d}})^{-1} > 1$.
By inspecting the relationship between convexity and its dependence on $r\sqrt{d}$, Corollary~\ref{cor:convexforalphagreaterthan1} seems to suggest that as $\alpha$ increases slightly greater than $1$,
convexity is lost faster nearer to the boundary of the parameter space.
Indeed, refer to Figure~\ref{fig:fourlandscapes} to observe an example of this effect for $\alpha$ increasing from $\alpha=1$ to $\alpha = 2$, and note that convexity is preserved in the small radius about $\mathbf{0}$ for $\alpha = 2$. 

Examining the $\alpha$-risk in Figure~\ref{fig:fourlandscapes} for $\alpha = 2$ more closely, we see that it is reminiscent of a quasi-convex function.
Recall that (e.g., Chapter 3.4 in~\cite{boyd2004convex}) a function $f: \mathbb{R}^{d} \rightarrow \mathbb{R}$ is quasi-convex if for all $\theta,\theta_{0} \in \mathbb{R}^{d}$, such that $f(\theta_{0}) \leq f(\theta)$, it follows that
\begin{equation} \label{eq:quasiconvexity}
\langle -\nabla f(\theta), \theta_{0}-\theta \rangle \geq 0.
\end{equation}
In other words, the negative gradient of a quasi-convex function always points in the direction of descent. 
While $\alpha$-loss \eqref{eq:alphadefLR} is quasi-convex for $\alpha > 1$, this does not imply that the \textit{$\alpha$-risk} \eqref{eq:alphariskdefLR} is quasi-convex for $\alpha > 1$ since the sum of quasi-convex functions is not guaranteed to be quasi-convex~\cite{boyd2004convex}.
%
Thus, we need a new tool in order to quantify the optimization complexity of the $\alpha$-risk for $\alpha > 1$ in the large radius regime.

\subsection{Strictly Local Quasi-Convexity and its Extensions}

We use a framework developed by Hazan \textit{et al.} in \cite{hazan2015beyond} called \textit{strictly local quasi-convexity} (SLQC), which is a generalization of quasi-convexity. 
Intuitively, SLQC functions allow for multiple local minima below an $\epsilon$-controlled region while stipulating (strict) quasi-convex functional behavior outside the same region. 
Formally, we recall the following parameteric definition of SLQC functions provided in~\cite{hazan2015beyond}. 
%
\begin{definition}[{Definition~3.1},\cite{hazan2015beyond}] \label{def:SLQC}
Let $\epsilon,\kappa>0$ and $\theta_{0} \in \mathbb{R}^{d}$. A function $f: \mathbb{R}^{d} \rightarrow \mathbb{R}$ is called $(\epsilon, \kappa, \theta_{0})$-strictly locally quasi-convex (SLQC) at $\theta\in\mathbb{R}^d$ if at least one of the following applies:
\begin{enumerate}
\item[1.] $f(\theta) - f(\theta_{0}) \leq \epsilon$,
\item[2.] $\|\nabla f(\theta)\| > 0$ and $\langle -\nabla f(\theta), \theta' - \theta \rangle \geq 0$ for every $\theta' \in \mathbb{B}(\theta_{0}, \epsilon/\kappa)$. 
\end{enumerate}
\end{definition}
\begin{figure}[h]
\centering
\begin{tikzpicture}
        \coordinate (A) at (5, 0) {};
        \coordinate (B) at (4.5500,-1.4309) {};
        \coordinate (C) at (4.4500, -1.4000) {};
\filldraw[color=black, fill=blue!5, very thick](5,0) circle (1.5); 
\draw[black, very thick] (5,0) -- (4.5500,-1.4309) node[right,midway] {$\dfrac{\epsilon}{\kappa}$};
\filldraw [blue] (4.5500,-1.4309) circle (2pt) node[anchor=north,color=black] {$\theta'$};
\filldraw[red] (5,0) circle (2pt) node[anchor=west, color = black] {$\theta_{0}$};
\draw[blue, thick,fill=blue!5] (.145,-.4) arc (-70:70:.425) node[anchor=west] {};
\draw[red, thick,fill=red!5] (0,-.3) arc (-90:90:.3) node[anchor=west] {};
\draw[black, very thick, ->, name path=B] (0,0) -- (0.45,1.4309) node[anchor=south] {}; 
\draw[black, very thick, ->, name path=B] (0,0) -- (0.45,-1.4309) node[anchor=south] {};
\draw[black, very thick, ->, name path=B] (0,0) -- (0,1.5) node[anchor=south] {}; 
\draw[black, very thick, ->, name path=B] (0,0) -- (0,-1.5) node[anchor=south] {};
\draw[dotted, very thick, ->, name path=B] (0,0) -- (1.4,0.5385) node[anchor=south] {$-\nabla f(\theta)$};
\filldraw [black] (0,0) circle (2pt) node[anchor=east] {$\theta$};
\end{tikzpicture}
\caption{An illustration highlighting the difference between quasi-convexity as given in \eqref{eq:quasiconvexity} and the second SLQC condition of Definition \ref{def:SLQC}. If $f$ is quasi-convex, the red angle describes the possible negative gradients of $f$ at $\theta$ with respect to $\theta_{0}$. If $f$ is SLQC, the blue angle describes the possible negative gradients of $f$ at $\theta$ with respect to $\theta_{0}$ and the given $\epsilon/\kappa$-radius ball.}
    \label{fig:quasiconvexityvsSLQC}
\end{figure}
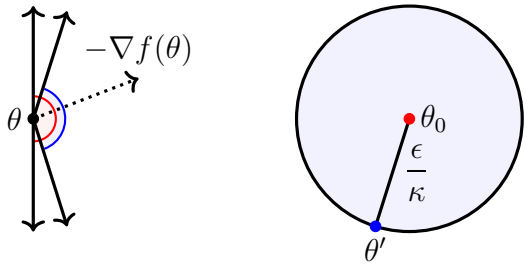
%
Briefly, in \cite{hazan2015beyond} Hazan \textit{et al.} refer to a function as SLQC \textit{in $\theta$}, whereas for the purposes of our analysis we refer to a function as SLQC \textit{at} $\theta$. We recover the uniform SLQC notion of Hazan \textit{et al.} by articulating a function is SLQC \textit{at $\theta$ for every $\theta$}. Our later analysis of the $\alpha$-risk in the logistic model benefits from this pointwise consideration.

Observe that where Condition 1 of Definition~\ref{def:SLQC} does not hold, Condition 2 implies quasi-convexity about $\mathbb{B}(\theta_{0}, \epsilon/\kappa)$ as evidence through~\eqref{eq:quasiconvexity}; see Figure~\ref{fig:quasiconvexityvsSLQC} for an illustration of the difference between classical quasi-convexity and SLQC in this regime.
We now present the following lemma, which is a structural result for general differentiable functions that provides an alternative formulation of the second requirement of SLQC functions in Definition~\ref{def:SLQC}; proof details can be found in Appendix~\ref{appen:prelim&reformSLQC}. 
\begin{lemma} \label{lemma:SLQCreformulation}
Assume that $f:\mathbb{R}^d\to\mathbb{R}$ is differentiable, $\theta_{0} \in \mathbb{R}^d$ and $\rho>0$. If $\theta\in\mathbb{R}^d$ is such that $\|\theta-\theta_0\|>\rho$, then the following are equivalent:
\begin{itemize}
    \item[1.] $\langle -\nabla f(\theta), \theta' - \theta \rangle \geq 0$ for all $\theta' \in \mathbb{B}_{d}\left(\theta_{0},\rho\right)$,
    \item[2.] $\langle -\nabla f(\theta), \theta_{0} - \theta \rangle \geq \rho \|\nabla f(\theta)\|$.
\end{itemize}
\end{lemma}
Intuitively, the equivalence presented by Condition~2 of Lemma~\ref{lemma:SLQCreformulation} is easier to manipulate in proving SLQC properties of the $\alpha$-risk as we merely need to control $\langle -\nabla f(\theta), \theta_{0} - \theta \rangle$ rather than $\langle -\nabla f(\theta), \theta' - \theta \rangle$ for every $\theta' \in \mathbb{B}(\theta_{0}, \epsilon/\kappa)$.

In \cite{hazan2015beyond}, Hazan \textit{et al.}~measure the optimization complexity of SLQC functions through the normalized gradient descent (NGD) algorithm, which is almost canonical gradient descent (see, e.g., Chapter 14 in~\cite{shalev2014understanding}) except gradients are normalized such that the algorithm applies uniform-size directional updates given by a fixed learning rate $\eta > 0$.
While NGD may not be the most appropriate optimization algorithm in some applications, we use it as a theoretical benchmark which allows us to understand optimization complexity; further details regarding NGD can be found in Appendix~\ref{appen:prelim&reformSLQC}.
Indeed, the convergence guarantees of NGD for SLQC functions are similar to those of Gradient Descent for convex functions.
\begin{prop}[{Theorem~4.1},\cite{hazan2015beyond}] \label{prop:NGDiterations}
Let $f: \mathbb{R}^{d} \rightarrow \mathbb{R}$, $\theta_{1} \in \mathbb{R}^{d}$, and $\theta^{*} = \argmin_{\theta \in \mathbb{R}^{d}} f(\theta)$. If $f$ is $(\epsilon, \kappa, \theta^{*})$-SLQC at $\theta$ for every $\theta \in \mathbb{R}^{d}$, then running the NGD algorithm with learning $\eta = \epsilon/\kappa$ for number of iterations $T \geq \kappa^{2}\|\theta_{1} - \theta^{*}\|^{2}/\epsilon^{2}$ achieves $\min\limits_{t=1, \ldots, T} f(\theta_{t}) - f(\theta^{*}) \leq \epsilon$. 
\end{prop}
For an $(\epsilon,\kappa,\theta_0)$-SLQC function, a smaller $\epsilon$ provides better optimality guarantees. Given $\epsilon>0$, smaller $\kappa$ leads to faster optimization as the number of required iterations increases with $\kappa^2$.
Finally, by using projections, NGD can be easily adapted to work over convex and closed sets (e.g., $\mathbb{B}(\theta_0,r)$ for some $\theta_0\in\mathbb{R}^d$ and $r>0$).

\subsection{SLQC Parameters of the $\alpha$-risk}
With the above SLQC preliminaries in hand, we start quantifying the SLQC parameters of the $\alpha$-risk, $R_{\alpha}$.
It can be shown that for $\alpha \in (0,\infty]$, $R_{\alpha}$ 
is $C_{d}(r,\alpha)$-Lipschitz in $\theta \in \mathbb{B}_{d}(r)$ where
\begin{align} \label{eq:alpharisklipintheta}
C_{d}(r,\alpha) := \begin{cases}
			\sqrt{d} \sigma(r\sqrt{d})\sigma(-r\sqrt{d})^{1-1/\alpha}, & \alpha \in (0,1] \\
            \sqrt{d} \left(\frac{\alpha-1}{2 \alpha - 1} \right)^{1-1/\alpha} \left(\frac{\alpha}{2\alpha -1} \right), & \alpha \in (1,\infty] \quad \text{and} \quad r\sqrt{d} \geq \log{\left(1 - 1/\alpha \right)} \\
            \sqrt{d} \sigma(r\sqrt{d})\sigma(-r\sqrt{d})^{1-1/\alpha}, & \alpha \in (1,\infty] \quad \text{and} \quad r\sqrt{d} < \log{\left(1 - 1/\alpha \right)}. \\
		 \end{cases}
\end{align}
Thus, in conjunction with Theorem~\ref{Thm:SLQClessthan1}, Corollary~\ref{cor:convexforalphagreaterthan1}, and a result by Hazan \textit{et al.} in \cite{hazan2015beyond} (after Definition 3), we provide the following result that explicitly characterizes the SLQC parameters of the $\alpha$-risk $R_{\alpha}$ for two separate ranges of $\alpha$ near $1$.
\begin{prop} \label{prop:SLQCpriortoevolution}
Suppose that $\Sigma \succ 0$ and $\theta_0 \in \mathbb{B}_{d}(r)$ is fixed. We have one of the following:
\begin{itemize}
    \item If $r\sqrt{d} < \arcsinh{(1/2)}$, then, for every $\epsilon>0$, $R_{\alpha}$ is $(\epsilon,C_{d}(r,\alpha),\theta_0)$-SLQC at $\theta$ for every $\theta\in\mathbb{B}_{d}(r)$ for $\alpha \in \left(0,(e^{2r\sqrt{d}}-e^{r\sqrt{d}})^{-1}\right]$ where $C_{d}(r,\alpha)$ is given in \eqref{eq:alpharisklipintheta};
    \item Otherwise, for every $\epsilon>0$, $R_{\alpha}$ is $(\epsilon,C_{d}(r,\alpha),\theta_0)$-SLQC at $\theta$ for every $\theta\in\mathbb{B}_{d}(r)$ for $\alpha \in (0, 1]$. 
\end{itemize}
\end{prop}

Thus, by Proposition~\ref{prop:NGDiterations} and~\eqref{eq:alpharisklipintheta}, the number of iterations of NGD, $T_{\alpha}$, tends to infinity as $\alpha$ tends to zero.
This consequence of the result feels somewhat counterintuitive because one would expect that increasing convexity ($R_{\alpha}$ becomes ``more'' strongly convex in $\theta$ as $\alpha$ decreases, see Theorem~\ref{Thm:SLQClessthan1} and Figure~\ref{fig:fourlandscapes}) would improve the convergence rate.
However, the number of iterations of NGD tends to infinity as $\alpha$ tends to zero because the Lipschitz constant of $R_{\alpha}$, $C_{d}(r,\alpha) = \kappa$ blows up. 
This phenomenon of the Lipschitz constant worsening the convergence rate is not merely a feature of the SLQC theory surrounding NGD. 
It is also present in convergence rates for SGD optimizing convex functions, e.g., see Theorem 14.8 in~\cite{shalev2014understanding}.
Therefore, we find that there exists a trade-off between the desired strong-convexity of $R_{\alpha}$ and the optimization complexity of NGD.

Next, we quantify the evolution of the SLQC parameters of $R_{\alpha}$ both in the small radius regime 
and in the large radius regime. 
%
Since $R_{\alpha}$ tends more towards the probability of error (expectation of $0$-$1$ loss) as $\alpha$ approaches infinity, we find that the SLQC parameters deteriorate and the optimization complexity of NGD increases as we increase $\alpha$.
\begin{figure}
    \centering
    \centerline{\includegraphics[width=.75\linewidth]{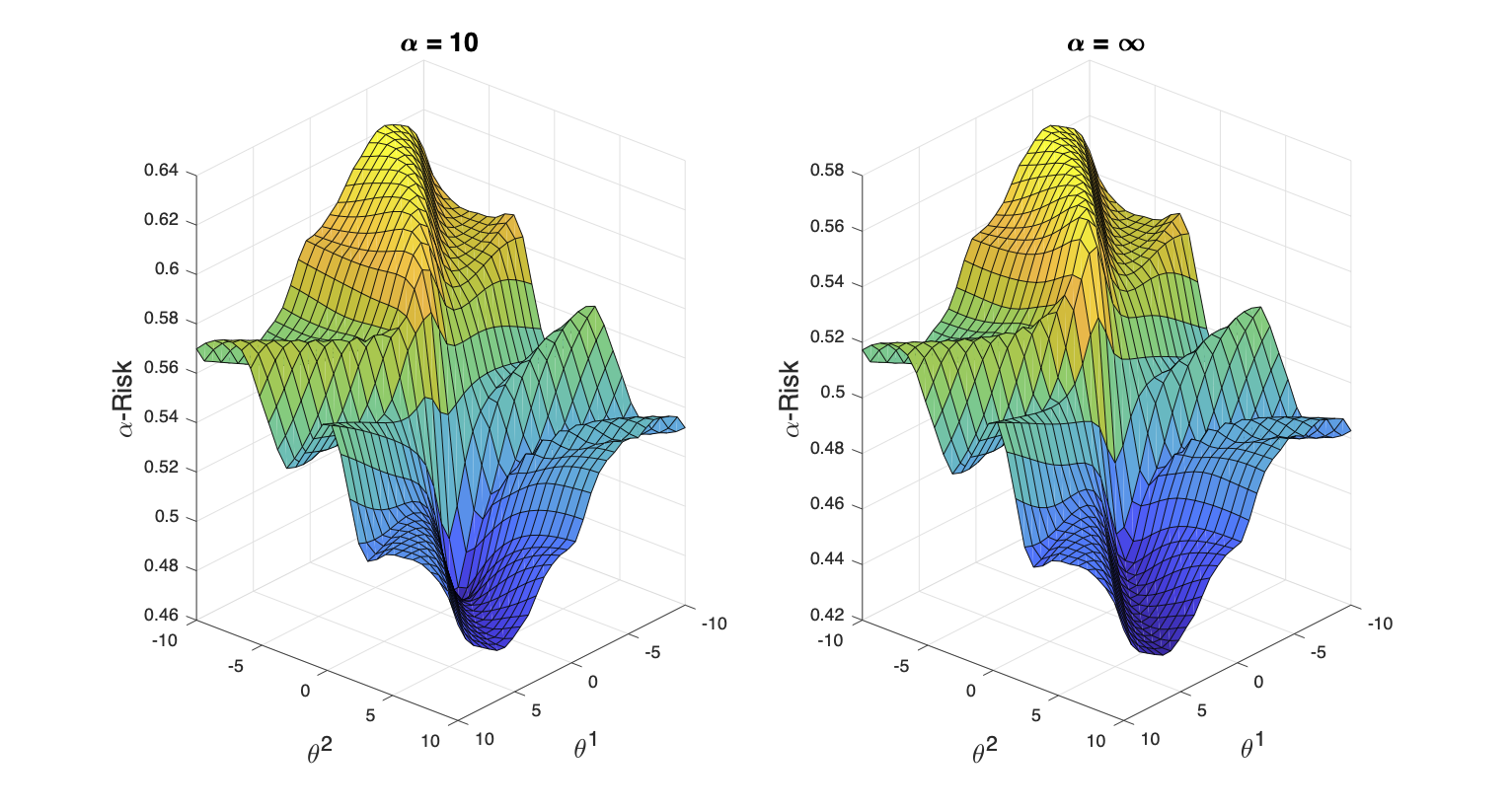}}
    \caption{An illustration of the saturation phenomenon of $\alpha$-loss ($R_{\alpha}$ for $\alpha = 10, \infty$) in the logistic model for a 2D-GMM with $\mathbb{P}[Y=1]=\mathbb{P}[Y=-1]$, $\mu_{X|Y=-1} =(-.91, .50)^{\intercal}$, $\mu_{X|Y=1} = (-.27,.20)^{\intercal}$, $\Sigma = [1.38, .55; .55, 2.18]$. Note the small difference, uniformly over the parameter space, between $R_{10}$ and $R_{\infty}$.}
    \label{fig:saturation}
\end{figure}
Fortunately, in the logistic model, $\alpha$-loss exhibits a saturation effect whereby relatively small values of $\alpha$ resemble the landscape induced by $\alpha = \infty$.
In order to quantify this effect, we state the following two Lipschitz inequalities which will also be instrumental for our main SLQC result.
\begin{lemma} \label{lemma:inversealphalip}
If $\alpha, \alpha' \in [1,\infty]$, then for all $\theta \in \mathbb{B}_{d}(r)$,
\begin{align}
|R_{\alpha}(\theta) - R_{\alpha'}(\theta)| \leq L_{d}(\theta) \left|\frac{\alpha-\alpha'}{\alpha \alpha'}\right| \quad \text{and} \quad \|\nabla R_{\alpha}(\theta) - \nabla R_{\alpha'}(\theta)\| \leq J_{d}(\theta) \left|\frac{\alpha- \alpha'}{\alpha \alpha'}\right|,
\end{align}
where, 
\begin{align} \label{eq:LdJd}
L_{d}(\theta) := \dfrac{\left(\log{\left(1 + e^{\|\theta\|\sqrt{d}}\right)}\right)^{2}}{2} \quad \text{and} \quad J_{d}(\theta) := \sqrt{d} \log{\left(1+e^{\|\theta \| \sqrt{d}}\right)} \sigma(\|\theta\|\sqrt{d}).
\end{align}
\end{lemma}
%
%
This result is proved in Appendix~\ref{appen:lemmas&mainSLQC}, and it can be applied to illustrate a saturation effect of $\alpha$-loss in the logistic model.
That is, let $\alpha = 10$ and $\alpha' = \infty$, then for all $\theta \in \mathbb{B}_{d}(r)$, we have that
\begin{equation} \label{eq:saturationeffect}
\left|R_{10}(\theta) - R_{\infty}(\theta)\right| \leq \frac{L_{d}(\theta)}{10} \quad \text{and} \quad \left|\nabla R_{10}(\theta) - \nabla R_{\infty}(\theta)\right| \leq \frac{J_{d}(\theta)}{10},
\end{equation}
where $L_{d}(\theta)$ and $J_{d}(\theta)$ are both given in~\eqref{eq:LdJd}.
In words, the pointwise distance between the $\alpha = 10$ landscape and the $\alpha =\infty$ landscape decreases geometrically; for a visual representation see Figure \ref{fig:saturation}.



The saturation effect of $\alpha$-loss suggests that it is unnecessary to work with large values of $\alpha$. In particular, this motivates us to study the evolution of the SLQC parameters of the $\alpha$-risk as we increase $\alpha > 1$.
\begin{theorem}\label{thm:SLQCresult1}
Let $\alpha_{0} \in [1,\infty]$, $\epsilon_{0}, \kappa_{0}>0$, and $\theta_{0}, \theta \in \mathbb{B}_{d}(r)$.  
If $R_{{\alpha_{0}}}$ is $(\epsilon_{0},\kappa_{0},\theta_{0})$-SLQC at $\theta$ and
\begin{equation} \label{eq:radchange}
0 \leq \alpha-\alpha_{0} < \dfrac{\alpha_{0}^{2} \|\nabla R_{\alpha_{0}}(\theta)\|}{2J_{d}(\theta)\left(1 + r \frac{\kappa_{0}}{\epsilon_{0}}\right)}, 
\end{equation}
then $R_{{\alpha}}$ is $(\epsilon,\kappa,\theta_{0})$-SLQC at $\theta$ with 
\begin{equation} \label{eq:thmeps&rho}
\centering \epsilon = \epsilon_{0} +2L_{d}(\theta)\left(\frac{\alpha-\alpha_{0}}{\alpha \alpha_{0}}\right) \text{ } \text{and } \text{ } \dfrac{\epsilon}{\kappa} = \frac{\epsilon_{0}}{\kappa_{0}} \left(1  - \frac{\left(1 +2r\frac{\kappa_{0}}{\epsilon_{0}}\right)J_{d}(\theta)(\alpha-\alpha_{0})}{\alpha\alpha_{0}\|\nabla R_{\alpha_{0}}(\theta)\| - J_{d}(\theta)(\alpha-\alpha_{0})}\right).
\end{equation}
\end{theorem}

The proof of Theorem~\ref{thm:SLQCresult1} can be found in Appendix~\ref{appen:lemmas&mainSLQC}. The crux of the proof is a consideration of two cases, dependent on the location of $\theta \in \mathbb{B}_{d}(r)$ relative to the $\epsilon_{0}$-plane. 
The first case considers $\theta \in \mathbb{B}_{d}(r)$ such that $R_{\alpha_{0}}(\theta) - R_{\alpha_{0}}(\theta_{0}) \leq \epsilon_{0}$ and provides the required increase for $\epsilon$ to capture such points as $\alpha$ increases. 
The second case considers $\theta \in \mathbb{B}_{d}(r)$ such that $R_{\alpha_{0}}(\theta) - R_{\alpha_{0}}(\theta_{0}) > \epsilon_{0}$ and provides the required decrease for $\epsilon/\kappa$ to capture such points as $\alpha$ increases.
The second case is far more geometric than the first one, as it makes use of finer gradient information. As a result, the decrease in $\epsilon/\kappa$ is more closely related to the landscape evolution of $R_{\alpha}$ than the corresponding increase in $\epsilon$.
From a numerical point of view, Proposition~\ref{prop:NGDiterations} implies that reducing the radius of the $\epsilon/\kappa$ ball about $\theta_{0}$ increases the required number of iterations (for optimality), and thus reflects the intuition that increasing $\alpha > 1$ more closely approximates the intractable $0$-$1$ loss. 
While on the contrary, Proposition~\ref{prop:NGDiterations} implies that increasing the value of $\epsilon$ reduces the optimality guarantee itself.
%


We note that the bounds provided in Theorem~\ref{thm:SLQCresult1} are pessimistic, but fortunately, we can improve them by employing a bootstrapping technique - we take infinitesimal steps in $\alpha$ and repeatedly apply the bounds in Theorem~\ref{thm:SLQCresult1} to derive improved bounds on $\alpha$, $\epsilon$, and $\kappa$.
The following result is the culmination of our analysis regarding the SLQC parameters of the $\alpha$-risk in the logistic model. 
The proof can be found in Appendix~\ref{appen:bootstrappingslqc}. 
\begin{theorem}\label{thm:SLQCbootstrapFTW}
Let $\alpha_{0} \in [1,\infty)$, $\epsilon_{0}, \kappa_{0}>0$, and $\theta_{0}, \theta \in \mathbb{B}_{d}(r)$.  
Suppose that $R_{{\alpha_{0}}}$ is $(\epsilon_{0},\kappa_{0},\theta_{0})$-SLQC at $\theta \in \mathbb{B}_{d}(r)$
and there exists $g_{\theta} > 0$ such that $\|\nabla R_{\alpha'}(\theta)\| > g_{\theta}$ for every $\alpha' \in [\alpha_{0},\infty]$.
Then for every $\lambda \in (0,1)$, $R_{\alpha_{\lambda}}$ is $(\epsilon_{\lambda},\kappa_{\lambda},\theta_{0})$-SLQC at $\theta$ where
\begin{align} \label{eq:bootstrappedalpha}
\alpha_{\lambda} &\coloneqq \alpha_{0} + \lambda \frac{\alpha_{0}^2g_{\theta}}{J_{d}(\theta)\left(1+2r\frac{\kappa_{0}}{\epsilon_{0}}\right)}, 
\end{align}
\begin{align} \label{eq:bootstrappedepskappa}
\centering \epsilon_{\lambda} \coloneqq \epsilon_{0} + 2 \lambda L_{d}(\theta) \left(\frac{\alpha_{\lambda} - \alpha_{0}}{\alpha_{\lambda}\alpha_{0}} \right) \frac{\alpha_{0}^{2}g_{\theta}}{J_{d}(\theta)\left(1+r \frac{\kappa_{0}}{\epsilon_{0}}\right)}, \quad \text{and} \quad \frac{\epsilon_{\lambda}}{\kappa_{\lambda}} > \frac{\epsilon_{0}}{\kappa_{0}}(1-\lambda).
\end{align}
\end{theorem}
We now provide three different interpretations and comments regarding the previous result.
First regarding the SLQC parameters themselves, observe from~\eqref{eq:bootstrappedalpha} that the bound on $\alpha$ is improved over Theorem~\ref{thm:SLQCresult1} as the factor of $2$ in the denominator in~\eqref{eq:radchange} is moved into the parentheses; next, it can be observed (upon plugging in $\alpha_{\lambda}$) that $\epsilon_{\lambda}$ in~\eqref{eq:bootstrappedepskappa} is linear in $\lambda$, which is again an improvement over the first equation in~\eqref{eq:thmeps&rho}; finally, note that the bound on $\epsilon_{\lambda}/\kappa_{\lambda}$ in~\eqref{eq:bootstrappedepskappa} is vastly more tractable and informative than the second expression in~\eqref{eq:thmeps&rho}. 
Thus, bootstrapping the bounds of Theorem~\ref{thm:SLQCresult1} provides strong improvements for all three relevant quantities, $\alpha$, $\epsilon$, and $\kappa$.
Next, regarding the extra assumption for Theorem~\ref{thm:SLQCbootstrapFTW} over Theorem~\ref{thm:SLQCresult1}, i.e., the existence of a lowerbound $g_{\theta}$ on the norm of the gradient $\|\nabla R_{\alpha'}(\theta)\|$ for all $\alpha' \geq \alpha_{0}$, observe that this is equivalent to the requirement that the landscape at $\theta$ does not become "flat" for any $\alpha' \geq \alpha_{0}$.
In essence, this is a distributional assumption in disguise, and it should be addressed in a case-by-case basis.
Finally, regarding the effect of the dimensionality of the feature space, $d$,  on the bounds, we observe that for $\theta \in \mathbb{B}_{d}(r)$ and $d \in \mathbb{N}$ large enough, $J_{d}(\theta) \approx d \|\theta\|$ as given in~\eqref{eq:LdJd}.
Thus in the high-dimensional regime, the bound on $\alpha$, i.e., $\alpha_{\lambda}$, is dominated by $1/d$. 
This implies that the convexity of the landscape worsens as the dimensionality of the feature/parameter vectors $d$ increases.

While a practitioner would ultimately like to approximate the 0-1 loss (captured by $\alpha = \infty$), the bounds presented in Theorem~\ref{thm:SLQCbootstrapFTW} suggest that the optimization complexity of NGD increases as $\alpha$ increases.
Fortunately, $\alpha$-loss exhibits a saturation effect as exemplified in~\eqref{eq:saturationeffect} and Figure~\ref{fig:saturation} whereby smaller values of $\alpha$ quickly resemble the landscape induced by $\alpha = \infty$.
Thus, while the optimization complexity increases as $\alpha$ increases (and increases even more rapidly in the high-dimensional regime), the saturation effect suggests that the practitioner need not increase $\alpha$ too much in order to reap the benefits of the $\infty$-risk.
Therefore, for the logistic model, we ultimately posit that there is a narrow range of $\alpha$ useful to the practitioner and we dub this the "Goldilocks zone"; we explore this theme in the experiments in Section~\ref{sec:Experiments}.

Before this however, we conclude the theoretical analysis of $\alpha$-loss with a study of the empirical $\alpha$-risk, and we provide generalization and optimality guarantees for all $\alpha \in (0,\infty]$.


\section{Generalization and Asymptotic Optimality} \label{sec:generalizationoptimality}
In this section, we provide generalization and asymptotic optimality guarantees for $\alpha$-loss for $\alpha \in (0,\infty]$ in the logistic model by utilizing classical Rademacher complexity tools and the notion of classification-calibration introduced by Bartlett \textit{et al.} in~\cite{bartlett2006convexity}. 
We invoke the same setting and definitions provided in Section~\ref{sec:landscapelogisticmodel}.
In addition, we consider the evaluation of $\alpha$-loss in the finite sample regime. Formally, let $X \in [0,1]^{d}$ be the normalized feature and $Y \in \{-1,+1\}$ the label as before, \textit{and} let $S_{n} = \{(X_{i},Y_{i}) : i = 1,\ldots,n\}$ be the training dataset where, for each $i\in\{1,\ldots,n\}$, the samples $(X_{i},Y_{i})$ are independently and identically drawn according to an unknown distribution $P_{X,Y}$. 
Finally, we let $\hat{R}_{\alpha}$ denote the empirical $\alpha$-risk of \eqref{eq:alphadefLR}, i.e., for each $\theta \in \mathbb{B}_{d}(r)$ we have
\begin{align} \label{eq:alphariskempiricalLR}
\hat{R}_{\alpha}(\theta) = \frac{1}{n} \sum\limits_{i=1}^{n} l^{\alpha}(Y_{i},g_{\theta}(X_{i})).
\end{align}
In the following sections, we consider the generalization capabilities and asymptotic optimality of a predictor $\theta \in \mathbb{B}_{d}(r)$ learned through empirical evaluation of $\alpha$-loss~\eqref{eq:alphariskempiricalLR}.
First, we recall classical results in Rademacher complexity generalization bounds.

\subsection{Rademacher Complexity Preliminaries} \label{section:RadGeneral}
In this section, we provide the main tools we use to derive generalization bounds for $\alpha$-loss in the sequel. The techniques are standard; see Chapter 26 in \cite{shalev2014understanding} for a complete discussion.
First, we recall that the Rademacher distribution is the uniform distribution on the set $\{-1,+1\}$. The Rademacher complexity of a set is as follows.
\begin{definition}
The Rademacher complexity of a nonempty set $A \subset \mathbb{R}^{n}$ is defined as
\begin{equation}
\mathcal{R}(A) := \mathbb{E}\left(\sup\limits_{a \in A} \dfrac{1}{n} \langle \sigma, a \rangle\right), 
\end{equation}
where $\sigma = (\sigma_{1},\sigma_{2},\ldots,\sigma_{n})$ with $\sigma_{1},\sigma_{2},\ldots,\sigma_{n}$ i.i.d. Rademacher random variables. 
\end{definition}
In words, the Rademacher complexity of a set approximately measures the richness of the set through the maximal correlation of its elements with uniformly distributed Rademacher vectors.
The notion of Rademacher complexity can be used to measure the richness of a hypothesis class as established in the following proposition.
\begin{prop}[{\cite[Theorem~26.5]{shalev2014understanding}}] \label{prop:radegeneralization}
Let $\mathcal{H}$ be a hypothesis class.
Assume that $l: \mathcal{X} \times \mathcal{Y} \times \mathcal{H} \rightarrow \mathbb{R}_{+}$ is a bounded loss function, i.e., there exists $D > 0$ such that
for all $h \in \mathcal{H}$ and $(x,y) \in (\mathcal{X},\mathcal{Y})$ we have that $|l(h,(x,y))| \leq D$.
Then, with probability at least $1-\delta$, for all $h \in \mathcal{H}$,
\begin{equation} \label{eq:symmetrization}
\left|R_{l}(h) - \hat{R}_{l}(h)\right| \leq 2 \mathcal{R}(l \circ \mathcal{H} \circ S_{n}) + 4D\sqrt{\dfrac{2\ln{(4/\delta)}}{n}},
\end{equation}
where $R_{l}(h)$ and $\hat{R}_{l}(h)$ denote the true risk and empirical risk of $l$, respectively, and $l \circ \mathcal{H} \circ S_{n}:= \{(l(h,(x_{1},y_{1})),\ldots,l(h,(x_{n},y_{n}))):h\in \mathcal{H}\} \subset \mathbb{R}^{n}$.\footnote{In~\eqref{eq:symmetrization} we present the two-sided version of Theorem 26.5 in~\cite{shalev2014understanding}, which can be readily obtained via the symmetrization technique.}
\end{prop}
For linear predictors, obtaining a bound on $\mathcal{R}(l \circ \mathcal{H} \circ S_{n})$ is feasible; 
we now provide two results (in conjunction with Proposition~\ref{prop:radegeneralization}) necessary to derive a generalization bound for $\alpha$-loss in the logistic model.
\begin{lemma}[{Lemma~26.9},\cite{shalev2014understanding}] \label{lemma:radecontraction}
Let $\tilde{l}_{1}, \ldots, \tilde{l}_{n}:\mathbb{R} \rightarrow \mathbb{R}$ be $r_{0}$-Lipschitz functions with common constant $r_{0} \geq 0$. If $\tilde{l} = (\tilde{l}_{1}, \ldots, \tilde{l}_{n})$ and $A \subset \mathbb{R}^{n}$, then $\mathcal{R}(\tilde{l}(A)) \leq r_{0} \mathcal{R}(A)$, where $\tilde{l}(A) := \{(\tilde{l}_{1}(a_{1}),\ldots,\tilde{l}_{n}(a_{n})):a \in A\}$.
\end{lemma}
The previous result, known as the Contraction Lemma, provides an upperbound on the Rademacher complexity of the composition of a function acting on a set. For our purposes, one can think of $\tilde{l} = (\tilde{l}_{1}, \ldots, \tilde{l}_{n})$ as a margin-based loss function acting on a training set with $n$ samples - this will be further elaborated in the sequel.
The following result provides an upperbound on the Rademacher complexity of the set comprised of inner products between a given parameter vector drawn from a bounded space and the $n$-sample training set.
\begin{lemma}[{Lemma~26.10},\cite{shalev2014understanding}] \label{lemma:raderadiusbound}
Let $x_{1:n} = \left\{x_{1}, \ldots, x_{n}\right\}$ be vectors in $\mathbb{R}^{d}$. Define $\mathcal{H} \circ x_{1:n} = \{(\langle \theta,x_{1}\rangle, \ldots, \langle \theta, x_{n} \rangle):\|\theta\|_{2} \leq r\}$. Then, 
\begin{equation}
\mathcal{R}(\mathcal{H} \circ x_{1:n}) \leq \dfrac{r \max_{i\in[n]} \|x_{i}\|_{2}}{\sqrt{n}}.
\end{equation}
\end{lemma}

With the above Rademacher complexity preliminaries in hand, we now apply these results to derive a generalization bound for $\alpha$-loss in the logistic model.

\subsection{Generalization and Asymptotic Optimality of $\alpha$-loss}
We now present the following Lipschitz inequality for the margin-based $\alpha$-loss (Definition \ref{def:marginbasedalphaloss}) and will be useful in applying Proposition \ref{prop:radegeneralization}. 
It can readily be shown that the margin-based $\alpha$-loss, $\tilde{l}^{\alpha}$ is $C_{r_{0}}(\alpha)$-Lipschitz in $z \in [-r_{0},r_{0}]$ for every $r_{0} >0$, where
\begin{align} \label{eq:marginalphalosslipintheta}
C_{r_{0}}(\alpha) := \begin{cases}
			\sigma(r_{0})\sigma(-r_{0})^{1-1/\alpha}, & \alpha \in (0,1] \\
            \left(\frac{\alpha-1}{2 \alpha - 1} \right)^{1-1/\alpha} \left(\frac{\alpha}{2\alpha -1} \right), & \alpha \in (1,\infty] \quad \text{and} \quad r_{0} \geq \log{\left(1 - 1/\alpha \right)} \\
            \sigma(r_{0})\sigma(-r_{0})^{1-1/\alpha}, & \alpha \in (1,\infty] \quad \text{and} \quad r_{0} < \log{\left(1 - 1/\alpha \right)}.
		 \end{cases}
\end{align}
That is, for $\alpha \in (0,\infty]$ and $z, z' \in [-r_{0},r_{0}]$, we have that $|\tilde{l}^{\alpha}(z) - \tilde{l}^{\alpha}(z')| \leq C_{r_{0}}(\alpha) |z - z'|$; see Lemma~\ref{lemma:marginlip} in Appendix~\ref{appen:generalization} for the proof. 
Lastly, note that for any fixed $r_{0}>0$, $C_{r_{0}}(\alpha)$ is monotonically decreasing in $\alpha$.

With the Lipschitz inequality for $\tilde{l}^{\alpha}$ in hand, we are now in a position to state a generalization bound for $\alpha$-loss in the logistic model.


\begin{theorem}\label{thm:radegeneralization}
If $\alpha\in(0,\infty]$, then, with probability at least $1-\delta$, for all $\theta \in \mathbb{B}_{d}(r)$,
\begin{equation} \label{eq:radegeneral}
\left|R_{\alpha}(\theta) - \hat{R}_{\alpha}(\theta)\right| \leq C_{r\sqrt{d}}\left(\alpha\right) \dfrac{2r\sqrt{d}}{\sqrt{n}} + 4D_{r\sqrt{d}}\left(\alpha\right)\sqrt{\dfrac{2\log{(4/\delta)}}{n}}, 
\end{equation}
where $C_{r\sqrt{d}}\left(\alpha\right)$ is given in~\eqref{eq:marginalphalosslipintheta} and $\displaystyle D_{r\sqrt{d}}\left(\alpha\right) := \frac{\alpha}{\alpha-1}\left(1-\sigma(-r\sqrt{d})^{1-1/\alpha}\right)$.
\end{theorem} 
Note that $D_{r\sqrt{d}}(\alpha)$ is also monotonically decreasing in $\alpha$ for fixed $r \sqrt{d} > 0$. 
Thus, Theorem~\ref{thm:radegeneralization} seems to suggest that generalization improves as $\alpha \rightarrow \infty$. However, because $R_{\alpha}$ and $\hat{R}_{\alpha}$ also monotonically decrease in $\alpha$, it is difficult to reach such a conclusion. 
Nonetheless, Corollary~\ref{cor:saturationgeneralization} in Appendix~\ref{appen:generalization} offers an attempt at providing a unifying comparison between the $\infty$-risk, $R_{\infty}$, and the empirical $\alpha$-risk, $\hat{R}_{\alpha}$.

Lastly, observe that for the generalization result in Theorem~\ref{thm:radegeneralization}, we make no distributional assumptions such as those by Tsybakov, \textit{et. al} in \cite{audibert2007fast}, where they assume the posterior satisfies a \textit{margin} condition. 
Under such an assumption, we observe that faster rates could be achieved, but optimal rates are not the focus of this work.
Nonetheless, the next theorem relies on the assumption that the minimum $\alpha$-risk is attained by the logistic model, i.e., given $\alpha \in (0,\infty]$, suppose that 
\begin{align} \label{eq:alphalinearseparable}
\min\limits_{\theta \in \mathbb{B}_{d}(r)} R_{\alpha}(\theta) = \min\limits_{f:\mathcal{X} \rightarrow \mathbb{R}} R_{\alpha}(f),
\end{align}
where $R_{\alpha}(\theta)$ is given in~\eqref{eq:alphariskdefLR} and $R_{\alpha}(f) = \mathbb{E}[\tilde{l}^{\alpha}(Yf(X))]$ for all measurable $f$.
\begin{theorem}  \label{thm:asymptoticoptimality}
Assume that the minimum $\alpha$-risk is attained by the logistic model, i.e.,~\eqref{eq:alphalinearseparable} holds.
Let $S_{n}$ be a training dataset with $n \in \mathbb{N}$ samples as before.
If for each $n\in\mathbb{N}$, $\hat{\theta}_n^{\alpha}$ is a global minimizer of the associated empirical $\alpha$-risk $\theta \mapsto \hat{R}_{\alpha}(\theta)$, then the sequence $(\hat{\theta}_n^{\alpha})_{n=1}^\infty$ is asymptotically optimal for the $0$-$1$ risk, i.e., almost surely,
\begin{equation}
\lim_{n\to\infty} R(f_{\hat{\theta}_n^{\alpha}}) = R^*,
\end{equation}
where $f_{\hat{\theta}_n^{\alpha}}(x) = \langle \hat{\theta}_n^{\alpha}, x \rangle$ and $R^{*} := \min\limits_{f: \mathcal{X} \rightarrow \mathbb{R}} \mathbb{P}[Y \neq \sign(f(X))]$.
\end{theorem}
In words, setting the optimization procedure aside, utilizing $\alpha$-loss for a given $\alpha \in (0,\infty]$ is asymptotically optimal with respect to the probability of error (expectation of the $0$-$1$ loss).
Observe that the assumption in~\eqref{eq:alphalinearseparable} is a stipulation for the the underlying data-generating distribution, $P_{X,Y}$, in disguise. 
That is, we assume that $P_{X,Y}$ is separable by a linear predictor, which is a global minimizer for the $\alpha$-risk.
In essence, Theorem~\ref{thm:asymptoticoptimality} is a combination of Theorem~\ref{thm:radegeneralization} and classification-calibration.

With the statistical, optimization, and generalization considerations of $\alpha$-loss behind us, we now provide experimental results in two canonical settings for $\alpha$-loss in logistic and convolutional-neural-network models. 



\section{Experimental Results} \label{sec:Experiments}
%
%
%
As was first introduced in Section~\ref{sec:syntheticexp}, in this section we further experimentally evaluate the efficacy of $\alpha$-loss in the following two canonical scenarios: 
\newline
(i) \textbf{Noisy labels:} the classification algorithm is trained on a binary-labeled dataset that suffers from symmetric noisy labels, and it attempts to produce a model which achieves strong performance on the clean test data. 
\newline 
(ii) \textbf{Class imbalance:} the classification algorithm is trained on a binary-labeled dataset that suffers from a class imbalance, and it attempts to produce a model which achieves strong performance on the balanced test data. 
\newline
\noindent Our hypotheses are as follows: for setting (i), tuning $\alpha > 1$ (away from log-loss) improves the robustness of the trained model to symmetric noisy labels; for setting (ii), tuning $\alpha < 1$ (again away from log-loss) improves the sensitivity of the trained model to the minority class. In general, we experimentally validate both hypotheses.

In our experimental procedure, we use the following common image datasets: MNIST~\cite{MNIST}, Fashion MNIST~\cite{xiao2017FashionMNIST}, and CIFAR-10~\cite{krizhevsky2014cifar}. While these datasets have predefined training and test sets, we present binary partitions of these datasets for both settings in the main text, in alignment with our theoretical investigations of $\alpha$-loss for binary classification problems; in Appendix~\ref{appen:multiclassexperiments}, we present multiclass symmetric noise experiments for the MNIST and FMNIST datasets.
%
Regarding the binary partitions themselves, we chose classes which are visually similar in order to increase the difficulty of the classification task. 
Specifically, for MNIST we used a binary partition on the \textit{1} and \textit{7} classes, for FMNIST we used a binary partition on the \textit{T-Shirt} and \textit{Shirt} classes, and finally for our binary experiments on CIFAR-10 we used a binary partition on the \textit{Cat} and \textit{Dog} classes.

All code is written in PyTorch, version 1.30~\cite{paszke2019pytorch}. 
Architectures learning CIFAR are trained with GPUs, while the architectures learning MNIST and FMNIST are both trained with CPUs. 
Throughout, we consider two broad classes of architectures: logistic regression (LR) and convolutional neural networks (CNNs) with one or two fully connected layers preceded by varying convolutional layer depths ($2$, $3$, $4$, and $6$) such that we obtain the shorthand CNN X$+$Y where X is one of $2$, $3$, $4$, or $6$ and Y is one of $1$ or $2$.
%
%
For all architectures learning CIFAR, we additionally use a sigmoid at the last layer for smoothing.
For each set of experiments, we randomly fix a seed, and for each iteration we reinitialize a new architecture with randomly selected weights.
We use softmax activation to generate probabilities over the labels, and we evaluate the model's soft belief using $\alpha$-loss on a one-hot-encoding of the training data.

All (dataset, architecture) tuples were trained with the same optimizer, vanilla SGD, with fixed learning rates.
In order to provide the fairest comparison to log-loss ($\alpha = 1$), for each (dataset, architecture) tuple we select a fixed learning rate from the set $\{10^{-4}, 5 \times 10^{-4}, 10^{-3}, 5 \times 10^{-3}, 10^{-2}, 5 \times 10^{-2}, 10^{-1}\}$ which provides the highest validation accuracy for a model trained with log-loss. 
Then for the chosen (dataset, architecture) tuple, we train $\alpha$-loss for each value of $\alpha$ using this fixed learning rate.
Regarding the optimization of $\alpha$-loss itself which is parameterized by $\alpha \in (0,\infty]$, in general we find that searching over $\alpha \in [.8,8]$ for noisy labels and $\alpha \in [.8,4]$ for class imbalances is sufficient, and we typically do so in step-sizes of $0.1$ or $0.05$ (near $\alpha = 1$) or a step-size of $1$ (when $\alpha > 1$). This is in line with our earlier theoretical discussions regarding the "Goldilocks zone" of $\alpha$-loss, i.e., the gradient explosion for very small values of $\alpha$, the increased difficulty of optimization for large values of $\alpha$, and the fact that relatively small values of $\alpha$ closely approximate the $\infty$-loss.

For all experiments, we employ a training batch size of $128$ samples. 
For all experiments on the MNIST and FMNIST datasets, training was allowed to progress for $50$ epochs; for all experiments on the CIFAR-10 dataset, training was allowed to progress for $120$ epochs - convergence for all values of $\alpha$ was ensured for both choices. 
Lastly, for each architecture we re-run each experiment $10$ times and report the average test accuracies calculated according to the relative accuracy gain, which we rewrite for our experimental setting as 
\begin{align} \label{eq:relativeaccuracygain}
\text{relative acc gain \%} = \frac{|\alpha\text{-loss}\text{ acc} - \text{log-loss}\text{ acc}|}{\text{log-loss}\text{ acc}} \times 100,
\end{align}
where we use $\text{acc}$ to denote test accuracy.
Also note that $\alpha^{*}$ is chosen as the $\alpha$ over the search range which maximizes the average test accuracy of its trained models.
For more details regarding architecture configurations (i.e., CNN channel sizes, kernel size, etc) and general experiment details, we refer the reader to 
the code for all of our experiments (including the implementation of $\alpha$-loss), which can be found at~\cite{codelink}.

\subsection{Noisy Labels} \label{sec:experimentsbinarynoise}
For the first set of experiments, we evaluate the robustness of $\alpha$-loss to symmetric noisy labels, and we generate symmetric noisy labels in the binary training data as follows: 
\begin{enumerate}
    \item For each run of an experiment, we randomly select $0$-$40\%$ of the training data in increments of $10\%$. 
    \item For each training sample in the randomly selected group, we flip the label of the selected training example to the other class.
\end{enumerate}
Note that for all symmetric noisy label experiments we keep the test data clean, i.e., we do not perform label flips on the test data.  
Thus, these experiments address the scenario where training data is noisy and test data is clean.
Also note that during our $10$-iteration averaging for each accuracy value presented in each table, we are also randomizing over the symmetric noisy labels in the training data.


The results on the binary MNIST dataset (composed of classes \textit{1} and \textit{7}), binary FMINIST dataset (composed of classes \textit{T-Shirt} and \textit{Shirt}), and binary CIFAR-10 (composed of classes \textit{Cat} and \textit{Dog}) are presented in Tables~\ref{table:MNISTsymmetricnoisylabels},~\ref{table:FMNISTsymmetricnoisylabels},~and~\ref{table:CIFARsymmetricnoisylabels}, respectively.
As stated previously, in order to report the fairest comparison between log-loss and $\alpha$-loss, we first find the optimal fixed learning rate for log-loss from our set of learning rates (given above), then we train each chosen architecture with $\alpha$-loss for all values of $\alpha$ also with this found fixed learning rate.
%
Following this procedure, for the binary MNIST dataset, we trained both the LR and CNN 2$+$2 architectures with a fixed learning rate of $10^{-2}$; for the binary FMNIST dataset, we trained the LR and CNN 2$+$2 architectures with fixed learning rates of $10^{-4}$ and $5 \times 10^{-3}$, respectively; for the binary CIFAR-10 dataset, we trained the CNN 2$+$1, 3$+$2, 4$+$2, and 6$+$2 architectures with fixed learning rates of $10^{-2}$, $10^{-1}$, $5 \times 10^{-2}$, and $10^{-1}$, respectively.

Regarding the results presented in Tables~\ref{table:MNISTsymmetricnoisylabels},~\ref{table:FMNISTsymmetricnoisylabels},~and~\ref{table:CIFARsymmetricnoisylabels}, in general we find for $0\%$ label flips (from now on referred to as baseline) the extra $\alpha$ hyperparameter does not offer significant gains over log-loss in the test results for each (dataset, architecture) tuple. 
However once we start to increase the percentage of label flips, we immediately find that $\alpha^{*}$ increases greater than $1$ (log-loss). 
Indeed for each (dataset, architecture) tuple, we find that as the number of symmetric label flips increases, training with $\alpha$-loss for a value of $\alpha > 1$ increases the test accuracy on clean data, often significantly outperforming log-loss.
Note that this performance increase induced by the new $\alpha$ hyperparameter is not monotonic as the number of label flips increases, i.e., there appears to be a noise threshold past which the performance of all losses decays, but this occurs for very high noise levels, which are not usually present in practice.
Recalling Section~\ref{sec:syntheticexp}, the strong performance of $\alpha$-loss for $\alpha > 1$ on binary symmetric noisy training labels can intuitively be accounted for by the quasi-convexity of $\alpha$-loss in this regime, i.e., the reduced sensitivity to outliers.
Thus, we conclude that the results in Tables~\ref{table:MNISTsymmetricnoisylabels}, \ref{table:FMNISTsymmetricnoisylabels}, and~\ref{table:CIFARsymmetricnoisylabels} on binary MNIST, FMNIST, and CIFAR-10, respectively, 
indicate that practitioners should employ $\alpha$-loss for $\alpha > 1$ when training robust architectures to combat against binary noisy training labels.
%
%
Lastly, we report two experiments for multiclass symmetric noisy training labels in Appendix~\ref{appen:multiclassexperiments}. In short, we find similar robustness to noisy labels for $\alpha > 1$, but we acknowledge that further empirical study of $\alpha$-loss on multiclass datasets is needed.

\begin{table}[]
\begin{center}
\begin{tabular}{cccccc}
\hlineB{2.5}
Architecture & Label Flip \% & LL Acc \%  & $\alpha^{*}$ Acc \% & $\alpha^{*}$    & Rel Gain \% \\ \hlineB{2.5}
& 0             & 99.26 & 99.26    & 0.95,1    & 0.00       \\ \cline{2-6} 
& 10            & 99.03 & 99.13    & 6         & 0.10      \\ \cline{2-6} 
LR           & 20            & 98.65  & 99.03    & 7         & 0.39       \\ \cline{2-6} 
& 30            & 97.89 & 98.96    & 3.5       & 1.10      \\ \cline{2-6} 
& 40            & 92.10 & 98.53     & 8         & 6.98       \\ \clineB{1-6}{2}
& 0             & 99.83 & 99.84    & 4-8 & $0.01$  \\ \cline{2-6} 
& 10            & 95.27 & 99.68    & 6,7       & 4.63       \\ \cline{2-6} 
CNN 2$+$2      & 20            & 87.41 & 98.72    & 8         & 12.94      \\ \cline{2-6} 
& 30            & 77.56 & 87.86    & 8         & 13.28     \\ \cline{2-6} 
& 40            & 62.89 & 66.10    & 8         & 5.12       \\ \hlineB{2}
\end{tabular}
\caption{Symmetric binary noisy label experiment on \textbf{MNIST} classes \textit{1} and \textit{7}. Note that LL Acc and $\alpha$ Acc stand for log-loss accuracy and $\alpha$-loss accuracy for $\alpha^{*}$, respectively and that Rel \% Gain is calculated according to~\eqref{eq:relativeaccuracygain}.
Also note that each reported accuracy is averaged over 10 runs.} 
\label{table:MNISTsymmetricnoisylabels}
\end{center}
\end{table}

\begin{table}[]
\begin{center}
\begin{tabular}{cccccc}
\hlineB{2.5}
 Architecture & Label Flip \% & LL Acc \%  & $\alpha^{*}$ Acc \% & $\alpha^{*}$ & Rel Gain \% \\ \hlineB{2.5}
& 0             & 84.51  & 84.78     & 1.5    & 0.32       \\ \cline{2-6} 
 & 10            & 83.80   & 84.41    & 2      & 0.72    \\ \cline{2-6} 
LR           & 20            & 83.11 & 83.94     & 2.5    & 1.01      \\ \cline{2-6} 
& 30            & 81.29 & 83.43    & 3      & 2.63      \\ \cline{2-6} 
& 40            & 74.39  & 92.02    & 8      & 23.69     \\ \clineB{1-6}{2}
& 0             & 86.96 & 87.19     & 1.1    & 0.27       \\ \cline{2-6} 
& 10            & 81.14  & 83.74    & 5      & 3.20      \\ \cline{2-6} 
CNN 2$+$2      & 20            & 72.96 & 78.00       & 8      & 6.93       \\ \cline{2-6} 
& 30            & 66.17  & 69.21    & 8      & 4.59       \\ \cline{2-6} 
& 40            & 57.90 & 58.56     & 3      & 1.15       \\ \hlineB{2}
\end{tabular}
\caption{Symmetric noisy label experiment on \textbf{FMNIST} classes \textit{T-Shirt} and \textit{Shirt}.}
\label{table:FMNISTsymmetricnoisylabels}
\end{center}
\end{table}

\begin{table}[]
\begin{center}
\begin{tabular}{cccccc}
\hlineB{2.5}
Architecture & Label Flip \% & LL Acc \%  & $\alpha^{*}$ Acc \% & $\alpha^{*}$ & Rel Gain \% \\ \hlineB{2.5}
              & 0             & 80.59 & 80.68    & 0.99   & 0.11        \\ \cline{2-6} 
              & 10            & 79.61 & 79.89    & 1.1    & 0.35        \\ \cline{2-6} 
 CNN 2$+$1      & 20            & 77.01 & 77.15     & 0.99   & 0.19        \\ \cline{2-6} 
              & 30            & 73.67 & 74.78    & 2.5    & 1.51        \\ \cline{2-6} 
                                       & 40            & 63.54  & 68.12     & 4      & 7.21        \\ \clineB{1-6}{2}
                                       & 0             & 85.80 & 85.80    & 1      & 0.00        \\ \cline{2-6} 
                                       & 10            & 82.92 & 83.15    & 0.99   & 0.28        \\ \cline{2-6} 
                          CNN 3$+$2      & 20            & 77.61 & 80.88    & 3      & 4.21        \\ \cline{2-6} 
                                       & 30            & 69.53  & 76.72     & 5      & 10.34       \\ \cline{2-6} 
               & 40            & 59.44 & 67.19    & 6      & 13.04       \\ \clineB{1-6}{2}
                                       & 0             & 87.49 & 87.59     & 0.9    & 0.12        \\ \cline{2-6} 
                                      & 10            & 83.65 & 84.69     & 1.2    & 1.25        \\ \cline{2-6} 
                          CNN 4$+$2      & 20            & 78.96  & 81.39    & 3.5    & 3.07        \\ \cline{2-6} 
                                       & 30            & 69.24  & 75.56     & 6      & 9.13        \\ \cline{2-6} 
                                       & 40            & 59.12 & 64.53    & 8      & 9.15        \\ \clineB{1-6}{2}
                                       & 0             & 87.31  & 87.93    & 1.2    & 0.70        \\ \cline{2-6} 
                                       & 10            & 84.91 & 85.33    & 2      & 0.49        \\ \cline{2-6} 
                         CNN 6$+$2      & 20            & 78.92  & 81.80    & 6      & 3.64        \\ \cline{2-6} 
                                       & 30            & 68.88 & 77.20      & 7      & 12.09       \\ \cline{2-6} 
                                       & 40            & 58.54 & 65.16     & 7      & 11.32       \\ \hlineB{2}
\end{tabular}
\caption{Symmetric binary noisy label experiment on \textbf{CIFAR-10} classes \textit{Cat} and \textit{Dog}.}
\label{table:CIFARsymmetricnoisylabels}
\end{center}
\end{table}

\subsection{Class Imbalance}
For the second set of experiments, we evaluate the sensitivity of $\alpha$-loss to class imbalances, and we generate binary class imbalances in the training data as follows: 
\begin{enumerate}
    \item Given a dataset, select two classes, Class 1 and Class 2, and generate baseline $50/50$ data where $\left|\text{Class 1}\right| = \left|\text{Class 2}\right| = 2500$ training samples.
    Throughout all experiments ensure that $\left|\text{Class 1}\right| + \left|\text{Class 2}\right| = 5000$ randomly drawn training samples.
    \item Starting at the baseline ($2500/2500$) and drawing from the available training samples in each dataset when necessary, increase the number of training samples of Class 1 by $500$, $1000$, $1500$, $2000$, and $2250$ and reduce the number of training samples of Class 2 by the same amounts in order to generate training sample splits of $60/40$, $70/30$, $80/20$, $90/10$, and $95/5$, respectively.
    \item Repeat the previous step where the roles of Class 1 and Class 2 are reversed.
\end{enumerate}
Note that the test set is balanced for all experiments with 2000 test samples (1000 for each class). 
Thus, these experiments address the scenario where training data is imbalanced and the test data is balanced.
Also note that during our $10$-iteration averaging for each accuracy value presented in each table, we are also randomizing over the training samples present in each class imbalance split, according to the procedure above.

The results on binary FMNIST (composed of classes \textit{T-Shirt} and \textit{Shirt}) and binary CIFAR-10 (composed of classes \textit{Cat} and \textit{Dog}) are presented in Tables~\ref{table:FMNISTclassimbalanceLR},~\ref{table:CIFARclassimbalance4+2},~and~\ref{table:CIFARclassimbalance6+2}.
For this set of experiments, note that $\alpha^{*}$ is the optimal $\alpha \in [0.8,4]$ (in our search set) which maximizes the average test accuracy of the minority class, and also note that there are slight test accuracy discrepancies between the baselines in the symmetric noisy labels and class imbalance experiments because of the reduced training and test set size for the class imbalance experiments. 
As before, for the binary FMNIST dataset, we trained the LR and CNN 2$+$2 architectures with fixed learning rates of $10^{-4}$ and $5 \times 10^{-3}$, respectively; for the binary CIFAR-10 dataset, we trained the CNN 2$+$1, 3$+$2, 4$+$2, and 6$+$2 architectures with fixed learning rates of $10^{-2}$, $10^{-1}$, $5 \times 10^{-2}$, and $10^{-1}$, respectively.

In general, we find that the minority class is almost always favored by the smaller values of $\alpha$, i.e., we typically have that $\alpha^{*} < 1$.
Further, we observe that as the percentage of class imbalance increases, the relative accuracy gain on the minority class typically increases through training with $\alpha$-loss. 
This aligns with our intuitions articulated in Section~\ref{sec:syntheticexp} regarding the benefits of "stronger" convexity of $\alpha$-loss when $\alpha < 1$ over log-loss ($\alpha = 1$), particularly when the practitioner desires models which are more sensitive to outliers.
Nonetheless, sometimes there does appear to exist a trade-off between how well learning the majority class influences predictions on the minority class, see e.g., recent work in the area of \textit{stiffness} by Fort \textit{et al.}~\cite{fort2019stiffness}. 
This is a possible explanation for why $\alpha < 1$ is not always preferred for the minority class, e.g., $30\%$ and $40\%$ imbalance in Table~\ref{table:CIFARclassimbalance4+2} when \textit{Dog} is the minority class.
Thus we conclude that the results in Tables~\ref{table:FMNISTclassimbalanceLR},~\ref{table:CIFARclassimbalance4+2},~and~\ref{table:CIFARclassimbalance6+2}, 
on binary FMNIST and CIFAR-10, respectively, indicate that practitioners 
should employ $\alpha$-loss (typically) for $\alpha < 1$ when training architectures to be sensitive to the minority class in the training data.
%

\begin{table}[]
\begin{center}
\begin{tabular}{cccccccccc}
\cline{3-8}
                                & & \multicolumn{3}{c}{Log-Loss}  &  \multicolumn{3}{c}{$\alpha$-Loss} &  &                   \\ \hlineB{2.5}
Imb \%              & Min     & Min Acc \%     & Ov Acc \%  & LL-$\text{F}_{1}$  & Min Acc  \%    & Ov Acc \% & $\alpha^{*}$-$\text{F}_{1}$ & $\alpha^{*}$ & Rel Gain \% \\ \hlineB{2.5}
\multirow{2}{*}{50}  & T-Shirt      & 85.4          & 84.31     & 0.8448      & 85.7      & 84.17 & 0.8441        & 1.5    & 0.35      \\ \clineB{2-10}{.5}
                      & Shirt   & 83.2       & 84.31            & 0.8413      & 83.4.     & 84.33 & 0.8418 & 0.85   & 0.24      \\ \hlineB{2}
\multirow{2}{*}{40}  & T-Shirt & 80.0        & 83.68            & 0.8306      & 80.2.   & 83.73 &  0.8313      & 1.1    & 0.25      \\ \clineB{2-10}{.5} 
                      & Shirt   & 77.7       & 83.88            & 0.8282      & 77.7    & 83.90 &  0.8284      & 0.99   & 0.00      \\ \hlineB{2}
\multirow{2}{*}{30}  & T-Shirt & 72.9      & 81.89              & 0.8010      & 73.0         & 81.88      &  0.8011 & 0.99   & 0.14       \\ \clineB{2-10}{.5}
                      & Shirt   & 70.8       & 82.04            & 0.7977      & 72.3        & 82.52       & 0.8053       & 0.8    & 2.12       \\ \hlineB{2}
\multirow{2}{*}{20}  & T-Shirt & 60.9       & 77.97             & 0.7344      & 61.7       & 78.20        & 0.7389  & 0.8    & 1.31       \\ \clineB{2-10}{.5}
                      & Shirt   & 63.1       & 79.81            & 0.7576      & 64.5        & 80.40       & 0.7669       & 0.8    & 2.22       \\ \hlineB{2}
\multirow{2}{*}{10}  & T-Shirt & 43.0       & 70.50             & 0.5931      & 45.2        & 71.50       & 0.6133  & 0.8    & 5.12       \\ \clineB{2-10}{.5}
                      & Shirt   & 55.2       & 76.97            & 0.7056      & 56.0         & 77.25      & 0.7111       & 0.8    & 1.45       \\ \hlineB{2}
\multirow{2}{*}{5} & T-Shirt & 24.6      & 61.85                & 0.3920      & 26.0         & 62.54      & 0.4097       & 0.8    & 5.69      \\ \clineB{2-10}{.5}
                      & Shirt   & 47.5       & 73.52            & 0.6421      & 47.6        & 73.48       & 0.6422        & 0.8    & 0.21      \\ \hlineB{2}
\end{tabular}
\caption{Binary \textbf{FMNIST} Logistic Regression Imbalance Experiments on the \textit{T-Shirt} and \textit{Shirt} classes.
Note that LL-$\text{F}_{1}$ corresponds to the $\text{F}_{1}$ score of log-loss on the imbalanced class; similarly $\alpha^{*}$-$\text{F}_{1}$ corresponds to the $\text{F}_{1}$ score of $\alpha^{*}$-loss on the imbalanced class. See Appendix~\ref{sec:f1review} for a brief review of the definition of the $\text{F}_{1}$ score.
The relative \% gain is defined as the relative percent gain~\eqref{eq:relativeaccuracygain} on the average minority class accuracy (on test data) of models trained with log-loss vs. the average minority class accuracy of models trained with $\alpha$-loss. Note that Ov $=$ Overall. Lastly, observe that for the baseline ($50\%$ imbalance) experiments, we present the accuracy and $\alpha^{*}$ for both classes.}
\label{table:FMNISTclassimbalanceLR}
\end{center}
\end{table}

\begin{table}[]
\begin{center}
\begin{tabular}{cccccccccc}
\cline{3-8}
                                & & \multicolumn{3}{c}{Log-Loss}  &  \multicolumn{3}{c}{$\alpha$-Loss} &  &                   \\
 \hlineB{2.5}
Imb \%                   & Min                  & Min Acc \%     & Ov Acc \% & LL-$\text{F}_{1}$   & Min Acc \%     & Ov Acc \%  & $\alpha^*$-$\text{F}_{1}$        & $\alpha^{*}$               & Rel Gain \%          \\ \hlineB{2.5}
\multirow{2}{*}{50}     & Cat                  & 83.7       & 83.48          & 0.8352          & 87.2        & 83.86       &  0.8438                     & 1.1                  & 4.18             \\ \clineB{2-10}{.5}                                                          
                      & Dog                  & 83.3       & 83.48            & 0.8345          & 86.1        & 84.06       & 0.8438                      & 0.99      & 3.36   \\\hlineB{2}
\multirow{2}{*}{40}  & Cat                  & 79.8       & 83.34             & 0.8273          & 82.7        & 83.39        & 0.8327                     & 0.95                 & 3.63               \\ \clineB{2-10}{.5}
                      & Dog                  & 78.4       & 83.85            & 0.8292          & 82.4        & 83.20        & 0.8306                     & 2.5                  & 5.10               \\ \hlineB{2}
\multirow{2}{*}{30}  & Cat                  & 73.0        & 81.98            & 0.8020          & 74.6        & 82.40        & 0.8000                     & 0.99                 & 2.19              \\ \clineB{2-10}{.5}
                      & Dog                  & 72.0        & 82.00           & 0.8091           & 74.9        & 83.18       & 0.8166                     & 1.2                  & 4.03                \\ \hlineB{2}
\multirow{2}{*}{20}  & Cat                  & 64.6       & 78.96             & 0.7543          & 66.2        & 78.85        &  0.7579                    & 0.8                  & 2.48           \\ \clineB{2-10}{.5}
                      & Dog                  & 63.1       & 78.94            & 0.7498          & 65.0         & 79.79       & 0.7628                     & 0.8                  & 3.01               \\ \hlineB{2}
\multirow{2}{*}{10}  & Cat                  & 39.1       & 68.04             & 0.5502          & 41.6        & 68.88        & 0.5721

                     & 0.9                  & 6.39               \\ \clineB{2-10}{.5}
                      & Dog                  & 42.1       & 70.03            & 0.5842          & 48.5        & 72.53        & 0.6384                     & 0.8                  & 15.20              \\ \hlineB{2}
\multirow{2}{*}{5} & Cat    & \hphantom{0}0.0            & 50.00            & 0.0000           & \hphantom{0}9.6 & 54.48      & 0.1742                       & 0.8                  & $\infty$                  \\ \clineB{2-10}{.5}
                      & Dog                  & 10.0         & 54.94         & 0.1816            & 23.2        & 61.31        & 0.3749                     & 0.8                  & 132.00              \\ \hlineB{2}
\end{tabular}
\caption{Binary \textbf{CIFAR-10} CNN 4$+$2 Imbalance Experiments on \textit{Cat} and \textit{Dog} classes. 
Note that LL-$\text{F}_{1}$ corresponds to the $\text{F}_{1}$ score of log-loss on the imbalanced class; similarly $\alpha^{*}$-$\text{F}_{1}$ corresponds to the $\text{F}_{1}$ score of $\alpha^{*}$-loss on the imbalanced class.
Note that due to our calculation of Rel \% Gain that division by 0 is $\infty$, and thus absolute \% gain for the minority
class \textit{Cat} at a 5\% imbalance is 9.6\%.}
\label{table:CIFARclassimbalance4+2}
\end{center}
\end{table}

\begin{table}[]
\begin{center}
\begin{tabular}{cccccccccc}
\cline{3-8}
                                & & \multicolumn{3}{c}{Log-Loss}  &  \multicolumn{3}{c}{$\alpha$-Loss} &  &                   \\ \hlineB{2.5}
Imb \%                   & Min      & Min Acc \%    & Ov Acc \% & LL-$\text{F}_{1}$   & Min Acc \%     & Ov Acc \%  & $\alpha^{*}$-$\text{F}_{1}$   & $\alpha^{*}$               & Rel Gain \%          \\ \hlineB{2.5}
\multirow{2}{*}{50}  & Cat          & 84.4       & 84.30        & 0.8432       & 85.2        & 84.93         & 0.8497                    & 0.99                 & 0.95            \\ \clineB{2-10}{.5}
                      & Dog         & 84.1       & 84.30        & 0.8427       & 87.0         & 83.91        & 0.8439                       & 2                    & 3.45               \\ \hlineB{2}
\multirow{2}{*}{40}  & Cat          & 80.3       & 83.79        & 0.8320       & 82.4        & 84.87         & 0.8449                       & 0.8                  & 2.62                \\ \clineB{2-10}{.5}
                      & Dog         & 81.2       & 84.91        & 0.8433       & 84.0         & 84.83        & 0.8470                       & 0.9                  & 3.45               \\ \hlineB{2}
\multirow{2}{*}{30}  & Cat          & 74.2       & 82.72        & 0.8111       & 78.2        & 83.32         & 0.8242                       & 0.8                  & 5.39               \\ \clineB{2-10}{.5}
                      & Dog         & 73.0        & 82.92       & 0.8104       & 77.2        & 83.60         & 0.8248                        & 0.9                  & 5.75               \\ \hlineB{2}
\multirow{2}{*}{20}  & Cat          & 64.6       & 78.98        & 0.7545       & 64.6        & 78.98          & 0.7545                   & 1                    & 0.00          \\ \clineB{2-10}{.5}
                      & Dog         & 67.4       & 81.02        & 0.7803       & 70.2        & 81.75          & 0.7937                   & 0.99                 & 4.15               \\ \hlineB{2}
\multirow{2}{*}{10}  & Cat          & 38.0        & 67.69       & 0.5405       & 41.8        & 69.34          & 0.5769                   & 0.85                 & 10.00            \\ \clineB{2-10}{.5}
                      & Dog         & 46.4       & 72.14        & 0.6248        & 50.1        & 73.53          & 0.6543                   & 0.9                  & 7.97            \\ \hlineB{2}
\multirow{2}{*}{5} & Cat            & \hphantom{0}1.7  & 50.80  & 0.0334       & 13.6        & 56.26          & 0.2372                    & 0.8                  & 700.00          \\ \clineB{2-10}{.5}
                      & Dog         & 23.7       & 61.44        & 0.3807        & 31.0         & 64.90         & 0.4690                   & 0.8                  & 30.80          \\ \hlineB{2}
\end{tabular}
\caption{Binary \textbf{CIFAR-10} CNN 6$+$2 Imbalance Experiments on \textit{Cat} and \textit{Dog} classes. }
\label{table:CIFARclassimbalance6+2}
\end{center}
\end{table}



\subsection{Key Takeaways}

We conclude this section by highlighting the key takeaways from our experimental results.

\textbf{Overall Performance Relative to Log-loss:} The experimental results as evidenced through \Cref{table:MNISTsymmetricnoisylabels,table:FMNISTsymmetricnoisylabels,table:CIFARsymmetricnoisylabels,table:FMNISTclassimbalanceLR,table:CIFARclassimbalance4+2,table:CIFARclassimbalance6+2} suggest that $\alpha$-loss, more often than not, yields models with improvements in test accuracy over models trained with log-loss, with more prominent gains in the canonical settings of noisy labels and class imbalances in the training data. 
%
In order to remedy the extra hyperparameter tuning induced by the seemingly daunting task of searching over $\alpha \in (0,\infty]$, we find that searching over $\alpha \in [.8,8]$ in the noisy label experiments or $\alpha \in [.8,4]$ in the class imbalance experiments is sufficient. 
This aligns with our earlier theoretical investigations (Section~\ref{sec:landscapelogisticmodel}) regarding the so-called "Goldilocks zone", i.e., most of the meaningful action induced by $\alpha$ occurs in a narrow region. 
Notably in the class imbalance experiments, we find that the relevant region is even narrower than our initial choice, i.e., $\alpha^{*} \in [.8,2.5]$ (in our search set) for all imbalances. 
For the noisy label experiments, we always find that $\alpha^{*} > 1$ and usually $\alpha$ is not too large, and for the class imbalance experiments, we almost always find that $\alpha^{*} < 1$.
These two heuristics enable the practitioner to readily determine a very good $\alpha$ in these two canonical scenarios.
Consequently, $\alpha$-loss seems to be a principled generalization of log-loss for the practitioner, and it perhaps remedies the concern of Janocha \textit{et al.} in~\cite{janocha2017loss} regarding the lack of canonical alternatives to log-loss (cross-entropy loss) in modern machine learning. 

\section{Conclusions}


In this work, we introduced a tunable loss function called $\alpha$-loss, $\alpha \in (0,\infty]$, which interpolates between the exponential loss ($\alpha = 1/2$), the log-loss ($\alpha = 1$), and the 0-1 loss ($\alpha = \infty$),
for the machine learning setting of classification.
We illustrated the connection between $\alpha$-loss and Arimoto conditional entropy (Section~\ref{sec:Prelim}), and then
we studied the statistical calibration (Section~\ref{sec:alphalossbinaryclassification}), optimization landscape (Section~\ref{sec:landscapelogisticmodel}), and generalization capabilities (Sec~\ref{sec:generalizationoptimality}) of $\alpha$-loss induced by navigating the $\alpha$ hyperparameter.
Regarding our main theoretical results, we showed that $\alpha$-loss is classification-calibrated for all $\alpha \in (0,\infty]$; we also showed that in the logistic model there is a "Goldilocks zone", such that most of the meaningful action induced by $\alpha$ occurs in a narrow region (usually $\alpha \in [.8,8]$); finally, we showed (under standard distributional assumptions) that empirical minimizers of $\alpha$-loss for all $\alpha \in (0,\infty]$ are asymptotically optimal with respect to the true 0-1 loss.
Practically, following intuitions developed in Section~\ref{sec:syntheticexp}, we performed noisy label and class imbalance experiments on MNIST, FMNIST, and CIFAR-10 using logistic regression and convolutional-neural-networks (Section~\ref{sec:Experiments}). 
We showed that models trained with $\alpha$-loss can be more robust or sensitive to outliers (depending on the application) over models trained with log-loss ($\alpha = 1$).
Thus, we argue that $\alpha$-loss seems to be a principled generalization of log-loss for classification algorithms in modern machine learning.
%
%
%
Regarding promising avenues to further explore the role of $\alpha$-loss in machine learning, 
the robustness of neural-networks to adversarial influence has recently drawn much attention \cite{zhang2019theoretically,madry2017towards,schmidt2018adversarially} in addition to learning censored and fair representations that ensure statistical fairness for all downstream learning tasks~\cite{kairouz2019censored}.

\bibliographystyle{IEEEtran}
\bibliography{TS_ML}

\appendix
\subsection{$\alpha$-loss in Binary Classification} \label{appen:binaryclass}
\begin{repprop}{Prop:relationhardtosoft}
Consider a soft classifier $g(x)=P_{\hat{Y}|X}(1 | x)$.
If $f(x) = \sigma^{-1}(g(x))$, then, for every $\alpha\in (0,\infty]$,
\begin{equation}
    l^{\alpha}(y,g(x)) = \tilde{l}^{\alpha}(yf(x)).
\end{equation}
Conversely, if $f$ is a classification function, then the set of beliefs $P_{\hat{Y}|X}$ associated to $g(x) := \sigma(f(x))$ satisfies \eqref{eq:lglf}. In particular, for every $\alpha\in(0,\infty]$,
\begin{equation}
\min_{g} \mathbb{E}_{X,Y}(l^\alpha(Y,g(x))) = \min_f \mathbb{E}_{X,Y}(\tilde{l}^\alpha(Yf(X))).
\end{equation}
\end{repprop}
%
\begin{proof}
Consider a soft classifier $g$ and let $P_{\hat{Y}|X}$ be the set of beliefs associated to it. Suppose $f(x) = \sigma^{-1}(g(x))$, where $g(x) = P_{\hat{Y}|X}(1|x)$. We want to show that 
\begin{equation} \label{desiredeq} l^{\alpha}(y,P_{\hat{Y}|X=x}) = \tilde{l}^{\alpha}(yf(x)).\end{equation} 
We assume that $\alpha \in (0,1) \cup (1,\infty)$. Note that the cases where $\alpha = 1$ and $\alpha = \infty$ follow similarly. 
\par
Suppose that $g(x) = P_{\hat{Y}|X}(1|x) = \sigma(f(x))$. If $y = 1$, then \begin{align} l^{\alpha}(1,P_{\hat{Y}|X}(1|x)) &= l^{\alpha}(1,\sigma(f(x))) \\ &= \frac{\alpha}{\alpha - 1}\left[1 - \sigma(f(x))^{1 - 1/\alpha}\right] \\ 
&= \tilde{l}^{\alpha}(f(x)).\end{align} 
If $y = -1$, then \begin{align} l^{\alpha}(-1,P_{\hat{Y}|X}(-1|x)) &= l^{\alpha}(-1,1 - P_{\hat{Y}|X}(1|x))
\\ &= l^{\alpha}(-1,1 - \sigma(f(x))) \\ \label{step} &= l^{\alpha}(-1,\sigma(-f(x))) \\ &= \frac{\alpha}{\alpha - 1}[1 - \sigma(-f(x))^{1 - 1/\alpha}]\\ &= \tilde{l}^{\alpha}(-f(x)),\end{align} where \eqref{step} follows from
\begin{equation}\label{simprop} \sigma(x) + \sigma(-x) = 1, \end{equation}
which can be observed by \eqref{eq:DefSigmoid}.
To show the reverse direction of \eqref{desiredeq} we substitute \begin{equation} f(x) = \sigma^{-1}(g(x)) = \sigma^{-1}(P_{\hat{Y}|X}(1|x))\end{equation}
in $\tilde{l}^{\alpha}(yf(x))$. 
For $y = 1$, \begin{align} \tilde{l}^{\alpha}(f(x)) &= \tilde{l}^{\alpha}(\sigma^{-1}(P_{\hat{Y}|X}(1|x))) 
\\ &= \frac{\alpha}{\alpha - 1}[1 - (\sigma(\sigma^{-1}(P_{\hat{Y}|X}(1|x))))^{1 - 1/\alpha}]
\\&=\frac{\alpha}{\alpha - 1}[1 - P_{\hat{Y}|X}(1|x)^{1 - 1/\alpha}]
\\&= l^{\alpha}(1,P_{\hat{Y}|X}(1|x)).\end{align} 
For $y = -1$, \begin{align} \tilde{l}^{\alpha}(-f(x)) &= \tilde{l}^{\alpha}(-\sigma^{-1}(P_{\hat{Y}|X}(1|x))) 
\\ &= \frac{\alpha}{\alpha - 1}[1 - \sigma(-\sigma^{-1}(P_{\hat{Y}|X}(1|x)))^{1 - 1/\alpha}]
\\ \label{step2} &= \frac{\alpha}{\alpha - 1}[1 - (1-\sigma(\sigma^{-1}(P_{\hat{Y}|X}(1|x))))^{1 - 1/\alpha}]
\\ &= \frac{\alpha}{\alpha - 1}[1 - P_{\hat{Y}|X}(-1|x)^{1 - 1/\alpha}]
\\&= l^{\alpha}(-1, P_{\hat{Y}|X}(-1|x)),\end{align} where \eqref{step2} follows from \eqref{simprop}.
\par
The equality in the results of the minimization procedures follows from the equality between $l^{\alpha}$ and $\tilde{l}^{\alpha}$. As was shown in \cite{liao2018tunable}, the minimizer of the left-hand-side is 
\begin{equation} 
P^{*}_{\hat{Y}|X}(y|x) = \dfrac{P_{Y|X}(y|x)^{\alpha}}{\sum\limits_{y} P_{Y|X} (y|x)^{\alpha}}.
\end{equation} 
Using $f(x) = \sigma^{-1}(P_{\hat{Y}|X}(1|x))$, $f^{*}(x) = \sigma^{-1}(P^{*}_{\hat{Y}|X}(1|x))$.
\end{proof}

\begin{repprop}{Prop:alpha-loss-convex}
As a function of the margin, $\tilde{l}^{\alpha}:\overline{\mathbb{R}}\to\mathbb{R}_+$ is convex for $\alpha \leq 1$ and quasi-convex for $\alpha > 1$.
\end{repprop}
\begin{proof}
The second derivative of the margin-based $\alpha$-loss for $\alpha \in (0,\infty]$ with respect to the margin is given by
\begin{equation} \label{2ndderiv}
\dfrac{d^2}{dz^{2}} \tilde{l}^{\alpha}(z) = \dfrac{(e^{-z}+1)^{1/\alpha}e^{z}(\alpha e^z - \alpha + 1)}{\alpha(e^{z} + 1)^{3}}.
\end{equation}
Observe that if $\alpha \in (0,1]$, then we have that, for all $z \in \overline{\mathbb{R}}$, $\dfrac{d^2}{dz^{2}} \tilde{l}^{\alpha}(z) \geq 0$, which implies that $\tilde{l}^{\alpha}$ is convex \cite{boyd2004convex}.
If $\alpha \in (1,\infty]$, then note that $\alpha e^z - \alpha + 1 < 0$ for all $z \in \overline{\mathbb{R}}$ such that $z < \log{\left(1-\alpha^{-1}\right)}$.
Thus, the margin-based $\alpha$-loss, $\tilde{l}^{\alpha}$, is not convex for $\alpha \in (1,\infty]$.
%
%
%
However, observe that 
\begin{equation}
\dfrac{d}{dz}\tilde{l}^{\alpha}(z) = \dfrac{-(e^{-z} + 1)^{1/\alpha}e^{z}}{(1+e^{z})^{2}}.
\end{equation}
Since $\dfrac{d}{dz}\tilde{l}^{\alpha}(z) < 0$ for $\alpha \in [1,\infty]$ and for all $z \in \overline{\mathbb{R}}$,
$\tilde{l}^{\alpha}$ is monotonically decreasing. Furthermore, since monotonic functions are quasi-convex \cite{boyd2004convex}, we have that $\tilde{l}^{\alpha}$ is quasi-convex for $\alpha > 1$.
%
\end{proof}

\begin{reptheorem}{thm:alphalossclassificationcalibration}
For $\alpha\in (0,\infty]$, the margin-based $\alpha$-loss $\tilde{l}^{\alpha}$ is classification-calibrated. In addition, its optimal classification function is given by
\begin{equation}
    f^{*}_{\alpha}(\eta) = \alpha \cdot \sigma^{-1}(\eta).
\end{equation}
\end{reptheorem}
\begin{proof}
We first show that $\tilde{l}^{\alpha}$ is classification-calibrated for all $\alpha \in (0,\infty]$.
Suppose that $\alpha \in (0,1]$; we rely on the following result by Bartlett \textit{et al.} in \cite{bartlett2006convexity}.
\begin{prop}[{\cite[Theorem 6]{bartlett2006convexity}}] \label{Prop:bartlett}
Suppose $\phi: \mathbb{R} \rightarrow \mathbb{R}$ is a convex function in the margin. Then $\phi$ is classification-calibrated if and only if it is differentiable at $0$ and $\phi'(0) < 0$. 
\end{prop}
Observe that $\tilde{l}^{\alpha}$ is smooth and monotonically decreasing for all $\alpha \in (0,\infty]$, and for $\alpha \in (0,1]$, $\tilde{l}^{\alpha}$ is convex by Proposition \ref{Prop:alpha-loss-convex}. Thus, $\tilde{l}^{\alpha}$ satisfies Proposition \ref{Prop:bartlett}, which implies that $\tilde{l}^{\alpha}$ is classification-calibrated for $\alpha \in (0,1)$. 

Now consider $\alpha \in (1, \infty)$. 
Since classification-calibration requires proving that the minimizer of \eqref{eq:mincondrisk} agrees in sign with the Bayes predictor, we 
first obtain the minimizer of the conditional risk for all $\eta \neq 1/2$. We have that
\begin{align} \label{plug2}
\inf\limits_{f \in \mathbb{R}} C_{\tilde{l}^{\alpha}}(\eta,f) &= \inf\limits_{f \in \mathbb{R}} \eta \tilde{l}^{\alpha}(f) + (1-\eta)\tilde{l}^{\alpha}(-f) \\
&= \dfrac{\alpha}{\alpha - 1} \left(1 - \sup\limits_{f\in\mathbb{R}} \left[\eta\left(\dfrac{1}{1 + e^{-f}}\right)^{1 - 1/\alpha} + (1 - \eta)\left(\dfrac{1}{1 + e^{f}}\right)^{1 - 1/\alpha}\right]\right),
\end{align}
where we substituted $\tilde{l}^{\alpha}$ into \eqref{plug2} and pulled the infimum through.
We take the derivative of the expression inside the supremum, which we denote $g(\eta,\alpha,f)$, and obtain
\begin{equation} \label{derg}
\begin{split}
\dfrac{d}{df} g(\eta,\alpha,f) = \left(1 - \dfrac{1}{\alpha}\right)\left(\dfrac{1}{e^{f} + 2 + e^{-f}}\right)\left[\eta\left({1+e^{-f}}\right)^{1/\alpha} - (1-\eta)\left({1+e^{f}}\right)^{1/\alpha}\right].
\end{split}
\end{equation}
One can then obtain the $f^{*}$ minimizing \eqref{plug2} by setting $\dfrac{d}{df} g(\eta,\alpha,f) = 0$, i.e., \begin{equation} \label{deriv}
\eta\left(1+e^{-f^{*}}\right)^{1/\alpha} = (1-\eta)\left(1+e^{f^{*}}\right)^{1/\alpha},
\end{equation} 
and solving for $f^{*}$ we have
\begin{equation} 
\label{attainer} 
f^{*}_{\alpha}(\eta) = \alpha \log{\Big(\dfrac{\eta}{1-\eta}\Big)} = \alpha \cdot \sigma^{-1}(\eta). 
\end{equation} 
Recall that the Bayes predictor is given by $h_{\text{Bayes}}(\eta) = \text{sign}{(2\eta - 1)}$, and notice that the classification function representation is simply
$f_{\text{Bayes}}(\eta) = 2\eta - 1$.
Observe that for all $\eta \neq 1/2$ and for $\alpha \in [1,\infty)$ (indeed $\alpha < 1$ as well), we have that $\text{sign}(f_{\text{Bayes}}(\eta)) = \text{sign}(f^{*}_{\alpha}(\eta))$. Thus, $\tilde{l}^{\alpha}$ is classification-calibrated for $\alpha \in (0,\infty)$.
Lastly, if $\alpha = +\infty$, then $\tilde{l}^{\alpha}$ becomes 
\begin{equation}
\tilde{l}^{\infty}(z) = 1 - \sigma(z) = \dfrac{e^{z}}{1+e^{z}},
\end{equation}
which is sigmoid loss. Similarly, sigmoid loss can be shown to be classification-calibrated as is given in \cite{bartlett2006convexity}.
Therefore, $\tilde{l}^{\alpha}$ is classification-calibrated for all $\alpha \in (0,\infty]$.

Finally, note that the proof of classification-calibration yielded the optimal classification function given in \eqref{attainer} for all $\alpha \in (0,\infty]$. 
Alternatively, the optimal classification function can be obtained from Proposition \ref{Prop:Liao} by Liao \textit{et al.} Specifically, substitute the $\alpha$-tilted distribution \eqref{eq:tilteddistribution} for a binary label $\mathcal{Y} = \{-1,+1\}$ into \eqref{eq:DefInverseSigmoid} as stated by Proposition \ref{Prop:relationhardtosoft}. Indeed, we have that
\begin{align}
f^{*}(x) = \sigma^{-1}(P^{*}_{\hat{Y}|X}(1|x)) = \log{\left(\frac{P_{Y|X}(1|x)^{\alpha}}{P_{Y|X}(-1|x)^{\alpha}} \right)} = \alpha \log{\left(\frac{\eta(x)}{1 - \eta(x)} \right)},
\end{align}
which aligns with \eqref{eq:optimalclassifier}.
\end{proof}

\begin{repcorollary}{cor:condrisk}
For $\alpha \in (0,\infty]$, the minimum conditional risk of $\tilde{l}^{\alpha}$ is given by
\begin{equation} 
C_{\alpha}^{*}(\eta) =
\begin{cases} 
    \frac{\alpha}{\alpha - 1} \left(1 - (\eta^{\alpha} + (1-\eta)^{\alpha})^{1/\alpha} \right) & \alpha \in (0,1) \cup (1,+\infty), \\
    -\eta \log{\eta} - (1-\eta) \log{(1-\eta)} & \alpha = 1, \\
    \min\{\eta, 1- \eta \} & \alpha \rightarrow +\infty.
\end{cases}
\end{equation}
\end{repcorollary}
\begin{proof}
For $\alpha = 1$, we recover logistic loss and we know from \cite{masnadi2009design} and \cite{sypherd2019tunable} that the minimum conditional risk is given by 
\begin{equation}
C_{1}^{*}(\eta) = -\eta \log{\eta} - (1-\eta) \log{(1-\eta)}.
\end{equation}
Similarly, for $\alpha = \infty$, we recover the sigmoid loss and we know from \cite{bartlett2006convexity} and \cite{sypherd2019tunable} that the minimum conditional risk is given by 
\begin{align}
C_{\infty}^{*}(\eta) = \min\{\eta, 1- \eta \}.
\end{align}
Thus, we now consider the case where $\alpha \in (0,\infty) \setminus \{1\}$. The conditional risk of $\tilde{l}^{\alpha}$ is given by
\begin{align} \label{eq:condrisk}
C_{\alpha}(\eta,f) &= \eta \tilde{l}^{\alpha}(f) + (1-\eta)\tilde{l}^{\alpha}(-f) \\
\label{eq:condriskstep1} &= \frac{\alpha}{\alpha-1} \left[1 - \eta \left(\frac{1}{1 + e^{-f}} \right)^{1 - 1/\alpha} - (1-\eta) \left(\frac{1}{1+e^{f}} \right)^{1 - 1/\alpha}  \right],
\end{align}
where we substituted \eqref{eq:marginalphaloss} into \eqref{eq:condrisk}.
We can obtain the minimum conditional risk upon substituting \eqref{eq:optimalclassifier} into \eqref{eq:condriskstep1} which yields
\begin{align}
C_{\alpha}^{*}(\eta) &= \frac{\alpha}{\alpha - 1} \left[1 - \eta \left(\frac{\eta^{\alpha}}{\eta^{\alpha} + (1-\eta)^{\alpha}} \right)^{1-1/\alpha} - (1-\eta) \left(\frac{(1-\eta)^{\alpha}}{\eta^{\alpha} + (1-\eta)^{\alpha}} \right)^{1-1/\alpha} \right] \\
&= \frac{\alpha}{\alpha - 1} \left[1 - (\eta^{\alpha} + (1-\eta)^{\alpha})^{1/\alpha} \right],
\end{align}
where the last equation is obtained after some algebra.
Finally, observe that $C_{1/2}^{*}(\eta) = 2 \sqrt{\eta(1-\eta)}$, which aligns with \cite{masnadi2009design}.
\end{proof}

\subsection{Optimization Guarantees for $\alpha$-loss in the Logistic Model} \label{appen:logisticmodel}
\begin{reptheorem}{Thm:SLQClessthan1}
Let $\Sigma := \mathbb{E}[XX^{\intercal}]$. If $\alpha \in (0,1]$, then $R_{\alpha}(\theta)$ 
is $\Lambda(\alpha,r\sqrt{d}) \min_{i \in [d]}\lambda_{i}\left(\Sigma\right)$-strongly convex in $\theta \in \mathbb{B}_{d}(r)$,
where 
\begin{align}
\Lambda(\alpha,r\sqrt{d}) := \sigma(r\sqrt{d})^{1-1/\alpha}\left(\sigma'(r\sqrt{d}) - \left(1 - \alpha^{-1} \right)\sigma(-r\sqrt{d})^{2}\right).
\end{align}
\end{reptheorem}
\begin{proof}
%
\noindent For each $\alpha \in (0,1]$, it can readily be shown that each component of $F_{2}(\alpha,\theta,x,y)$ is positive and monotonic in $\langle \theta,x \rangle$, which implies that $F_{2}(\alpha,\theta,x,y) \geq \Lambda(\alpha,r\sqrt{d}) > 0$.
%
%
Now, consider $R_{\alpha}(\theta) = \mathbb{E}[l^{\alpha}(Y,g_{\theta}(X))]$. We have
\begin{align}
\nonumber \nabla_{\theta}^{2} R_{\alpha}(\theta) &= \mathbb{E}_{X,Y}[\nabla_{\theta}^{2} l^{\alpha}(Y,g_{\theta}(X))] \\
\label{eq:thm1_0} &= \mathbb{E}_{X,Y}[F_{2}(\alpha,\theta,X,Y) XX^{\intercal}] \\ 
\label{eq:thm1_1} &\succeq \Lambda(\alpha,r\sqrt{d}) \mathbb{E}[XX^{\intercal}] \\
\label{eq:thm1_2} &= \Lambda(\alpha,r\sqrt{d}) \Sigma \succeq 0,
\end{align}
where we used an identity of positive semi-definite matrices for \eqref{eq:thm1_1} (see, e.g., \cite[Ch.~7]{horn2012matrix}); for \eqref{eq:thm1_2}, we used the fact that $\Lambda(\alpha,r\sqrt{d}) \geq 0$ and we recognize that $\Sigma$ is positive semi-definite as it is the 
correlation of the random vector $X \in [0,1]^{d}$ (see, e.g., \cite[Ch.~7]{papoulis2002probability}). We also note that $\min_{i \in [d]}\lambda_{i}\left(\Sigma\right) \geq 0$ (see, e.g., \cite[Ch.~7]{horn2012matrix}).
Thus, $\nabla_{\theta}^{2} R_{\alpha}(\theta)$ is positive semi-definite for every $\theta \in \mathbb{B}_{d}(r)$. 
Therefore, since $\lambda_{\min}(\nabla^{2} R_{\alpha}(\theta)) \geq \Lambda(\alpha,r\sqrt{d}) \min_{i \in [d]}\lambda_{i}\left(\Sigma\right) \geq 0$ for every $\theta \in \mathbb{B}_{d}(r)$, which follows by the Courant-Fischer min-max theorem~\cite[Theorem 4.2.6]{horn2012matrix}, we have that $R_{\alpha}$ is $\Lambda(\alpha,r\sqrt{d}) \min_{i \in [d]}\lambda_{i}\left(\Sigma\right)$-strongly convex for $\alpha \in (0,1]$.
\end{proof}

\begin{repcorollary}{cor:convexforalphagreaterthan1}
Let $\Sigma := \mathbb{E}[XX^{\intercal}]$.
If $r\sqrt{d} < \arcsinh{(1/2)}$, then $R_{\alpha}(\theta)$ is $\tilde{\Lambda}(\alpha,r\sqrt{d}) \min_{i \in [d]}\lambda_{i}\left(\Sigma\right)$-strongly convex in $\theta \in \mathbb{B}_{d}(r)$ for $\alpha \in \left(0,(e^{2r\sqrt{d}}-e^{r\sqrt{d}})^{-1}\right]$, where 
\begin{align}
\tilde{\Lambda}(\alpha,r\sqrt{d}) := \sigma(-r \sqrt{d})^{2-1/\alpha}\sigma(r \sqrt{d})\left(1 - e^{r\sqrt{d}} + \alpha^{-1} e^{-r\sqrt{d}} \right).
\end{align}
\end{repcorollary} 
\begin{proof}
Let $\theta \in \mathbb{B}_{d}(r)$ be arbitrary. We similarly have that 
\begin{align}
\nabla_{\theta}^{2} R_{\alpha}(\theta) &= \mathbb{E}_{X,Y}[\nabla_{\theta}^{2} l^{\alpha}(Y,g_{\theta}(X))] \\
&= \mathbb{E}_{X,Y}[g_{\theta}(YX)^{1-1/\alpha}(g_{\theta}'(YX) - (1-1/\alpha) g_{\theta}(-YX)^{2})XX^{\intercal}] \\
\label{eq:cor2_1} &= \mathbb{E}_{X,Y}[g_{\theta}(YX)^{1-1/\alpha}g_{\theta}(-YX)\left(g_{\theta}(YX) - (1-1/\alpha) g_{\theta}(-YX)\right)XX^{\intercal}], 
\end{align}
where we recall~\eqref{eq:thm1_0} and factored out $g_{\theta}(-YX)$. Considering the expression in parentheses in \eqref{eq:cor2_1}, we note that this is the only part of the Hessian which can become negative.
%
Examining this term more closely, we find that 
\begin{align}
g_{\theta}(YX) - \left(1-\frac{1}{\alpha}\right) g_{\theta}(-YX) &= \frac{1}{1+e^{-\langle \theta, YX \rangle}} - \left(1-\frac{1}{\alpha}\right) \frac{1}{1+e^{\langle \theta, YX \rangle}} \\
&= g_{\theta}(YX) \left[1 - \left(1-\frac{1}{\alpha}\right) \frac{1+e^{-\langle \theta, YX \rangle}}{1+e^{\langle \theta, YX \rangle}} \right] \\
&= \label{eq:cor2important} g_{\theta}(YX) \left[1 - \left(1-\frac{1}{\alpha}\right) e^{-\langle \theta, Y X \rangle} \right].
\end{align}
Continuing, observe that 
\begin{align}
1 - \left(1-\frac{1}{\alpha}\right) e^{-\langle \theta, Y X \rangle} &= 1 - e^{-\langle \theta, Y X \rangle} + \frac{e^{-\langle \theta, Y X \rangle}}{\alpha} \\
\label{eq:cor2_3} &\geq 1 - e^{r\sqrt{d}} + \frac{e^{-r\sqrt{d}}}{\alpha} \geq 0, 
\end{align}
where we lowerbound using the radius of the balls $\langle \theta, YX \rangle \leq |Y| \|\theta\|\|X\| \leq r\sqrt{d}$ (Cauchy-Schwarz) and the last inequality in \eqref{eq:cor2_3} holds if $\alpha \leq e^{-r\sqrt{d}}(e^{r\sqrt{d}} - 1)^{-1}$.
Thus, returning to \eqref{eq:cor2_1}, we have that 
\begin{align}
\nabla_{\theta}^{2} R_{\alpha}(\theta) &= \mathbb{E}_{X,Y}\left[g_{\theta}(YX)^{1-1/\alpha}g_{\theta}(-YX)\left(g_{\theta}(YX) - (1-1/\alpha) g_{\theta}(-YX)\right)XX^{\intercal}\right] \\
\label{eq:cor2_4.5} &= \mathbb{E}_{X,Y}\left[g_{\theta}(YX)^{1-1/\alpha}g_{\theta}'(YX)\left(1 - \left(1-\frac{1}{\alpha}\right) e^{-\langle \theta, Y X \rangle}\right)XX^{\intercal}\right] \\
\label{eq:cor2_4} &\succeq \sigma(-r \sqrt{d})^{2-1/\alpha}\sigma(r \sqrt{d})\left(1 - e^{r \sqrt{d}} + \frac{e^{-r \sqrt{d}}}{\alpha}\right) \mathbb{E}\left[XX^{\intercal}\right] \\
\label{eq:cor2_5} &=\sigma(-r \sqrt{d})^{2-1/\alpha}\sigma(r \sqrt{d})\left(1 - e^{r\sqrt{d}} + \frac{e^{-r\sqrt{d}}}{\alpha}\right) \Sigma \succeq 0,
\end{align}
where in~\eqref{eq:cor2_4.5} we used~\eqref{eq:cor2important} and the fact that $\sigma'(z) = \sigma(z)\sigma(-z)$ (as given in~\eqref{eq:sigprop1}),
and in
\eqref{eq:cor2_4} and \eqref{eq:cor2_5} we use the upper-bound derived above and the same arguments as Theorem \ref{Thm:SLQClessthan1}, \textit{mutatis mudandis}.
%
Thus, if $\alpha \leq e^{-r\sqrt{d}} (e^{r\sqrt{d}} - 1)^{-1}$, then $R_{\alpha}(\theta)$ is $\tilde{\Lambda}(\alpha,r\sqrt{d}) \min_{i \in [d]}\lambda_{i}\left(\Sigma\right)$-strongly convex in $\theta \in \mathbb{B}_{d}(r)$, 
where $\tilde{\Lambda}(\alpha,r\sqrt{d}) := \sigma(-r \sqrt{d})^{2-1/\alpha}\sigma(r \sqrt{d})\left(1 - e^{r\sqrt{d}} + \alpha^{-1} e^{-r\sqrt{d}} \right)$.

Finally, recall that $\sinh(x) = (e^{x} - e^{-x})/2$ and $\arcsinh{x} = \log{(x+\sqrt{x^{2}+1})}$. Observe that $r\sqrt{d} \leq \arcsinh{(1/2)}$ implies that $e^{-r\sqrt{d}} (e^{r\sqrt{d}} - 1)^{-1} \geq 1$. Also note that $e^{-r\sqrt{d}} (e^{r\sqrt{d}} - 1)^{-1}$ is monotonically decreasing in $r\sqrt{d}$ and that $\arcsinh{(1/2)} \approx 0.48$.
\end{proof}
\begin{repprop}{prop:SLQCpriortoevolution}
Suppose that $\Sigma \succ 0$ and $\theta_0 \in \mathbb{B}_{d}(r)$ is fixed. We have one of the following:
\begin{itemize}
    \item If $r\sqrt{d} < \arcsinh{(1/2)}$, then, for every $\epsilon>0$, $R_{\alpha}$ is $(\epsilon,C_{d}(r,\alpha),\theta_0)$-SLQC at $\theta$ for every $\theta\in\mathbb{B}_{d}(r)$ for $\alpha \in \left(0,(e^{2r\sqrt{d}}-e^{r\sqrt{d}})^{-1}\right]$ where $C_{d}(r,\alpha)$ is given in \eqref{eq:alpharisklipintheta};
    \item Otherwise, for every $\epsilon>0$, $R_{\alpha}$ is $(\epsilon,C_{d}(r,\alpha),\theta_0)$-SLQC at $\theta$ for every $\theta\in\mathbb{B}_{d}(r)$ for $\alpha \in (0, 1]$. 
\end{itemize}
\end{repprop}
\begin{proof}
In order to prove the result, we apply a result by Hazan, \textit{et al.}~\cite{hazan2015beyond} where they show that if a function $f$ is $G$-Lipschitz and strictly-quasi-convex, then for all $\epsilon > 0$, $f$ is $(\epsilon,G,\theta_{0})$-SLQC in $\theta$. Thus, one may view $\kappa$ as approximately quantifying the growth of the gradients of general functions.

First, we show that $R_{\alpha}$ is $C_{d}(r,\alpha)$-Lipschitz in $\theta \in \mathbb{B}_{d}(r)$ where 
\begin{align} 
C_{d}(r,\alpha) := \begin{cases}
			\sqrt{d} \sigma(r\sqrt{d})\sigma(-r\sqrt{d})^{1-1/\alpha}, & \alpha \in (0,1] \\
            \sqrt{d} \left(\frac{\alpha-1}{2 \alpha - 1} \right)^{1-1/\alpha} \left(\frac{\alpha}{2\alpha -1} \right), & \alpha \in (1,\infty] \quad \text{and} \quad r\sqrt{d} \geq \log{\left(1 - 1/\alpha \right)} \\
            \sqrt{d} \sigma(r\sqrt{d})\sigma(-r\sqrt{d})^{1-1/\alpha}, & \alpha \in (1,\infty] \quad \text{and} \quad r\sqrt{d} < \log{\left(1 - 1/\alpha \right)}. \\
		 \end{cases}
\end{align}
Formally, we want to show that for all $\theta, \theta' \in \mathbb{B}_{d}(r)$,
\begin{align}
|R_{\alpha}(\theta) - R_{\alpha}(\theta')| \leq C \|\theta - \theta' \|,
\end{align}
where $C:= \sup_{\theta \in \mathbb{B}_{d}(r)} \|\nabla R_{\alpha}(\theta) \|$. 
Recall from~\eqref{eq:alphagradLR} that 
\begin{align}
\nabla_{\theta} R_{\alpha}(\theta) = \mathbb{E}[\nabla_{\theta} l^{\alpha}(Y,g_{\theta}(X)] = \mathbb{E}[F_{1}(\alpha,\theta,X,Y)X],
\end{align}
where from~\eqref{eq:alphaderLR} we have
\begin{align}
F_{1}(\alpha,\theta,x,y) = - y g_\theta(yx)^{1-1/\alpha}(1-g_\theta(yx)).
\end{align}
It can be shown that for $\alpha \leq 1$, $|F_{1}(\alpha,\theta,x,y)| = g_\theta(yx)^{1-1/\alpha}(1-g_\theta(yx))$ is monotonically decreasing in $\langle \theta, x \rangle$. Thus for $\alpha \leq 1$,
\begin{align}
C = \sqrt{d} \sigma(r\sqrt{d})\sigma(-r\sqrt{d})^{1-1/\alpha}.
\end{align}
It can also be shown that for $\alpha > 1$, $|F_{1}(\alpha,\theta,x,y)|$ is unimodal and quasi-concave with maximum obtained at~$\langle \theta, x \rangle^{*} = \log{\left(1 - 1/\alpha \right)}$.
If $r\sqrt{d} \geq \log{\left(1 - 1/\alpha \right)}$, we obtain upon plugging in $\langle \theta, x \rangle^{*}$ for $\alpha > 1$, 
\begin{align}
C = \sqrt{d} \left(\frac{\alpha - 1}{2 \alpha - 1} \right)^{1 - 1/\alpha} \left(\frac{\alpha}{2 \alpha - 1} \right).
\end{align}
Otherwise, if $r\sqrt{d} < \log{\left(1 - 1/\alpha \right)}$, then, using the local monotonicity of $|F_{1}(\alpha,\theta,x,y)|$, we obtain for $\alpha > 1$, 
\begin{align}
C = \sqrt{d} \sigma(r\sqrt{d})\sigma(-r\sqrt{d})^{1-1/\alpha},
\end{align}
which mirrors the $\alpha < 1$ case.
Therefore, combining the two regimes of $\alpha$ we have that $R_{\alpha}$ is $C_{d}(r,\alpha)$-Lipschitz in $\theta \in \mathbb{B}_{d}(r)$ for $\alpha \in (0,\infty]$ where $C_{d}(r,\alpha)$ is given in~\eqref{eq:alpharisklipintheta}.

Finally when $R_{\alpha}$ is strongly-convex, this implies that $R_{\alpha}$ is strictly-quasi-convex.
That is, since $\Sigma \succ 0$, we merely apply Corollary~\ref{cor:convexforalphagreaterthan1} to obtain strong-convexity of $R_{\alpha}$ for $\alpha \in (0,(e^{2r\sqrt{d}}-e^{r\sqrt{d}})^{-1}]$ when $r\sqrt{d} < \arcsinh{(1/2)}$.
Similarly, we apply Theorem~\ref{Thm:SLQClessthan1} to obtain strong-convexity of $R_{\alpha}$ for $\alpha \in (0,1]$, otherwise.
\end{proof}

\subsubsection{Fundamentals of SLQC and Reformulation} \label{appen:prelim&reformSLQC}

In this subsection, we briefly review \textit{strictly locally quasi-convexity} (SLQC) which was introduced by Hazan \textit{et al.} in \cite{hazan2015beyond}. 
Recall that in \cite{hazan2015beyond} Hazan \textit{et al.} refer to a function as SLQC \textit{in $\theta$}, whereas for the purposes of our analysis we refer to a function as SLQC \textit{at} $\theta$. We recover the uniform SLQC notion of Hazan \textit{et al.} by articulating a function is SLQC \textit{at $\theta$ for every $\theta$}. Our later analysis of the $\alpha$-risk in the logistic model benefits from this pointwise consideration.
Intuitively, the notion of SLQC functions extends quasi-convex functions in a parameterized manner.  
Regarding notation, for $\theta_{0}\in\mathbb{R}^d$ and $r>0$, we let $\mathbb{B}(\theta_0,r) := \{\theta\in\mathbb{R}^d : \|\theta-\theta_0\| \leq r\}$.
\begin{repdefinition}{def:SLQC}[{\!\cite[Definition 3.1]{hazan2015beyond}}]
Let $\epsilon,\kappa>0$ and $\theta_{0} \in \mathbb{R}^{d}$. A function $f: \mathbb{R}^{d} \rightarrow \mathbb{R}$ is called $(\epsilon, \kappa, \theta_{0})$-strictly locally quasi-convex (SLQC) at $\theta\in\mathbb{R}^d$ if at least one of the following applies:
\begin{enumerate}
\item $f(\theta) - f(\theta_{0}) \leq \epsilon$,
\item $\|\nabla f(\theta)\| > 0$ and $\langle -\nabla f(\theta), \theta' - \theta \rangle \geq 0$ for every $\theta' \in \mathbb{B}(\theta_{0}, \epsilon/\kappa)$.
\end{enumerate}
\end{repdefinition}
Observe that the notion of SLQC implies quasi-convexity about $\mathbb{B}(\theta_{0}, \epsilon/\kappa)$ on $\{\theta\in\Theta : f(\theta) - f(\theta_{0}) > \epsilon\}$; see Figure \ref{fig:quasiconvexityvsSLQC} for an illustration of the difference between classical quasi-convexity and SLQC in this regime.
In \cite{hazan2015beyond}, Hazan \textit{et al.} note that if a function $f$ is $G$-Lipschitz and strictly-quasi-convex, then for all $\tilde{\theta}_{1},\tilde{\theta}_{2} \in \mathbb{R}^{d}$, for all $\epsilon > 0$, it holds that $f$ is $(\epsilon,G,\tilde{\theta}_{1})$-SLQC at $\tilde{\theta}_{2}$ for every $\tilde{\theta}_{2} \in \mathbb{R}^{d}$; this will be useful in the sequel.

As shown by Hazan \textit{et al.} in \cite{hazan2015beyond}, the convergence guarantees of Normalized Gradient Descent (Algorithm~\ref{algo:NGD}) for SLQC functions are similar to those of Gradient Descent for convex functions.
%
\begin{algorithm}[h]
\caption{Normalized Gradient Descent (NGD)}\label{algo:NGD}
\begin{algorithmic}[1]
\State \textbf{Input:} $T\in\mathbb{N}$ no.\ of iterations, $\theta_{0} \in \mathbb{R}^{d}$ initial parameter, $\eta > 0$ learning rate
\For {$t = 0, 1, \ldots, T-1$}
\State Update: $\theta_{t+1} = \theta_{t} - \eta \dfrac{\nabla f(\theta_{t})}{\|\nabla f(\theta_{t})\|}$
\EndFor
\State \textbf{Return} $\bar{\theta}_{T} = \argmin\limits_{\theta_{1}, \ldots, \theta_{T}} f(\theta_{t})$
\end{algorithmic}
\end{algorithm}

\begin{repprop}{prop:NGDiterations}[{Theorem~4.1},\cite{hazan2015beyond}]
Let $f: \mathbb{R}^{d} \rightarrow \mathbb{R}$, and $\theta^{*} = \argmin_{\theta \in \mathbb{R}^{d}} f(\theta)$. If $f$ is $(\epsilon, \kappa, \theta^{*})$-SLQC at $\theta$ for every $\theta \in \mathbb{R}^{d}$, then Algorithm \ref{algo:NGD} with $T \geq \kappa^{2}\|\theta_{1} - \theta^{*}\|^{2}/\epsilon^{2}$ and $\eta = \epsilon/\kappa$ satisfies that $f(\bar{\theta}_{T}) - f(\theta^{*}) \leq \epsilon$.
\end{repprop}
For an $(\epsilon,\kappa,\theta_0)$-SLQC function, a smaller $\epsilon$ provides better optimality guarantees. Given $\epsilon>0$, smaller $\kappa$ leads to faster optimization as the number of required iterations increases with $\kappa^2$.
Hazan, \textit{et al.}~\cite{hazan2015beyond} show that if a function $f$ is $G$-Lipschitz and strictly-quasi-convex, then for all $\epsilon > 0$, $f$ is $(\epsilon,G,\theta_{0})$-SLQC in $\theta$. Thus, one may view $\kappa$ as approximately quantifying the growth of the gradients of general functions.
Finally, by using projections, NGD can be easily adapted to work over convex and closed sets (e.g., $\mathbb{B}(\theta_0,r)$ for some $\theta_0\in\mathbb{R}^d$ and $r>0$).

We conclude this subsection by studying the behavior of $(\epsilon,\kappa,\theta_0)$-SLQC functions on the ball $\overline{\mathbb{B}_{d}(\theta_0,\epsilon/\kappa)}$, which is articulated by the following novel result.
\begin{prop} \label{prop:epsregioncontainsball}
Let $\epsilon,\kappa>0$ and $\theta_0\in\mathbb{R}^d$.
Suppose that $f$ is $(\epsilon,\kappa,\theta_0)$-SLQC at $\theta \in \mathbb{R}^{d}$.
If $\theta\in\mathbb{B}_{d}(\theta_0,\epsilon/\kappa)$, then $f(\theta) - f(\theta_0) \leq \epsilon$. In particular, if $f$ is $(\epsilon,\kappa,\theta_0)$-SLQC on $\Theta$, then 
\begin{equation*}
    \overline{\mathbb{B}_{d}(\theta_0,\epsilon/\kappa) \cap \Theta} \subset \{\theta\in\Theta : f(\theta) - f(\theta_0) \leq \epsilon\}.
\end{equation*}
\end{prop}

\begin{proof}
Since $f$ is $(\epsilon,\kappa,\theta_0)$-SLQC at $\theta\in\mathbb{R}^d$ we have that at least one condition of Definition~\ref{def:SLQC} holds. Suppose that Condition 2 holds. In this case, we have that $\|\nabla f(\theta)\| > 0$ and $\langle -\nabla f(\theta), \theta' - \theta \rangle \geq 0$ for every $\theta' \in \mathbb{B}(\theta_{0}, \epsilon/\kappa)$. 
Since $\|\theta - \theta_0\| < \epsilon/\kappa$, choose $\delta > 0$ small enough such that
\begin{equation}
    \theta' := \theta + \delta \nabla f(\theta) \in \mathbb{B}(\theta_{0}, \epsilon/\kappa).
\end{equation}
Thus, we have that
\begin{align}
\nonumber 0 &\leq \langle -\nabla f(\theta), \theta' - \theta \rangle \\
\nonumber & = \langle -\nabla f(\theta), \theta + \delta \nabla f(\theta) - \theta \rangle \\
\nonumber & = - \delta \langle \nabla f(\theta), \nabla f(\theta) \rangle \\
& = - \delta \|\nabla f(\theta)\|^{2},
\end{align}
which is a contradiction since $\delta>0$ and $\|\nabla f(\theta)\| > 0$.
Therefore, we must have that Condition~1 of Definition~\ref{def:SLQC} holds, i.e., $f(\theta) - f(\theta_0) \leq \epsilon$. Finally, a continuity argument shows that $f(\theta) - f(\theta_0) \leq \epsilon$ whenever $\theta\in\overline{\mathbb{B}_{d}(\theta_0,\epsilon/\kappa) \cap \Theta}$.
\end{proof}

The following is the formal statement and proof of Lemma~\ref{lemma:SLQCreformulation}, which provides a useful characterization of the gradient of $(\epsilon,\kappa,\theta_0)$-SLQC functions outside the set $\overline{\mathbb{B}_{d}(\theta_{0},\epsilon/\kappa)}$. 
Refer to Figure~\ref{fig:LemmaGoodGradients} for a visual description of the relavant quantities.
\begin{figure}[h]
\centering
\begin{tikzpicture}
        \coordinate (A) at (5, 0) {};
        \coordinate (B) at (4.5500,-1.4309) {};
        \coordinate (C) at (4.4500, -1.4000) {};
\filldraw[color=black, fill=none, very thick](5,0) circle (1.5); 
\filldraw[black] (5,0) circle (2pt) node[anchor=west] {$\theta_{0}$};
\draw[black, very thick] (0,0) -- (5,0);
\filldraw [black] (0,0) circle (2pt) node[anchor=east] {$\theta$};
\draw[black, very thick] (0,0) -- (4.5500,-1.4309);
\filldraw [black] (4.5500,-1.4309) circle (2pt) node[anchor=north] {$\theta'$};
\draw[black, very thick] (5,0) -- (4.5500,-1.4309) node[right,midway] {$\rho$};
\tkzMarkRightAngle[draw=black,size=.25, very thick](A,B,C);
\draw[black, very thick, ->, name path=B] (0,0) -- (1.25,1.0625) node[anchor=south] {$-\nabla f(\theta)$}; 
\draw[black, thick] (.3,0) arc (0:40:.3) node[anchor=west] {$\psi$};
\draw[black, thick] (1,-.33) arc (-30:50:.8) node[anchor=west,midway] {$\phi$};
\draw[black, thick] (1.5,-.475) arc (-45:27.5:.4) node[anchor=west,midway] {$\delta$};
%
%
\end{tikzpicture}
\caption{A companion illustration for Lemma~\ref{lemma:SLQCreformulation} which depicts the relevant quantities involved. Note that there are three different configurations of the angles $\delta$, $\phi$ and $\psi$. Refer to Figure~\ref{fig:slqcreformulationproof} for this illustration.}
    \label{fig:LemmaGoodGradients}
\end{figure}

\begin{replemma}{lemma:SLQCreformulation}
Assume that $f:\mathbb{R}^d\to\mathbb{R}$ is differentiable, $\theta_{0} \in \mathbb{R}^d$ and $\rho>0$. If $\theta\in\mathbb{R}^d$ is such that $\|\theta-\theta_0\| > \rho$ and $\|\nabla f(\theta)\| > 0$, then the following are equivalent:
\begin{itemize}
    \item[\textnormal{(1)}] $\langle -\nabla f(\theta), \theta' - \theta \rangle > 0$ for all $\theta' \in \mathbb{B}_{d}\left(\theta_{0},\rho\right)$;
    
    \item[\textnormal{(2)}] $\langle -\nabla f(\theta), \theta' - \theta \rangle \geq 0$ for all $\theta' \in \mathbb{B}_{d}\left(\theta_{0},\rho\right)$;
    
    \item[\textnormal{(3)}] $\langle -\nabla f(\theta), \theta_{0} - \theta \rangle \geq \rho \|\nabla f(\theta)\|$. 
\end{itemize}
\end{replemma}

    
    

\begin{proof}
Clearly \textnormal{(1)} $\Rightarrow$ \textnormal{(2)}.
\textnormal{(2)} $\Rightarrow$ \textnormal{(3)}: Let $\theta'$ be the point of tangency of a line tangent to $\overline{\mathbb{B}_{d}(\theta_0,\rho)}$ passing through $\theta$, as depicted in Figure~\ref{fig:LemmaGoodGradients}. We define
\begin{itemize}
    \item[$\delta$:] the angle between $\theta_{0} - \theta$ and $\theta' - \theta$;
    
    \item[$\phi$:] the angle between $-\nabla f(\theta)$ and $\theta' - \theta$;
    
    \item[$\psi$:] the angle between $-\nabla f(\theta)$ and $\theta_0 - \theta$.
\end{itemize}
Recall that the inner product satisfies that
\begin{equation}
\label{eq:GoodGradientsInnerProduct}
    \langle u, v\rangle = \|u\| \|v\| \cos(\varphi_{u,v}),
\end{equation}
where $\varphi_{u,v}\in[0,\pi]$ is the angle between $u$ and $v$. By continuity and Condition~(2),
\begin{equation}
    \|\nabla f(\theta)\| \|\theta' - \theta\| \cos(\phi) = \langle -\nabla f(\theta), \theta' - \theta\rangle \geq 0,
\end{equation}
which implies that $\phi \leq \frac{\pi}{2}$. Observe that, by construction, $\phi = \psi + \delta$. In particular, we have that $\psi \leq \frac{\pi}{2} - \delta$. Since $\cos(\cdot)$ is decreasing over $[0,\pi]$, we have that
\begin{equation}
    \cos(\psi) \geq \cos\left(\frac{\pi}{2} - \delta\right) = \sin(\delta).
\end{equation}
Since the triangle $\triangle \theta\theta'\theta_0$ is a right triangle, we have that $\sin(\delta) = \frac{\rho}{\|\theta_{0} - \theta\|}$ and thus
\begin{equation}
\label{eq:GoodGradientsPsi}
    \cos(\psi) \geq \frac{\rho}{\|\theta_{0} - \theta\|}.
\end{equation}
Therefore, we conclude that
\begin{equation}
    \langle -\nabla f(\theta), \theta_{0} - \theta\rangle = \|\nabla f(\theta)\| \|\theta_{0} - \theta\| \cos(\psi) \geq \rho \|\nabla f(\theta)\|,
\end{equation}
as we wanted to prove.

\textnormal{(3)} $\Rightarrow$ \textnormal{(1)}: For a given $\theta'\in\mathbb{B}_{d}(\theta_{0},\rho)$, we define $\psi$, $\phi$ and $\delta$ as above. By assumption,
\begin{equation}
\label{eq:GoodGradientsPrePsiBound}
    \|\nabla f(\theta)\| \|\theta_{0} - \theta\| \cos(\psi) = \langle -\nabla f(\theta), \theta_{0} - \theta\rangle \geq \rho \|\nabla f(\theta)\| \geq 0.
\end{equation}
Since $\cos^{-1}(\cdot)$ is decreasing over $[-1,1]$, \eqref{eq:GoodGradientsPrePsiBound} implies that
\begin{equation}
\label{eq:GoodGradientsPsiBound}
    \psi \leq \cos^{-1}\left(\frac{\rho}{\|\theta_{0} - \theta\|}\right).
\end{equation}
Also, an immediate application of the law of cosines shows that
\begin{equation}
    \delta = \cos^{-1}\left(\frac{\|\theta_{0} - \theta\|^2 + \|\theta' - \theta\|^2 - \|\theta' - \theta_{0}\|^2}{2 \|\theta_{0} - \theta\| \|\theta' - \theta\|}\right).
\end{equation}
Since $\|\theta' - \theta_{0}\| < \rho$, we have that
\begin{equation}
\label{eq:GoodGradientsOptimizationDelta}
    \delta < \cos^{-1}\left(\frac{\|\theta_{0} - \theta\|^2 + \|\theta' - \theta\|^2 - \rho^2}{2 \|\theta_{0} - \theta\| \|\theta' - \theta\|}\right).
\end{equation}
A routine minimization argument further implies that
\begin{equation}
\label{eq:GoodGradientsDeltaBound}
    \delta < \cos^{-1}\left(\sqrt{1 - \left(\frac{\rho}{\|\theta_{0} - \theta\|}\right)^2}\right) = \sin^{-1}\left(\frac{\rho}{\|\theta_{0} - \theta\|}\right),
\end{equation}
where the equality follows from the trigonometric identity $\cos(\sin^{-1}(x)) = \sqrt{1 - x^2}$. Observe that, in order to prove that
\begin{equation}
    \langle -\nabla f(\theta), \theta' - \theta\rangle = \|\nabla f(\theta)\| \|\theta'- \theta\| \cos(\phi) > 0,
\end{equation}
it is enough to show that $\phi < \frac{\pi}{2}$. Depending on the position of $\theta'$, the angles $\delta$, $\phi$ and $\psi$ can be arranged in three different configurations, as depicted in Figure~\ref{fig:slqcreformulationproof}. 
\begin{itemize}
    \item[a)] Since $\frac{\rho}{\|\theta_{0} - \theta\|} > 0$, \eqref{eq:GoodGradientsPsiBound} implies that $\psi < \frac{\pi}{2}$. Therefore, $\phi < \frac{\pi}{2}$ as $\phi \leq \psi$.
    
    \item[b)] Since $\frac{\rho}{\|\theta_{0} - \theta\|} < 1$, \eqref{eq:GoodGradientsDeltaBound} implies that $\delta < \frac{\pi}{2}$. Therefore, $\phi < \frac{\pi}{2}$ as $\phi \leq \delta$.
    
    \item[c)] Since $\sin^{-1}(x) + \cos^{-1}(x) = \frac{\pi}{2}$, \eqref{eq:GoodGradientsPsiBound} and \eqref{eq:GoodGradientsDeltaBound} imply that $\phi = \psi + \delta < \frac{\pi}{2}$.
\end{itemize}
Since in all cases $\phi < \frac{\pi}{2}$, the result follows.
\end{proof}


\begin{figure}[h]
\centering
\begin{tikzpicture}
        \coordinate (A) at (-5, 0) {};
        \coordinate (B) at (0,0) {};
        \coordinate (C) at (5, 0) {};
\draw[black, very thick] (A) -- (-2,0) node[anchor = north,midway] {a)};
\draw[black, very thick] (A) -- (-2,2);
\draw[black, very thick] (A) -- (-2,1);
\draw[black, thick] (-4,0) arc (-30:50:.25) node[pos=.6,right] {$\phi$};
\draw[black, thick] (-3.25,0) arc (-30:62.5:.7) node[near start, right] {$\psi$};
\draw[black, thick] (-2.5,.8333) arc (-30:62.5:.5) node[pos=.6,right] {$\delta$};
\draw[black, very thick] (B) -- (3,0) node[anchor = north,midway] {b)};
\draw[black, very thick] (B) -- (3,2);
\draw[black, very thick] (B) -- (3,1);
\draw[black, thick] (1,0) arc (-30:50:.25) node[pos=.6,right] {$\psi$};
\draw[black, thick] (1.75,0) arc (-30:62.5:.7) node[near start, right] {$\delta$};
\draw[black, thick] (2.5,.8333) arc (-30:62.5:.5) node[pos=.6,right] {$\phi$};
\draw[black, very thick] (C) -- (8,0) node[anchor = north,midway] {c)};
\draw[black, very thick] (C) -- (8,2);
\draw[black, very thick] (C) -- (8,1);
\draw[black, thick] (6,0) arc (-30:50:.25) node[pos=.6,right] {$\delta$};
\draw[black, thick] (6.75,0) arc (-30:62.5:.7) node[near start, right] {$\phi$};
\draw[black, thick] (7.5,.8333) arc (-30:62.5:.5) node[pos=.6,right] {$\psi$};
\end{tikzpicture}
\caption{Three different configurations of the angles $\delta$, $\phi$ and $\psi$.}
    \label{fig:slqcreformulationproof}
\end{figure}

\subsubsection{Lipschitz Inequalities in $\alpha^{-1}$ and Main SLQC Result for the $\alpha$-risk} \label{appen:lemmas&mainSLQC}

\begin{replemma}{lemma:inversealphalip}
If $\alpha, \alpha' \in [1,\infty]$, then for all $\theta \in \mathbb{B}_{d}(r)$,
\begin{align}
|R_{\alpha}(\theta) - R_{\alpha'}(\theta)| \leq L_{d}(\theta) \left|\frac{\alpha-\alpha'}{\alpha \alpha'}\right| \quad \text{and} \quad \|\nabla R_{\alpha}(\theta) - \nabla R_{\alpha'}(\theta)\| \leq J_{d}(\theta) \left|\frac{\alpha- \alpha'}{\alpha \alpha'}\right|,
\end{align}
where, 
\begin{align} 
L_{d}(\theta) := \dfrac{\left(\log{\left(1 + e^{\|\theta\|\sqrt{d}}\right)}\right)^{2}}{2} \quad \text{and} \quad J_{d}(\theta) := \sqrt{d} \log{\left(1+e^{\|\theta \| \sqrt{d}}\right)} \sigma(\|\theta\|\sqrt{d}).
\end{align}
\end{replemma}
\begin{proof} 
Here, we present proofs for both Lipschitz inequalities. 
\newline \textbf{Proof of First Inequality:}
For ease of notation, we let $\beta = 1/\alpha$. 
Thus, we have that for $\alpha \in [1,\infty]$, i.e., $\beta \in [0,1]$,
\begin{equation}
R_{\alpha}(\theta) = \mathbb{E}[l^{\alpha}(Y,g_{\theta}(X))] = \mathbb{E}\left[\frac{1}{1-\beta}\left(1 - g_{\theta}(yx)^{1-\beta}\right)\right] = R_{\beta}(\theta).
\end{equation}
To show that $R_{\alpha}$ is Lipschitz in $\alpha^{-1} = \beta \in [0,1]$, it suffices to show $\dfrac{d}{d\beta} R_{\beta}(\theta) \leq L$ for some $L > 0$.
Observe that 
\begin{equation}
\dfrac{d}{d\beta} R_{\beta}(\theta) = \mathbb{E}\left[\dfrac{d}{d\beta} \frac{1}{1-\beta}\left(1 - g_{\theta}(yx)^{1-\beta}\right)\right],
\end{equation}
where the equality follows since we assume well-behaved integrals. 
Consider without loss of generality the expression in the brackets; we denote this expression as 
\begin{equation}
f(\beta,\theta,yx) =\dfrac{d}{d\beta} \frac{1}{1-\beta}\left(1 - g_{\theta}(yx)^{1-\beta}\right).
\end{equation}
It can be shown that 
\begin{align}
f(\beta,\theta,yx) = \dfrac{g_{\theta}(yx)^{1-\beta}\log\left(g_{\theta}(yx)\right)}{1-\beta}+\dfrac{1-g_{\theta}(yx)^{1-\beta}}{\left(1-\beta\right)^2}
\end{align}
and 
\begin{equation}
f(1,\theta,yx) = \frac{\left(\log{g_{\theta}(yx)}\right)^{2}}{2}.
\end{equation}
In addition, it can be shown that for any $y \in \{-1,+1\}$, $x \in [0,1]^{d}$, and $\theta \in \mathbb{B}_{d}(r)$ that $f(\beta,\theta,yx)$ is monotonically increasing in $\beta \in [0,1]$. 
Therefore, for any $\beta \in [0,1]$, $y \in \{-1,+1\}$, $x \in [0,1]^{d}$, and $\theta \in \mathbb{B}_{d}(r)$,
\begin{equation}
f(\beta,\theta,yx) \leq f(1,\theta,yx) = \frac{\left(\log{g_{\theta}(yx)}\right)^{2}}{2} \leq \frac{\left(\log{\sigma(- \|\theta\|\sqrt{d})}\right)^{2}}{2}.
\end{equation}

\noindent \textbf{Proof of Second Inequality:}
For ease of notation, we let $\beta = 1/\alpha$. Since $\alpha \in [1,\infty]$, $\beta \in [0,1]$. Thus, we have that for $\alpha \in [1,\infty]$, i.e., $\beta \in [0,1]$,
\begin{equation}
\nabla R_{\alpha}(\theta) = \mathbb{E}[F_{1}(\alpha,\theta,X,Y)X] = \mathbb{E}[- Y g_\theta(YX)^{1-\beta}(1-g_\theta(YX))X],
\end{equation}
and we let $\tilde{F}_1(\beta,\theta,X,Y) := - Y g_\theta(YX)^{1-\beta}(1-g_\theta(YX))$.
For any $\theta \in \mathbb{B}_{d}(r)$ we have 
\begin{align}
\nonumber \|\nabla R_{{\alpha}} (\theta) -  \nabla R_{{\alpha'}} (\theta)\| &= \|\mathbb{E}[(\tilde{F}_1(\beta,\theta,X,Y) - \tilde{F}_1(\beta',\theta,X,Y))X] \|\\ 
\nonumber &\leq \mathbb{E}[|(\tilde{F}_1(\beta,\theta,X,Y) - \tilde{F}_1(\beta',\theta,X,Y))|\|X\|]\\ 
&\leq \sqrt{d} \mathbb{E}[|(\tilde{F}_1(\beta,\theta,X,Y) - \tilde{F}_1(\beta',\theta,X,Y))|],
\end{align}
where we used the fact that $X$ has support $[0,1]^{d}$ for the second inequality.
Here, we obtain a Lipschitz inequality on $\tilde{F}_{1}$ by considering the variation of $\tilde{F}_{1}$ with respect to $\beta$ for any $\theta \in \mathbb{B}_{d}(r)$, $x \in [0,1]^{d}$, and $y \in \{-1,+1\}$. 
Taking the derivative of $\tilde{F}_{1}(\beta,\theta,x,y)$ with respect to $\beta$ we obtain
\begin{align}
\nonumber \dfrac{d}{d \beta} F_{1}(\beta,\theta,x,y) &= \dfrac{d}{d \beta} - y g_\theta(yx)^{1-\beta}(1-g_\theta(yx)) \\
\nonumber &= y(1-g_\theta(yx)) g_\theta(yx)^{1-\beta} \log{g_{\theta}(yx)},
\end{align}
where we used the fact that $\dfrac{d}{dx} a^{1-x} = -a^{1-x} \log{a}$. Continuing, we have
\begin{align}
y(1-g_\theta(yx)) g_\theta(yx)^{1-\beta} \log{g_{\theta}(yx)} \nonumber &\leq \log{\left(1+e^{\|\theta \| \sqrt{d}}\right)} \sigma(\|\theta \| \sqrt{d}) \sigma(\|\theta\|\sqrt{d})^{1-\beta} \\
&= \log{\left(1+e^{\|\theta \| \sqrt{d}}\right)} \sigma(\|\theta\|\sqrt{d})^{2-\beta} \\
&\leq \log{\left(1+e^{\|\theta \| \sqrt{d}}\right)} \sigma(\|\theta\|\sqrt{d}).
\end{align}
Thus, we have that, for any $\theta \in \mathbb{B}_{d}(r)$, 
\begin{equation}
\|\nabla R_{{\alpha}} (\theta) -  \nabla R_{{\alpha'}} (\theta)\| \leq J_{d}(\theta) |\beta - \beta'|,
\end{equation}
where $\beta, \beta' \in [0,1]$ ($\alpha, \alpha' \in [1,\infty]$). Therefore, we have that, for any $\theta \in \mathbb{B}_{d}(r)$, 
\begin{equation}
\|\nabla R_{{\alpha}} (\theta) -  \nabla R_{{\alpha'}} (\theta)\| \leq J_{d}(\theta) \left|\dfrac{1}{\alpha} - \dfrac{1}{\alpha'}\right|,
\end{equation}
where $\alpha, \alpha' \in [1,\infty]$. 
\end{proof}
\begin{figure}[h]
    \centering
    \centerline{\includegraphics[width=.5\linewidth]{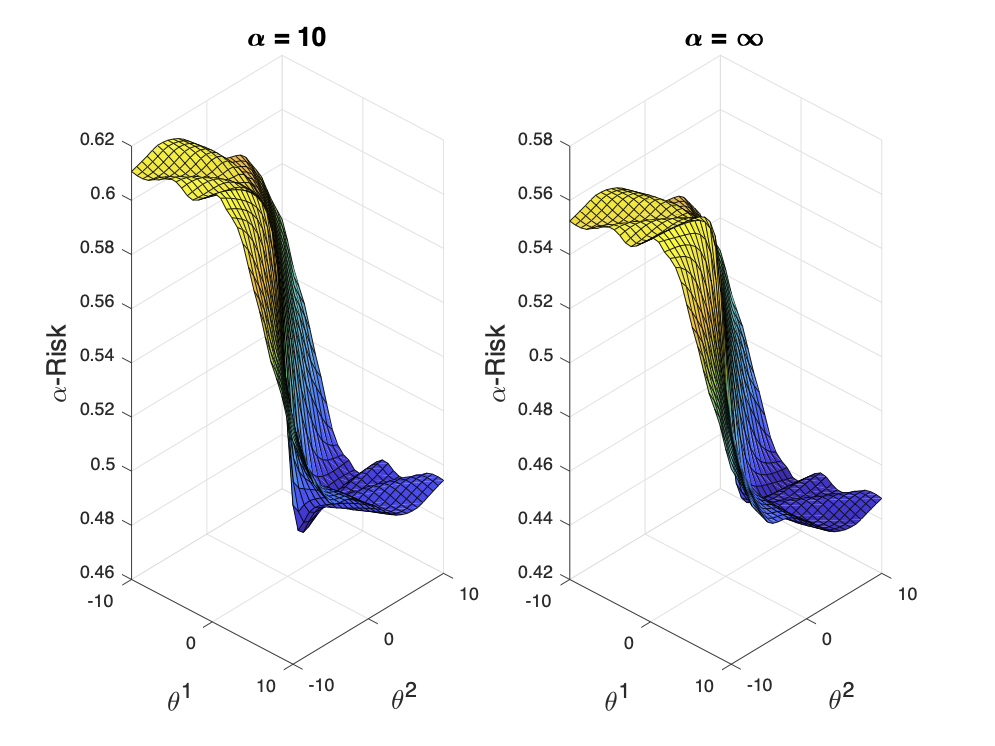}}
    \caption{Another illustration highlighting the saturation of $\alpha$-loss ($R_{\alpha}$ for $\alpha = 10, \infty$) in the logistic model for a 2D-GMM with $\mathbb{P}[Y=1]=\mathbb{P}[Y=-1]$, $\mu_{X|Y=-1} =[.5, .5]$, $\mu_{X|Y=1} = [1,1]$, and shared covariance matrix $\Sigma = [1, .5; .5, 3]$. }
\end{figure}
\begin{figure}[h] 
    \centering
        \begin{subfigure}[b]{0.475\textwidth}
            \centering
            \includegraphics[width=.85\textwidth]{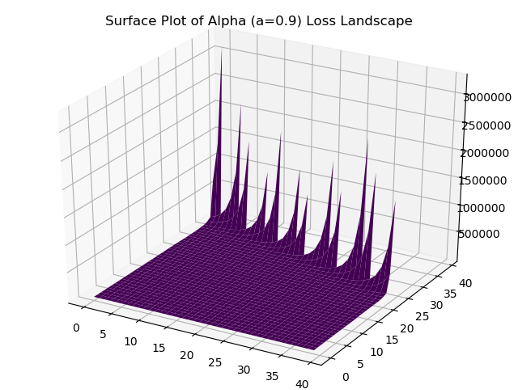}
            \caption[Network2]%
            {{\small $\alpha = .9$ loss landscape}}    
            \label{fig:mean and std of net14}
        \end{subfigure}
        \hfill
        \begin{subfigure}[b]{0.475\textwidth}  
            \centering 
            \includegraphics[width=.85\textwidth]{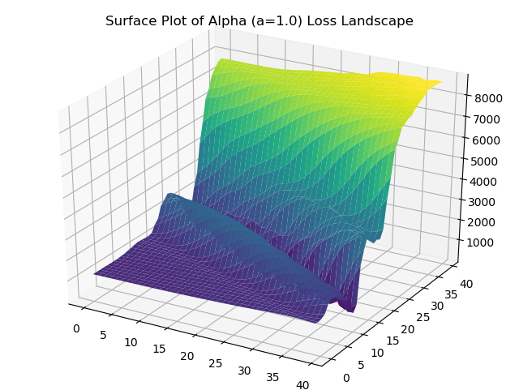}
            \caption[]%
            {{\small $\alpha = 1$ loss landscape}}    
            \label{fig:mean and std of net24}
        \end{subfigure}
        \vskip\baselineskip
        \begin{subfigure}[b]{0.475\textwidth}   
            \centering 
            \includegraphics[width=.85\textwidth]{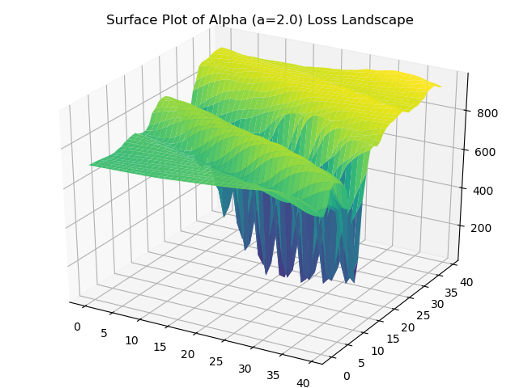}
            \caption[]%
            {{\small $\alpha = 2$ loss landscape}}    
            \label{fig:mean and std of net34}
        \end{subfigure}
        \hfill
        \begin{subfigure}[b]{0.475\textwidth}   
            \centering 
            \includegraphics[width=.85\textwidth]{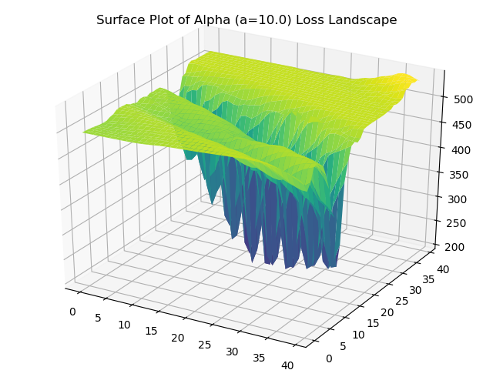}
            \caption[]%
            {{\small $\alpha = 10$ loss landscape}}    
            \label{fig:mean and std of net44}
        \end{subfigure}
    \caption{
    Loss landscape visualizations using~\cite{li2017visualizing} for $\alpha \in \{.9,1,2,10\}$ training a ResNet-18 on the MNIST dataset. The visualization technique finds two ``principal directions'' of the model to allow for a 3D plot. We note that similar themes as theoretically articulated in Section~\ref{sec:landscapelogisticmodel} for the simpler logistic model are also evident here; i.e., exploding gradients for $\alpha$ too small, a loss of convexity (and increasing ``flatness'') as $\alpha$ increases greater than $1$, and also a saturation effect as exhibited by the visual similarity between the $\alpha=2$ and $\alpha =10$ loss landscapes. This hints at the generality of the theory presented in Section~\ref{sec:landscapelogisticmodel}.
}
    \label{fig:fourmorelandscapes}
\end{figure}
\begin{reptheorem}{thm:SLQCresult1}
Let $\alpha_{0} \in [1,\infty]$, $\epsilon_{0}, \kappa_{0}>0$, and $\theta_{0}, \theta \in \mathbb{B}_{d}(r)$.  
If $R_{{\alpha_{0}}}$ is $(\epsilon_{0},\kappa_{0},\theta_{0})$-SLQC at $\theta$ and
\begin{equation} 
0 \leq \alpha-\alpha_{0} < \dfrac{\alpha_{0}^{2} \|\nabla R_{\alpha_{0}}(\theta)\|}{2J_{d}(\theta)\left(1 + r \frac{\kappa_{0}}{\epsilon_{0}}\right)}, 
\end{equation}
then $R_{{\alpha}}$ is $(\epsilon,\kappa,\theta_{0})$-SLQC at $\theta$ with 
\begin{equation} \label{eq:thmeps}
\epsilon = \epsilon_{0} +2L_{d}(\theta)\left(\frac{\alpha-\alpha_{0}}{\alpha \alpha_{0}}\right),
\end{equation}
and 
\begin{align} \label{eq:thmrho}
\dfrac{\epsilon}{\kappa} = \frac{\epsilon_{0}}{\kappa_{0}} \left(1  - \frac{\left(1 +2r\frac{\kappa_{0}}{\epsilon_{0}}\right)J_{d}(\theta)(\alpha-\alpha_{0})}{\alpha\alpha_{0}\|\nabla R_{\alpha_{0}}(\theta)\| - J_{d}(\theta)(\alpha-\alpha_{0})}\right).
\end{align}
\end{reptheorem}
\begin{proof}
For ease of notation let $\displaystyle \rho_{0} = \frac{\epsilon_{0}}{\kappa_{0}}$ and $\displaystyle \rho = \frac{\epsilon}{\kappa}$, and consider the following two cases.

\noindent \textbf{Case 1}: Assume that $R_{\alpha_{0}}(\theta) - R_{\alpha_{0}}(\theta_{0}) \leq \epsilon_{0}$. Then,
\begin{align}
\nonumber R_{\alpha}(\theta) - R_{\alpha}(\theta_{0}) &= R_{\alpha}(\theta) - R_{\alpha_{0}}(\theta) + R_{\alpha_{0}}(\theta) - R_{\alpha_{0}}(\theta_{0}) + R_{\alpha_{0}}(\theta_{0}) - R_{\alpha}(\theta_{0}) \\
&\leq L_{d}(\theta) \left(\frac{\alpha-\alpha_{0}}{\alpha \alpha_{0}}\right) + \epsilon_{0} + L_{d}(\theta)\left(\frac{\alpha-\alpha_{0}}{\alpha \alpha_{0}}\right).
\end{align}
Since $\epsilon_{0} +2L_{d}(\theta)\left(\frac{\alpha-\alpha_{0}}{\alpha \alpha_{0}}\right) = \epsilon$, we have $R_{\alpha}(\theta) - R_{\alpha}(\theta_{0}) \leq \epsilon$. 

\noindent \textbf{Case 2}:
Assume that $R_{\alpha_{0}}(\theta) - R_{\alpha_{0}}(\theta_{0}) > \epsilon_{0}$. Since $R_{\alpha_{0}}$~is $(\epsilon_{0},\kappa_{0},\theta_{0})$-SLQC at $\theta$ by assumption, we have that $\|\nabla R_{\alpha_{0}}(\theta)\| > 0$ and $\langle -\nabla R_{\alpha_{0}}(\theta), \theta' - \theta \rangle \geq 0$ for every $\theta' \in \mathbb{B}(\theta_{0}, \rho_{0})$.

Let $\rho = \epsilon/\kappa$ be given as in \eqref{eq:thmrho}. If we can show that $\|\theta - \theta_{0}\|>\rho$, $\|\nabla R_{\alpha}(\theta)\|>0$ and
\begin{equation}
\label{eq:ProofSLQCresult1}
\langle-\nabla R_{\alpha}(\theta), \theta_{0} - \theta \rangle \geq \rho \|\nabla R_{\alpha}(\theta) \|,
\end{equation}
then Lemma~\ref{lemma:SLQCreformulation}
would imply that $R_{\alpha}$ is $(\epsilon,\kappa,\theta_{0})$-SLQC at $\theta$. 
In order to show these three expressions, we make ample use of the following three inequalities:
The first is the reverse triangle inequality associated with $\nabla R_{\alpha}$ and $\nabla R_{\alpha_{0}}$, i.e., 
\begin{align} \label{eq:reversetriangleinequality}
\|\nabla R_{\alpha_{0}}(\theta) - \nabla R_{\alpha}(\theta)\| \geq \lvert \|\nabla R_{\alpha}(\theta)\| - \|\nabla R_{\alpha_{0}}(\theta)\| \lvert.
\end{align}
The second is the fact that $\nabla R_{\alpha}(\theta)$ is $J_{d}(\theta)$-Lipschitz in $\alpha^{-1}$, i.e.,
\begin{align} \label{eq:slqcjdlip}
\left|\frac{1}{\alpha_{0}} - \frac{1}{\alpha} \right| J_{d}(\theta) \geq \|\nabla R_{\alpha_{0}}(\theta) - \nabla R_{\alpha}(\theta)\|.
\end{align}
The third follows from a manipulation of~\eqref{eq:radchange}, i.e.,
\begin{equation} \label{eq:slqclowerbound}
    \lVert \nabla R_{\alpha_0}(\theta) \rVert > 2J_{d}(\theta) \left(1 + r \rho_{0}^{-1} \right) (\alpha_{0}^{-1} - \alpha^{-1}) > J_{d}(\theta) (\alpha_{0}^{-1} - \alpha^{-1}),
\end{equation}
which uses the fact that $\alpha_{0}^{2} \leq \alpha \alpha_{0}$ and since $r\rho_{0}^{-1} \geq 1$.
With these inequalities in hand, we are now in a position to complete the three steps required to show that $R_{\alpha}$ is $(\epsilon,\kappa,\theta_{0})$-SLQC at $\theta$. 

First, we show that $\|\theta - \theta_{0}\|>\rho$. 
Since $R_{{\alpha_{0}}}$ is $(\epsilon_{0},\kappa_{0},\theta_{0})$-SLQC at $\theta$ and $R_{\alpha_{0}}(\theta) - R_{\alpha_{0}}(\theta_{0}) > \epsilon_{0}$ by assumption, we have by the contrapositive of Proposition~\ref{prop:epsregioncontainsball} that $\theta \notin \mathbb{B}_{d}(\theta_{0},\rho_{0})$.
Thus, we have that $\|\theta - \theta_{0}\|>\rho_{0}$.
Next, note that $\rho$ is related to $\rho_{0}$ by~\eqref{eq:thmrho}. If we can show that $\rho_{0} > \rho$, then we have the desired conclusion. 
Rearranging the left-hand-side of~\eqref{eq:slqclowerbound}, we have that 
\begin{align}
\|\nabla R_{\alpha_{0}}(\theta)\| (\alpha_{0}^{-1}-\alpha^{-1})^{-1} > 2 J_{d}(\theta) (1+r\rho_{0}^{-1}),
\end{align}
which can be rewritten to obtain
\begin{align} \label{eq:slqcthmint1}
\|\nabla R_{\alpha_{0}}(\theta)\| (\alpha_{0}^{-1}-\alpha^{-1})^{-1} - J_{d}(\theta) > J_{d}(\theta) (1+2r\rho_{0}^{-1}).
\end{align}
Since by the right-hand-side of~\eqref{eq:slqclowerbound} we have that 
\begin{align}
  \lVert \nabla R_{\alpha_0}(\theta) \rVert (\alpha_{0}^{-1} - \alpha^{-1})^{-1} - J_{d}(\theta) > 0,
\end{align}
it follows by~\eqref{eq:slqcthmint1} that 
\begin{align} \label{eq:slqcthmint2}
1 > \frac{J_{d}(\theta) (1+2r\rho_{0}^{-1})}{\|\nabla R_{\alpha_{0}}(\theta)\| (\alpha_{0}^{-1}-\alpha^{-1})^{-1} - J_{d}(\theta)}.
\end{align}
Thus examining~\eqref{eq:thmrho} in light of~\eqref{eq:slqcthmint2}, we have that $\rho_{0} > \rho$, which implies that 
$\|\theta - \theta_{0}\|>\rho$, as desired.

Second, we show that $\|\nabla R_{\alpha}(\theta)\|>0$. 
Applying~\eqref{eq:reversetriangleinequality} to~\eqref{eq:slqcjdlip} we obtain 
\begin{align} \label{eq:thm2_3}
\|\nabla R_{\alpha}(\theta)\| \geq \|\nabla R_{\alpha_{0}}(\theta)\| - J_{d}(\theta)(\alpha_{0}^{-1} - \alpha^{-1}) > 0,
\end{align}
where the right-hand-side inequality again follows by~\eqref{eq:slqclowerbound}.
Thus, we have that $\|\nabla R_{\alpha}(\theta) \| > 0$, as desired.

Finally, we show the expression in~\eqref{eq:ProofSLQCresult1}, i.e., $\langle-\nabla R_{\alpha}(\theta), \theta_{0} - \theta \rangle \geq \rho \|\nabla R_{\alpha}(\theta) \|$.
%
%
By the Cauchy-Schwarz inequality, 
\begin{align}
\label{eq:slqcint3} \langle -\nabla R_{\alpha}(\theta),\theta_{0} - \theta \rangle &\geq \langle -\nabla R_{\alpha_{0}}(\theta),  \theta_{0} - \theta \rangle - \|\nabla R_{\alpha}(\theta) - \nabla R_{\alpha_{0}}(\theta)\|\|\theta_{0} - \theta\| \\
\label{eq:slqcint4} &\geq \rho_{0} \|\nabla R_{\alpha_{0}}(\theta)\| - J_{d}(\theta)(\alpha_{0}^{-1} - \alpha^{-1}) 2r,
\end{align}
where in~\eqref{eq:slqcint3} we apply Lemma~\ref{lemma:SLQCreformulation} for the first term; for the second term
we use the fact that $\nabla R_{\alpha}$ is $J_{d}(\theta)$-Lipschitz in $\alpha^{-1}$ as given by~\eqref{eq:slqcjdlip} and the fact that $\theta_{0}-\theta \in \mathbb{B}_{d}(2r)$. 
Continuing from~\eqref{eq:slqcint4}, we have that 
\begin{align}
\nonumber \langle -\nabla R_{\alpha}(\theta),\theta_{0} - \theta \rangle &\geq \rho_{0} \|\nabla R_{\alpha}(\theta)\| - \rho_{0} \|\nabla R_{\alpha_{0}}(\theta) - \nabla R_{\alpha}(\theta)\| - J_{d}(\theta)(\alpha_{0}^{-1} - \alpha^{-1}) 2r  \\
\label{eq:thm2_2} &\geq \rho_{0} \|\nabla R_{\alpha}(\theta)\| - J_{d}(\theta)(\alpha_{0}^{-1} - \alpha^{-1}) (\rho_{0} + 2r),
\end{align}
where we first apply the reverse triangle inequality in~\eqref{eq:reversetriangleinequality} and then we use the fact that $\nabla R_{\alpha}(\theta)$ is $J_{d}(\theta)$-Lipschitz in $\alpha^{-1}$, i.e., the expression in~\eqref{eq:slqcjdlip}.
Rearranging the expression in~\eqref{eq:thm2_2}, we obtain
\begin{align}
\nonumber \rho_{0} \|\nabla R_{\alpha}(\theta)\| - J_{d}(\theta)(\alpha_{0}^{-1} - \alpha^{-1}) (\rho_{0} + 2r) &= \|\nabla R_{\alpha}(\theta)\| \left(\rho_{0} - \frac{J_{d}(\theta)(\alpha_{0}^{-1} - \alpha^{-1}) (\rho_{0} + 2r)}{\|\nabla R_{\alpha}(\theta)\|}\right) \\ 
\label{eq:thm2_4}&\geq  \|\nabla R_{\alpha}(\theta)\| \left(\rho_{0} - \frac{(\rho_{0}+2r)J_{d}(\theta)}{\|\nabla R_{\alpha_{0}}(\theta)\|(\alpha_{0}^{-1} - \alpha^{-1})^{-1} - J_{d}(\theta)}\right)
\end{align}
where we used the inequality in \eqref{eq:thm2_3}.
Thus, we finally obtain that 
\begin{align}
\langle-\nabla R_{\alpha}(\theta), \theta_{0} - \theta \rangle \geq \rho \|\nabla R_{\alpha}(\theta) \|,
\end{align}
where $\rho > 0$ is given by 
\begin{equation}
\rho = 
\rho_{0} \left(1 - \frac{(1+2r\rho_{0}^{-1})J_{d}(\theta)}{\|\nabla R_{\alpha_{0}}(\theta)\|(\alpha_{0}^{-1} - \alpha^{-1})^{-1} - J_{d}(\theta)}\right)
\end{equation}
~as desired.
Therefore by collecting all three parts, we have by Lemma~\ref{lemma:SLQCreformulation} that $R_{\alpha}$ is $(\epsilon,\kappa,\theta_{0})$-SLQC at $\theta$.
\end{proof}

\subsubsection{Bootstrapping SLQC} \label{appen:bootstrappingslqc}
Recall that the floor function $\lfloor \cdot \rfloor: \mathbb{R}^{+} \rightarrow \mathbb{N}$ can be written as $\lfloor x \rfloor = x - q$, for some $q \in [0,1)$.

\begin{lemma} \label{lemma:bootstrapping}
Fix $\theta \in \mathbb{B}_{d}(r)$. 
Suppose that  $\rho_{0} > 0$ and there exists $g_{\theta}>0$ such that $\lVert \nabla R_{\alpha'}(\theta) \rVert > g_{\theta}$ for all $\alpha'\in[\alpha_{0},\infty]$.
Given $N\in\mathbb{N}$, for each $n\in[N]$ we define 
\begin{equation}
\label{eq:BootstrappingDefaer}
    \alpha_{n} = \alpha_{n-1} + \frac{1}{N}, \quad \quad \epsilon_{n} = \epsilon_{n-1} + 2L_{d}(\theta) \frac{1}{\alpha_{n}\alpha_{n-1}}\frac{1}{N}, \quad \quad \rho_{n} = \rho_{n-1} - \frac{(\rho_{n-1}+2r)J_{d}(\theta)}{\alpha_{n}\alpha_{n-1} G_{n-1} - J_{d}(\theta)/N}\frac{1}{N},
\end{equation}
where $G_{n-1} \coloneqq \lVert \nabla R_{\alpha_{n-1}}(\theta) \rVert$. 
If $\displaystyle N > J_{d}(\theta)\left(\alpha_{0}^{2}g_{\theta}\right)^{-1}$, then we have that $\{\alpha_{n}\}_{n=0}^{N}$, $\{\epsilon_{n}\}_{n=0}^{N}$, and $\{\rho_{n}\}_{n=0}^{N}$ are well-defined. 
Further, we have that $\rho_{n} > 0$ for all $\displaystyle n \leq \left\lfloor \alpha_{0}^{2}g_{\theta}(1+2r\rho_{0}^{-1})^{-1}J_{d}(\theta)^{-1}N\right\rfloor$.
\end{lemma}
\begin{proof}
For ease of notation, let $J \coloneqq J_{d}(\theta)$, $L \coloneqq L_{d}(\theta)$, and $g \coloneqq g_{\theta}$.
Observe that $\{\alpha_{n}\}_{n=0}^{N}$ is well defined and so is $\{\epsilon_{n}\}_{n=0}^{N}$. It is straightforward to verify that if $\displaystyle N > J\left(\alpha_{0}^{2}g\right)^{-1}$, then $\alpha_{n-1}\alpha_{n} G_{n-1} - J/N > 0$ and thus $\{\rho_{n}\}_{n=0}^{N}$ is well defined. Now we show by induction that $\rho_{n} > 0$ for
\begin{equation}
\label{eq:BootstrappingConditionn}
    n < \left\lfloor\frac{\rho_{0}}{\rho_{0}+2r}\frac{\alpha_{0}^{2}g}{J}N\right\rfloor.
\end{equation}

By assumption, $\rho_{0} > 0$. For the inductive hypothesis, assume that $\rho_{0},\ldots,\rho_{n-1}$ are non-negative. Observe that, by definition,
\begin{equation}
    \rho_{k} - \rho_{k+1} = \frac{(\rho_{k}+2r)J}{\alpha_{k}\alpha_{k+1}G_{k} - J/N}\frac{1}{N}.
\end{equation}
The previous equation and a telescoping sum lead to
\begin{equation}
    \rho_{0} - \rho_{n} = \sum_{k=0}^{n-1} \frac{(\rho_{k}+2r)J}{\alpha_{k}\alpha_{k+1}G_{k} - J/N} \frac{1}{N}.
\end{equation}
Since $\rho_{k}>0$ for all $k\in[n-1]$, we have that $\rho_{0} > \rho_{1} > \cdots > \rho_{n}$ and, as a result,
\begin{equation}
    \rho_{0} - \rho_{n} < \frac{(\rho_{0} + 2r)J}{\alpha_{0}^{2}g-J/N}\frac{n}{N}.
\end{equation}
It can be shown that our choice of $n$ in \eqref{eq:BootstrappingConditionn} implies that
\begin{equation}
\label{eq:BootstrappingInqRadius}
    \rho_{n} > \rho_{0} - \frac{(\rho_{0} + 2r)J}{\alpha_{0}^{2}g-J/N}\frac{n}{N} > 0,
\end{equation}
which implies that $\rho_{n}>0$ as desired.
\end{proof}

\begin{reptheorem}{thm:SLQCbootstrapFTW}
Let $\alpha_{0} \in [1,\infty)$, $\epsilon_{0}, \kappa_{0}>0$, and $\theta_{0}, \theta \in \mathbb{B}_{d}(r)$.  
Suppose that $R_{{\alpha_{0}}}$ is $(\epsilon_{0},\kappa_{0},\theta_{0})$-SLQC at $\theta \in \mathbb{B}_{d}(r)$
and there exists $g_{\theta} > 0$ such that $\|\nabla R_{\alpha'}(\theta)\| > g_{\theta}$ for every $\alpha' \in [\alpha_{0},\infty]$.
Then for every $\lambda \in (0,1)$, $R_{\alpha_{\lambda}}$ is $(\epsilon_{\lambda},\kappa_{\lambda},\theta_{0})$-SLQC at $\theta$ where
\begin{align} 
\alpha_{\lambda} &\coloneqq \alpha_{0} + \lambda \frac{\alpha_{0}^2g_{\theta}}{J_{d}(\theta)\left(1+2r\frac{\kappa_{0}}{\epsilon_{0}}\right)}, 
\end{align}
\begin{align} 
\centering \epsilon_{\lambda} \coloneqq \epsilon_{0} + 2 \lambda L_{d}(\theta) \left(\frac{\alpha_{\lambda} - \alpha_{0}}{\alpha_{\lambda}\alpha_{0}} \right) \frac{\alpha_{0}^{2}g_{\theta}}{J_{d}(\theta)\left(1+r \frac{\kappa_{0}}{\epsilon_{0}}\right)}, \quad \text{and} \quad \frac{\epsilon_{\lambda}}{\kappa_{\lambda}} > \frac{\epsilon_{0}}{\kappa_{0}}(1-\lambda).
\end{align}
\end{reptheorem}

\begin{proof}
For ease of notation, let $J \coloneqq J_{d}(\theta)$, $L \coloneqq L_{d}(\theta)$, and $g \coloneqq g_{\theta}$. Let $\lambda\in(0,1)$ be given. For each $\displaystyle N > \frac{1+2r\rho_{0}^{-1}}{1-\lambda} \frac{2J}{\alpha_{0}^{2}g}$, we define
\begin{equation}
\label{eq:BootstrappingDefNlambda}
    N_{\lambda} = \left\lfloor\lambda \frac{\rho_{0}}{\rho_{0}+2r} \frac{\alpha_{0}^{2}g}{J}N\right\rfloor.
\end{equation}
The bootstrapping proof strategy is as follows: 1) For fixed $N \in \mathbb{N}$ large enough (as given above), we show by induction that $R_{\alpha_{n}}$ is $(\epsilon_{n},\kappa_{n},\theta_{0})$-SLQC at $\theta$ with $ \rho_{n} = \epsilon_{n}/\kappa_{n}$ for $n \leq N_{\lambda}$ using Lemma~\ref{lemma:bootstrapping} and Theorem~\ref{thm:SLQCresult1}; 2) We take the limit as $N$ approaches infinity in order to derive the largest range on $\alpha$ and the strongest SLQC parameters.

First, we show by induction that $R_{\alpha_{n}}$ is $(\epsilon_{n},\kappa_{n},\theta_{0})$-SLQC at $\theta$ with $ \rho_{n} = \epsilon_{n}/\kappa_{n}$ for $n \leq N_{\lambda}$. 
%
By assumption, $R_{\alpha_{0}}$ is $(\epsilon_{0},\kappa_{0},\theta_{0})$-SLQC at $\theta$.
%
%
For the inductive hypothesis, assume that $R_{\alpha_{k}}$ is $(\alpha_{k},\epsilon_{k},\kappa_{k})$-SLQC at $\theta$ for all $k\in[n-1]$. 
In order to apply Lemma~\ref{lemma:bootstrapping} to show that $\rho_{0}>\rho_{1}>\ldots>\rho_{n}> \cdots > \rho_{N_{\lambda}} > C_{\lambda} >0$ for all $n \leq N_{\lambda}$ and for some $C_{\lambda} > 0$,
we first show that the assumptions of Lemma~\ref{lemma:bootstrapping} are satisfied.
Observe that, by our assumption on $N \in \mathbb{N}$, we have that 
\begin{align} \label{eq:bootstrappinglowerboundN}
N > \frac{1+2r\rho_{0}^{-1}}{1-\lambda} \frac{2J}{\alpha_{0}^{2}g} > \frac{1+r\rho_{0}^{-1}}{1-\lambda} \frac{J}{\alpha_{0}^{2}g} > \frac{J}{\alpha_{0}^{2}g},
\end{align}
which is the first requirement of Lemma~\ref{lemma:bootstrapping}. 
Next, we want to show that 
\begin{equation} \label{eq:nlambdabound}
    n \leq N_{\lambda} < \left\lfloor\frac{\rho_{0}}{\rho_{0}+2r}\frac{\alpha_{0}^{2}g}{J}N\right\rfloor,
\end{equation}
which is the last requirement of Lemma~\ref{lemma:bootstrapping}.
This is achieved by observing that 
\begin{align}
N_{\lambda} = \left\lfloor\lambda \frac{\rho_{0}}{\rho_{0}+2r} \frac{\alpha_{0}^{2}g}{J}N\right\rfloor = \lambda \frac{\rho_{0}}{\rho_{0}+2r} \frac{\alpha_{0}^{2}g}{J}N - q,
\end{align}
for some $q \in [0,1)$ and that
\begin{align}
\left\lfloor \frac{\rho_{0}}{\rho_{0}+2r} \frac{\alpha_{0}^{2}g}{J}N\right\rfloor =  \frac{\rho_{0}}{\rho_{0}+2r} \frac{\alpha_{0}^{2}g}{J}N - w,
\end{align}
also for some $w \in [0,1)$. Thus, note that~\eqref{eq:nlambdabound}
is equivalent to 
\begin{align}
(q-w) \frac{1+r\rho_{0}^{-1}}{1 - \lambda} \frac{J}{\alpha_{0}^{2}g} < N,
\end{align}
which holds by the fact that 
$\displaystyle N > \frac{1+r \rho_{0}^{-1}}{1-\lambda} \frac{J}{\alpha_{0}^{2}g}$ in~\eqref{eq:bootstrappinglowerboundN} and $q-w \leq 1$.
Thus by Lemma~\ref{lemma:bootstrapping}, we have that 
%
\begin{align}
\rho_{n} &> \rho_{0}- \frac{(\rho_{0} + 2r)J}{\alpha_{0}^{2}g-J/N}\frac{n}{N} > 0,
\end{align}
for all $n \leq N_{\lambda}$.
In particular for $n = N_{\lambda}$, we have that 
\begin{align}
    \rho_{N_{\lambda}} &> \rho_{0}- \frac{(\rho_{0} + 2r)J}{\alpha_{0}^{2}g-J/N}\frac{N_{\lambda}}{N}\\
    &> \rho_{0}\left(1 - \lambda - \frac{\lambda J}{\alpha_{0}^{2}g-J/N} \frac{1}{N}\right)\\
    &> \frac{\rho_{0}(1-\lambda)}{2},
\end{align}
where the second inequality follows by plugging in $N_{\lambda}$ and adding and subtracting $\lambda J/N$ in the fraction and the last inequality follows from $\displaystyle N > \frac{1+2r\rho_{0}^{-1}}{1-\lambda} \frac{2J}{\alpha_{0}^{2}g} > \frac{1+\lambda}{1-\lambda} \frac{J}{\alpha_{0}^{2}g}$ since $2r\rho_{0}^{-1} \geq \lambda$ for all $\lambda \in (0,1)$. 
Therefore, we have that $\displaystyle C_{\lambda} = \frac{\rho_{0}(1-\lambda)}{2}$; in other words,
\begin{align}
 \label{eq:BootstrappingLowerBoundrn} \rho_{0}>\rho_{1}>\ldots>\rho_{n-1} > \rho_{n} > \cdots > \rho_{N_{\lambda}} > \frac{\rho_{0}(1-\lambda)}{2} > 0.    
\end{align}
Also, observe that 
\begin{equation} \label{eq:upperboundon1/N}
    \alpha_{n} - \alpha_{n-1} = \frac{1}{N} 
    < \frac{\alpha_{0}^{2} g}{2J(1+2r\rho_{0}^{-1}(1-\lambda)^{-1})},
\end{equation}
where the inequality follows from the fact that
\begin{align}
N > \frac{1+2r\rho_{0}^{-1}}{1-\lambda} \frac{2J}{\alpha_{0}^{2}g} > \left(1+\frac{2r\rho_{0}^{-1}}{1-\lambda}\right) \frac{2J}{\alpha_{0}^{2}g}.
\end{align}
%
In particular, \eqref{eq:BootstrappingLowerBoundrn} and \eqref{eq:upperboundon1/N} leads to
\begin{equation}
    \alpha_{n} - \alpha_{n-1} < \frac{\alpha_{0}^{2} g}{2J(1+2r\rho_{0}^{-1}(1-\lambda)^{-1})} < \frac{\alpha_{n-1}^{2} G_{n-1}}{2J(1+r\rho_{n-1}^{-1})},
\end{equation}
where we use the fact that $\alpha_{n} \geq \alpha_{0}$ and $G_{n-1} \geq g$.
As a result, we can apply Theorem~\ref{thm:SLQCresult1} to conclude that $R_{\alpha_{n}}$ is $(\epsilon_{n},\rho_{n},\theta_{0})$-SLQC at $\theta$ with $\alpha_{n}$, $\epsilon_{n}$ and $\rho_{n}$ given as in \eqref{eq:BootstrappingDefaer}. In particular by unfolding the recursion, we have that $R_{\alpha_{N_{\lambda}}}$ is $(\epsilon_{N_{\lambda}},\rho_{N_{\lambda}},\theta_{0})$-SLQC at $\theta$ with
\begin{align}
   \label{eq:alphanlambda} \alpha_{N_{\lambda}} &= \alpha_{0} + \lambda (1+2r\rho_{0}^{-1})^{-1} \frac{\alpha_{0}^2g}{J} - \frac{q}{N},\\
    \label{eq:epsilonnlambda} \epsilon_{N_{\lambda}} &= \epsilon_{0} + 2L\sum_{n=0}^{N_{\lambda}-1} \frac{1}{\alpha_{n}(\alpha_{n} + 1/N)}\frac{1}{N}, \\
     \label{eq:rhonlambda}  \rho_{N_{\lambda}} &= \rho_{0} \prod_{n=0}^{N_{\lambda}-1} \left(1 - \frac{(1 + 2r\rho_{n}^{-1}) J/N}{\alpha_{n+1} \alpha_{n} \|\nabla R_{\alpha_{n}}(\theta)\| - J/N}\right),
\end{align}
%
for some $q\in[0,1)$. 

Finally, we take the limit as $N$ approaches infinity in order to derive the largest range on $\alpha$ and the strongest SLQC parameters.
Recall that $N_{\lambda} = \left\lfloor\lambda \frac{\rho_{0}}{\rho_{0}+2r} \frac{\alpha_{0}^{2}g}{J}N\right\rfloor = \lambda \frac{\rho_{0}}{\rho_{0}+2r} \frac{\alpha_{0}^{2}g}{J}N - q$,
for some $q \in [0,1)$.
Thus, we have the following relationship 
\begin{align}
\frac{1}{N} = \frac{\lambda \rho_{0} \alpha_{0}^{2}g}{(N_{\lambda}+q)(\rho_{0}+r)J}.
\end{align}
Observe that taking the limit as $N$ approaches infinity is equivalent to taking the limit as $N_{\lambda}$ approaches infinity. 

Examining~\eqref{eq:alphanlambda} as $N_{\lambda}$ approaches infinity, we have that 
\begin{align}
\alpha_{\lambda} := \lim\limits_{N_{\lambda} \rightarrow \infty} \alpha_{N_{\lambda}} = \alpha_{0} + \lambda (1+2r\rho_{0}^{-1})^{-1} \frac{\alpha_{0}^2g}{J}.
\end{align}

Next considering~\eqref{eq:epsilonnlambda}, we rewrite to obtain
\begin{align}
\epsilon_{N_{\lambda}} &= \epsilon_{0} + 2L\sum_{n=0}^{N_{\lambda}-1} \frac{1}{\alpha_{n}(\alpha_{n} + 1/N)}\frac{1}{N} \\
  \label{eq:thm2intstep12} &= \epsilon_{0} + \frac{2L}{N}\sum_{n=0}^{N_{\lambda}-1} \left(\frac{1}{\alpha_{n}^{2}} + \frac{1}{N} \frac{1}{\alpha_{n}^{3} - \alpha_{n}^{2}/N} \right),
\end{align}
where we used a partial fraction decomposition. 
Let $\mu_{N_{\lambda}}$ be the discrete measure given by
\begin{equation*}
    \mu_{N_{\lambda}} = \frac{1}{N_{\lambda}} \sum_{n=0}^{N_{\lambda}-1} \delta_{\alpha_{n}},
\end{equation*}
where $\delta_{\alpha_{n}}$ is the point mass at $\alpha_{n}$.
In particular for large $N$, we can write \eqref{eq:thm2intstep12} as 
\begin{equation}
\label{eq:thm2muNstep2}
    \epsilon_{N_{\lambda}} = \epsilon_{0} + \frac{2L \lambda \alpha_{0}^{2}g}{(1+r\rho_{0}^{-1})J} \int \frac{1}{x^{2}}d\mu_{N_{\lambda}}(x) + O\left(\frac{1}{N_{\lambda}}\right).
\end{equation}
Let $\mu_{\lambda}$ denote the uniform measure over $(\alpha_0,\alpha_{\lambda}]$, i.e., the Lebesgue measure on the interval $(\alpha_0,\alpha_{\lambda}]$. Note that $\mu_{N_{\lambda}}$ converges in distribution to $\mu_{\lambda}$ as $N_{\lambda}$ goes to infinity. By taking limits, \eqref{eq:thm2muNstep2} becomes
\begin{align}
\epsilon_{\lambda} = \lim\limits_{N_{\lambda} \rightarrow \infty} \epsilon_{N_{\lambda}} = \epsilon_{0} + \frac{2L \lambda \alpha_{0}^{2}g}{(1+r\rho_{0}^{-1})J} \int\limits_{\alpha_{0}}^{\alpha_{\lambda}} \frac{1}{x^{2}} dx &= \epsilon_{0} + \frac{2L \lambda \alpha_{0}g}{(1+r\rho_{0}^{-1})J} \left(1 - \frac{\alpha_{0}}{\alpha_{\lambda}} \right).
\end{align}

Finally, we consider~\eqref{eq:rhonlambda}.
Observe that from \eqref{eq:BootstrappingInqRadius} we have that
\begin{align}
\rho_{N_{\lambda}} &> \rho_{0} - \frac{(\rho_{0} + 2r)J}{\alpha_{0}^{2}g-J/N}\frac{N_{\lambda}}{N} \\
&= \rho_{0} - \frac{(\rho_{0} + 2r)J}{\alpha_{0}^{2}g-J/N}\frac{\lambda \frac{N\alpha_{0}^{2}g\rho_{0}}{J(\rho_{0}+2r)} - q}{N} \\
&= \rho_{0} - \left[\frac{N \lambda \rho_{0} \alpha_{0}^{2}g}{N \alpha_{0}^{2}g - J} - \frac{q}{N} \left(\frac{(\rho_{0}+2r)J}{\alpha_{0}^{2}g - J/N} \right)\right],
\end{align}
for $q \in [0,1)$, where we plugged in the definition of $N_{\lambda}$ and simplified. Thus, taking the limit as $N_{\lambda}$ approaches infinity we have that 
\begin{align}
\rho_{\lambda} = \lim\limits_{N_{\lambda} \rightarrow \infty} \rho_{N_{\lambda}} &> \lim\limits_{N_{\lambda} \rightarrow \infty} \left( \rho_{0} - \left[\frac{N \lambda \rho_{0} \alpha_{0}^{2}g}{N \alpha_{0}^{2}g - J} - \frac{q}{N} \left(\frac{(\rho_{0}+2r)J}{\alpha_{0}^{2}g - J/N}\right)\right] \right) = \rho_{0} (1-\lambda).
\end{align}

Therefore, we conclude that $R_{\alpha_{\lambda}}$ is $(\epsilon_{\lambda},\kappa_{\lambda},\theta_{0})$-SLQC at $\theta$ with
\begin{align}
    \alpha_{\lambda} &\coloneqq \alpha_{0} + \lambda (1+2r\rho_{0}^{-1})^{-1} \frac{\alpha_{0}^2g}{J},\\
    \epsilon_{\lambda} &\coloneqq \epsilon_{0} + \frac{2L \lambda \alpha_{0}g}{(1+r\rho_{0}^{-1})J} \left(1 - \frac{\alpha_{0}}{\alpha_{\lambda}} \right) \\
    \label{eq:rholambda} \rho_{\lambda} &> \rho_{0}(1-\lambda).
\end{align}
A change of variables leads to the desired result.
\end{proof}

\subsection{Rademacher Complexity Generalization and Asymptotic Optimality} \label{appen:generalization}

\begin{lemma} \label{lemma:marginlip}
If $\alpha \in (0,\infty]$, then $\tilde{l}^{\alpha}(z)$ is $C_{r_{0}}(\alpha)$-Lipschitz in $z \in [-r_{0},r_{0}]$ for every $r_{0} >0$, where 
\begin{align} \label{eq:calpha}
C_{r_{0}}(\alpha) := \begin{cases}
			\sigma(r_{0})\sigma(-r_{0})^{1-1/\alpha}, & \alpha \in (0,1] \\
            \left(\frac{\alpha-1}{2 \alpha - 1} \right)^{1-1/\alpha} \left(\frac{\alpha}{2\alpha -1} \right), & \alpha \in (1,\infty] \quad \text{and} \quad r_{0} \geq \log{\left(1 - 1/\alpha \right)} \\
            \sigma(r_{0})\sigma(-r_{0})^{1-1/\alpha}, & \alpha \in (1,\infty] \quad \text{and} \quad r_{0} < \log{\left(1 - 1/\alpha \right)}.
		 \end{cases}
\end{align}
\end{lemma}
\begin{proof}
The proof is analogous to the proof in Proposition~\ref{prop:SLQCpriortoevolution}.
In order to show that $\tilde{l}^{\alpha}(z)$ is $C_{r_{0}}(\alpha)$-Lipschitz, we take the derivative of $\tilde{l}^{\alpha}(z)$ and seek to maximize it over $z \in [-r_{0},r_{0}]$. Specifically, we have that for $\alpha \in (0,\infty]$,
\begin{align}
\dfrac{d}{dz} \tilde{l}^{\alpha}(z) &= \dfrac{d}{dz} \frac{\alpha}{\alpha - 1}\left(1 - \sigma(z)^{1 - 1/\alpha}\right) \\ 
&= \sigma(z)^{2-1/\alpha} - \sigma(z)^{1-1/\alpha} \\
&= (\sigma(z) - 1) \sigma(z)^{1-1/\alpha}  \\
&\leq |(\sigma(z) - 1) \sigma(z)^{1-1/\alpha}| \\
&= \sigma(-z)\sigma(z)^{1-1/\alpha},
\end{align}
where we used the fact that $\sigma(z) = 1 - \sigma(-z)$. 
If $\alpha \leq 1$, it can be shown that 
\begin{align}
\max\limits_{z \in [-r_{0},r_{0}]} \sigma(-z)\sigma(z)^{1-1/\alpha} = \sigma(r_{0})\sigma(-r_{0})^{1-1/\alpha}.
\end{align}
Similarly if $\alpha > 1$ and if $r_{0} \geq \log{\left(1 - 1/\alpha \right)}$, it can be shown that 
\begin{align}
\max\limits_{z \in [-r_{0},r_{0}]} \sigma(-z)\sigma(z)^{1-1/\alpha} = \left(\frac{\alpha-1}{2 \alpha - 1} \right)^{1-1/\alpha} \left(\frac{\alpha}{2\alpha -1} \right),
\end{align}
where $z^{*} = \log{\left(1 - 1/\alpha \right)}$.
Otherwise for $\alpha > 1$, if $r_{0} < \log{\left(1 - 1/\alpha \right)}$, we obtain using local monotonicity,
\begin{align}
\max\limits_{z \in [-r_{0},r_{0}]} \sigma(-z)\sigma(z)^{1-1/\alpha} = \sigma(r_{0})\sigma(-r_{0})^{1-1/\alpha},
\end{align}
analogous to the case where $\alpha < 1$.
Thus, combining the two regimes of $\alpha$, we have the result. 
\end{proof}

\begin{reptheorem}{thm:radegeneralization}
If $\alpha\in(0,\infty]$, then, with probability at least $1-\delta$, for all $\theta \in \mathbb{B}_{d}(r)$,
\begin{equation} 
\left|R_{\alpha}(\theta) - \hat{R}_{\alpha}(\theta)\right| \leq C_{r\sqrt{d}}\left(\alpha\right) \dfrac{2r\sqrt{d}}{\sqrt{n}} + 4D_{r\sqrt{d}}\left(\alpha\right)\sqrt{\dfrac{2\log{(4/\delta)}}{n}}, 
\end{equation}
where $C_{r\sqrt{d}}\left(\alpha \right)$ is given in~\eqref{eq:calpha} and $\displaystyle D_{r\sqrt{d}}\left(\alpha\right) := \frac{\alpha}{\alpha-1}\left(1-\sigma(-r\sqrt{d})^{1-1/\alpha}\right)$.
\end{reptheorem} 
\begin{proof}
By Proposition \ref{Prop:relationhardtosoft}, which gives a relation between $\alpha$-loss and its margin-based form, we have
\begin{align} \label{eq:rade}
\mathcal{R}(l^{\alpha}\circ\mathcal{G}\circ S_{n}) = \mathbb{E}\left(\sup\limits_{g_{\theta} \in \mathcal{G}} \dfrac{1}{n} \sum\limits_{i=1}^{n} \sigma_{i} l^{\alpha}(y_{i},g_{\theta}(x_{i}))\right) 
= \mathbb{E} \left(\sup\limits_{\theta \in \mathbb{B}_{d}(r)} \dfrac{1}{n} \sum\limits_{i=1}^{n} \sigma_{i} \tilde{l}^{\alpha}(y_{i}\langle \theta,x_{i} \rangle)\right).
\end{align}
The right-hand-side of \eqref{eq:rade} can be rewritten as 
\begin{align}
\mathbb{E} \left(\sup\limits_{\theta \in \mathbb{B}_{d}(r)} \dfrac{1}{n} \sum\limits_{i=1}^{n} \sigma_{i} \tilde{l}^{\alpha}(y_{i}\langle \theta,x_{i} \rangle)\right) = \mathcal{R}\left(\{\tilde{l}^{\alpha}(y_{1}\langle \theta, x_{1} \rangle), \ldots, \tilde{l}^{\alpha}( y_{n}\langle \theta, x_{n} \rangle):\theta \in \mathbb{B}_{d}(r)\}\right).  
\end{align}
Observe that, for each $i \in [n]$, $y_{i}\langle \theta, x_{i} \rangle \leq r \sqrt{d}$ by the Cauchy-Schwarz inequality since $\theta \in \mathbb{B}_{d}(r)$ and for each $i \in [n]$, $x_{i} \in [0,1]^{d}$.
Further, by Lemma \ref{lemma:marginlip}, we know that $\tilde{l}^{\alpha}(z)$ is $C_{r_{0}}\left(\alpha\right)$-Lipschitz in $z \in [-r_{0},r_{0}]$. Thus setting $r_{0} = r\sqrt{d}$, we may apply Lemma \ref{lemma:radecontraction} (Contraction Lemma) to obtain
\begin{align}
\mathbb{E} \left(\sup\limits_{\theta \in \mathbb{B}_{d}(r)} \dfrac{1}{n} \sum\limits_{i=1}^{n} \sigma_{i} \tilde{l}^{\alpha}(y_{i}\langle \theta,x_{i} \rangle)\right) &= \mathcal{R}\left(\{\tilde{l}^{\alpha}(y_{1}\langle \theta, x_{1} \rangle), \ldots, \tilde{l}^{\alpha}( y_{n}\langle \theta, x_{n} \rangle):\theta \in \mathbb{B}_{d}(r)\}\right) \\
&\leq C_{r\sqrt{d}}\left(\alpha\right) \mathcal{R}\left(\{(y_{1}\langle \theta, x_{1} \rangle, \ldots, y_{n}\langle \theta, x_{n} \rangle):\theta \in \mathbb{B}_{d}(r)\}\right).
\end{align}
We absorb $y_{i}$ into its corresponding $x_{i}$ and apply Lemma \ref{lemma:raderadiusbound} to obtain
\begin{equation}
C_{r\sqrt{d}}\left(\alpha\right) \mathcal{R}(\{(y_{1}\langle \theta, x_{1} \rangle, \ldots, y_{n}\langle \theta, x_{n} \rangle):\theta \in \mathbb{B}_{d}(r)\}) \leq C_{r\sqrt{d}}\left(\alpha\right) \dfrac{r \sqrt{d}}{\sqrt{n}},
\end{equation}
which follows since we assume that $x_{i} \in [0,1]^{d}$ for each $i \in [n]$. 
In order to apply Proposition \ref{prop:radegeneralization}, it can readily be shown that for $\alpha \in (0, \infty]$ 
\begin{equation}
\max\limits_{z \in \left[-r\sqrt{d},r\sqrt{d}\right]} \tilde{l}^{\alpha}(z) \leq D_{r\sqrt{d}}\left(\alpha\right),
\end{equation}
where $D_{r\sqrt{d}}\left(\alpha\right) = \frac{\alpha}{\alpha-1}\left(1-\sigma(-r\sqrt{d})^{1-1/\alpha}\right)$.
Thus, we apply Proposition \ref{prop:radegeneralization} to achieve the desired result. 
\end{proof}

The following result attempts to quantify the uniform discrepancy between the empirical $\alpha$-risk and the probability of error (true $\infty$-risk); the technique is a combination of Theorem~\ref{thm:radegeneralization} and Lemma~\ref{lemma:inversealphalip}.
The result is most useful in the regime where $r\sqrt{d} \le \alpha/\sqrt{n}$; this prohibits the second term in the right-hand-side of~\eqref{eq:coruniformdiscrepancy} from dominating the first, which is the most meaningful form of the bound.
\begin{corollary} \label{cor:saturationgeneralization}
If $\alpha\in[1,\infty]$, then, with probability at least $1-\delta$, for all $\theta \in \mathbb{B}_{d}(r)$,
\begin{equation} \label{eq:coruniformdiscrepancy}
\left|R_{\infty}(\theta) - \hat{R}_{\alpha}(\theta)\right| \leq \sigma\left(r\sqrt{d}\right)\left(\dfrac{2r\sqrt{d}}{\sqrt{n}} + 4\sqrt{\dfrac{2\log{(4/\delta)}}{n}}\right) + \dfrac{\left(\log{\sigma(-r\sqrt{d})}\right)^{2}}{2\alpha}.
\end{equation}
\end{corollary}
\begin{proof}
Consider the expression, $R_{\infty}(\theta) - \hat{R}_{\alpha}(\theta)$.
Since $\hat{R}_{\infty}(\theta) \leq \hat{R}_{\alpha}(\theta)$ for all $\theta \in \mathbb{B}_{d}(r)$, the following holds
\begin{equation}
R_{\infty}(\theta) - \hat{R}_{\alpha}(\theta) \leq R_{\infty}(\theta) - \hat{R}_{\infty}(\theta) \leq \sigma\left(r\sqrt{d}\right)\left(\dfrac{2r\sqrt{d}}{\sqrt{n}} + 4\sqrt{\dfrac{2\log{(4/\delta)}}{n}}\right),
%
\end{equation}
where we applied Theorem \ref{thm:radegeneralization} for $\alpha = \infty$. 
Now, consider the reverse direction, $\hat{R}_{\alpha}(\theta)-R_{\infty}(\theta)$. 
For any $\theta \in \mathbb{B}_{d}(r)$, we add and subtract $\hat{R}_{\infty}(\theta)$ such that 
\begin{align}
\nonumber \hat{R}_{\alpha}(\theta)-R_{\infty}(\theta)&=\hat{R}_{\infty}(\theta) - R_{\infty}(\theta) + \hat{R}_{\alpha}(\theta) - \hat{R}_{\infty}(\theta) \\ &\leq \sigma\left(r\sqrt{d}\right)\left(\dfrac{2r\sqrt{d}}{\sqrt{n}} + 4\sqrt{\dfrac{2\log{(4/\delta)}}{n}}\right) + \dfrac{\left(\log{\sigma(-r\sqrt{d})}\right)^{2}}{2\alpha},
\end{align}
where we apply Theorem~\ref{thm:radegeneralization} for the first term and Lemma~\ref{lemma:inversealphalip} for the second term\footnote{We apply Lemma~\ref{lemma:inversealphalip} to the empirical distribution instead of the true distribution, leading to a bound for the empirical $\alpha$-risk.} on the maximum value of $\theta$, i.e, $\|\theta\|_{2} = r$.
Thus, combining the two cases we have
\begin{equation}
\left|R_{\infty}(\theta) - \hat{R}_{\alpha}(\theta)\right| \leq \sigma\left(r\sqrt{d}\right)\left(\dfrac{2r\sqrt{d}}{\sqrt{n}} + 4\sqrt{\dfrac{2\log{(4/\delta)}}{n}}\right) + \dfrac{\left(\log{\sigma(-r\sqrt{d})}\right)^{2}}{2\alpha},
\end{equation}
which is the desired statement for the corollary.
\end{proof}
\begin{reptheorem}{thm:asymptoticoptimality}
Assume that the minimum $\alpha$-risk is attained by the logistic model, i.e.,~\eqref{eq:alphalinearseparable} holds.
Let $S_{n}$ be a training dataset with $n \in \mathbb{N}$ samples as before.
If for each $n\in\mathbb{N}$, $\hat{\theta}_n^{\alpha}$ is a global minimizer of the associated empirical $\alpha$-risk $\theta \mapsto \hat{R}_{\alpha}(\theta)$, then the sequence $(\hat{\theta}_n^{\alpha})_{n=1}^\infty$ is asymptotically optimal for the $0$-$1$ risk, i.e., almost surely,
\begin{equation}
\lim_{n\to\infty} R(f_{\hat{\theta}_n^{\alpha}}) = R^*,
\end{equation}
where $f_{\hat{\theta}_n^{\alpha}}(x) = \langle \hat{\theta}_n^{\alpha}, x \rangle$ and $R^{*} := \min\limits_{f: \mathcal{X} \rightarrow \mathbb{R}} \mathbb{P}[Y \neq \sign(f(X))]$.
\end{reptheorem}
\begin{proof}
We begin by recalling the following proposition which establishes an important consequence of classification-calibration. In words, the following result assures that minimizing a classification-calibrated loss to optimality also minimizes the $0$-$1$ loss to optimality.
\begin{prop}[{\cite[Theorem~3]{bartlett2006convexity}}] \label{prop1}
Assume that $\phi$ is a classification-calibrated margin-based loss function. Then, for every sequence of measurable functions $(f_{i})_{i=1}^\infty$ and every probability distribution on $\mathcal{X} \times \mathcal{Y}$,
\begin{equation}
\lim_{i\to\infty} R_{\phi}(f_{i}) = R^*_{\phi} \text{ implies that } \lim_{i\to\infty} R(f_{i}) = R^*,
\end{equation}
where $R^*_{\phi} := \min_f R_{\phi}(f)$ and $R^* := \min_f R(f)$.
\end{prop}
By the assumption that the minimum $\alpha$-risk is obtained by the logistic model, we have that 
\begin{align} \label{eq:logisticisenough}
\min\limits_{\theta \in \mathbb{B}_{d}(r)} R_{\alpha}(\theta) = \min\limits_{f:\mathcal{X} \rightarrow \mathbb{R}} R_{\alpha}(f),
\end{align}
where $R_{\alpha}(\theta)$ is given in~\eqref{eq:alphariskdefLR} and $R_{\alpha}(f) = \mathbb{E}[\tilde{l}^{\alpha}(Yf(X))]$ for all measurable $f$.
Thus, the proof strategy is to show that 
\begin{equation}
\label{eq:CorollaryNotMargin}
   \lim_{n\to\infty} R_{\alpha}(\hat{\theta}^{\alpha}_{n}) = \min_{\theta\in \mathbb{B}_{d}(r)} R_{\alpha}(\theta),
\end{equation}
and then apply Proposition~\ref{prop1} to obtain the result. 

Let $\theta_*^{\alpha}$ be a minimizer of the $\alpha$-risk, i.e.,
\begin{equation}
R_{\alpha}(\theta_*^{\alpha}) = \min_{\theta\in\mathbb{B}_{d}(r)} R_{\alpha}(\theta).
\end{equation}
Observe that
\begin{equation}
\label{eq:SplitInII}
    0 \leq R_{\alpha}(\hat{\theta}_n^{\alpha}) - R_{\alpha}(\theta_*^{\alpha}) = {\rm I}_n + {\rm II}_n,
\end{equation}
where ${\rm I}_n := R_{\alpha}(\hat{\theta}_n^{\alpha}) - \hat{R}_{\alpha}(\hat{\theta}_n^{\alpha})$ and ${\rm II}_n := \hat{R}_{\alpha}(\hat{\theta}_n^{\alpha}) - R_{\alpha}(\theta_*^{\alpha})$. After some straightforward manipulations of Theorem \ref{thm:radegeneralization}, \eqref{eq:radegeneral} implies that, for every $\epsilon>0$,
\begin{equation}
\label{eq:GeneralizationCriticalPoints}
   \displaystyle \mathbb{P}\left(|R_{\alpha}(\hat{\theta}_n^{\alpha}) - \hat{R}_{\alpha}(\hat{\theta}_n^{\alpha})| > \epsilon\right) \leq 4 e^{-n\left(\frac{\epsilon - C_{r\sqrt{d}}(\alpha) 2r\sqrt{d}/{n}}{4\sqrt{2} D_{r\sqrt{d}}(\alpha)} \right)^{2}},
\end{equation}
whenever $n$ is large enough. A routine application of the Borel-Cantelli lemma shows that, almost surely,
\begin{equation}
\label{eq:asLimit1}
    \lim_{n\to\infty} {\rm I}_n = \lim_{n\to\infty} R_{\alpha}(\hat{\theta}_n^{\alpha}) - \hat{R}_{\alpha}(\hat{\theta}_n^{\alpha}) = 0.
\end{equation}
Since $\hat{\theta}_n^{\alpha}$ is a minimizer of the empirical risk $\hat{R}_{\alpha}$,
\begin{equation}
    {\rm II}_n = \hat{R}_{\alpha}(\hat{\theta}_n^{\alpha}) - R_{\alpha}(\theta_*^{\alpha}) \leq \hat{R}_{\alpha}(\theta_*^{\alpha}) - R_{\alpha}(\theta_*^{\alpha}).
\end{equation}
Again by Theorem \ref{thm:radegeneralization}, for every $\epsilon>0$,
\begin{equation}
\label{eq:HoeffdingEmpiricalExpected}
    \mathbb{P}\left(|\hat{R}_{\alpha}(\theta_*^{\alpha}) - R_{\alpha}(\theta_*^{\alpha})| > \epsilon \right) \leq 4 e^{-n\left(\frac{\epsilon - C_{r\sqrt{d}}(\alpha) 2r\sqrt{d}/{n}}{4\sqrt{2} D_{r\sqrt{d}}(\alpha)} \right)^{2}},
\end{equation}
whenever $n$ is large enough.
Hence, the Borel-Cantelli lemma implies that, almost surely,
\begin{equation}
    \lim_{n\to\infty} |\hat{R}_{\alpha}(\theta_*^{\alpha}) - R_{\alpha}(\theta_*^{\alpha})| = 0.
\end{equation}
In particular, we have that, almost surely,
\begin{equation}
\label{eq:asLimit2}
    \limsup_{n\to\infty} {\rm II}_n \leq 0.
\end{equation}
By plugging \eqref{eq:asLimit1} and \eqref{eq:asLimit2} in \eqref{eq:SplitInII}, we obtain that, almost surely,
\begin{equation}
    0 \leq \limsup_{n\to\infty} \left[R_{\alpha}(\hat{\theta}_n^{\alpha}) - R_{\alpha}(\theta_*^{\alpha})\right] \leq 0,
\end{equation}
from which \eqref{eq:CorollaryNotMargin} follows.

For each $n\in\mathbb{N}$, let $f_{\hat{\theta}_n^{\alpha}}:\mathcal{X}\to\overline{\mathbb{R}}$ be given by $f_{\hat{\theta}_n^{\alpha}}(x) = \langle \hat{\theta}_{n}^{\alpha}, x \rangle$. Since we have
\begin{align}
f_{\hat{\theta}_n^{\alpha}}(x) = \sigma^{-1}(\sigma(\hat{\theta}_{n}^{\alpha}\cdot x)) = \sigma^{-1}(g_{\hat{\theta}_{n}^{\alpha}}(x)),
\end{align}
Proposition~\ref{Prop:relationhardtosoft},~\eqref{eq:logisticisenough}, and~\eqref{eq:CorollaryNotMargin} imply that
\begin{equation}
\label{eq:CorollaryMargin}
    \lim_{n\to\infty} R_{\alpha}(f_{\hat{\theta}_{n}^{\alpha}}) = \min_{\theta\in\mathbb{B}_{d}(r)} R_{\alpha}(f_{\theta}) = \min\limits_{f:\mathcal{X} \rightarrow \mathbb{R}} R_{\alpha}(f) =: R_{\alpha}^*.
\end{equation}
Since $\tilde{l}^{\alpha}$ is classification-calibrated as established in Theorem~\ref{thm:alphalossclassificationcalibration}, Proposition~\ref{prop1} and \eqref{eq:CorollaryMargin} imply that
\begin{equation}
\lim_{n\to\infty} R(f_{\hat{\theta}_n^{\alpha}}) = \min\limits_{f: \mathcal{X} \rightarrow \mathbb{R}} \mathbb{P}[Y \neq \sign(f(X))] =: R^*,
\end{equation}
as required.
\end{proof}


\subsection{Further Experimental Results and Details}

\subsubsection{Brief Review of the F1 Score} \label{sec:f1review}
In binary classification, the $F_{1}$ score is a measure of a model's accuracy and is particularly useful when there is an imbalanced class, since it is known to give more precise performance information for an imbalanced class than simply using accuracy itself~\cite{f1score}.
In words, the $F_{1}$ score is the harmonic mean of the precision and recall, where precision is defined as the number of true positives divided by the number of true positives plus false positives (all examples the model declares as positive) and where recall is defined as the number of true positives divided by the number of true positives plus false negatives (all the examples that the model should have declared as positive).
Formally, the definition of the $\text{F}_{1}$ score is 
\begin{align} \label{eq:f1score}
\text{F}_{1} = \frac{2}{\text{recall}^{-1} + \text{precision}^{-1}} = \frac{\text{tp}}{\text{tp} + .5(\text{fp}+\text{fn})},
\end{align}
where $\text{tp}$, $\text{fp}$, $\text{fn}$ denote true positives, false positive, and false negatives, respectively. 
In practice, $\text{tp}$, $\text{fp}$, and $\text{fn}$ are drawn from the confusion matrix of the model on test data.
Note that the use of the term ``positive'', denoting the class name is arbitrarily chosen; in practice, one lets ``positive'' class denote the imbalanced class.

\subsubsection{Experiments for Section~\ref{sec:syntheticexp}} \label{sec:extrasyntheticexp}
In this section, we provide additional synthetic experiments, which follow the same experiment protocol as Figure~\ref{fig:planeplots}. They highlight some of the main themes of the paper, namely, $\alpha^{*} < 1$ in imbalanced experiments, $\alpha^{*} > 1$ in noisy experiments,
trade-offs between computational feasibility and accuracy (for both regimes of $\alpha$), and the saturation effect.

\begin{figure}[h] 
    \centering
    \centerline{\includegraphics[scale=.15]{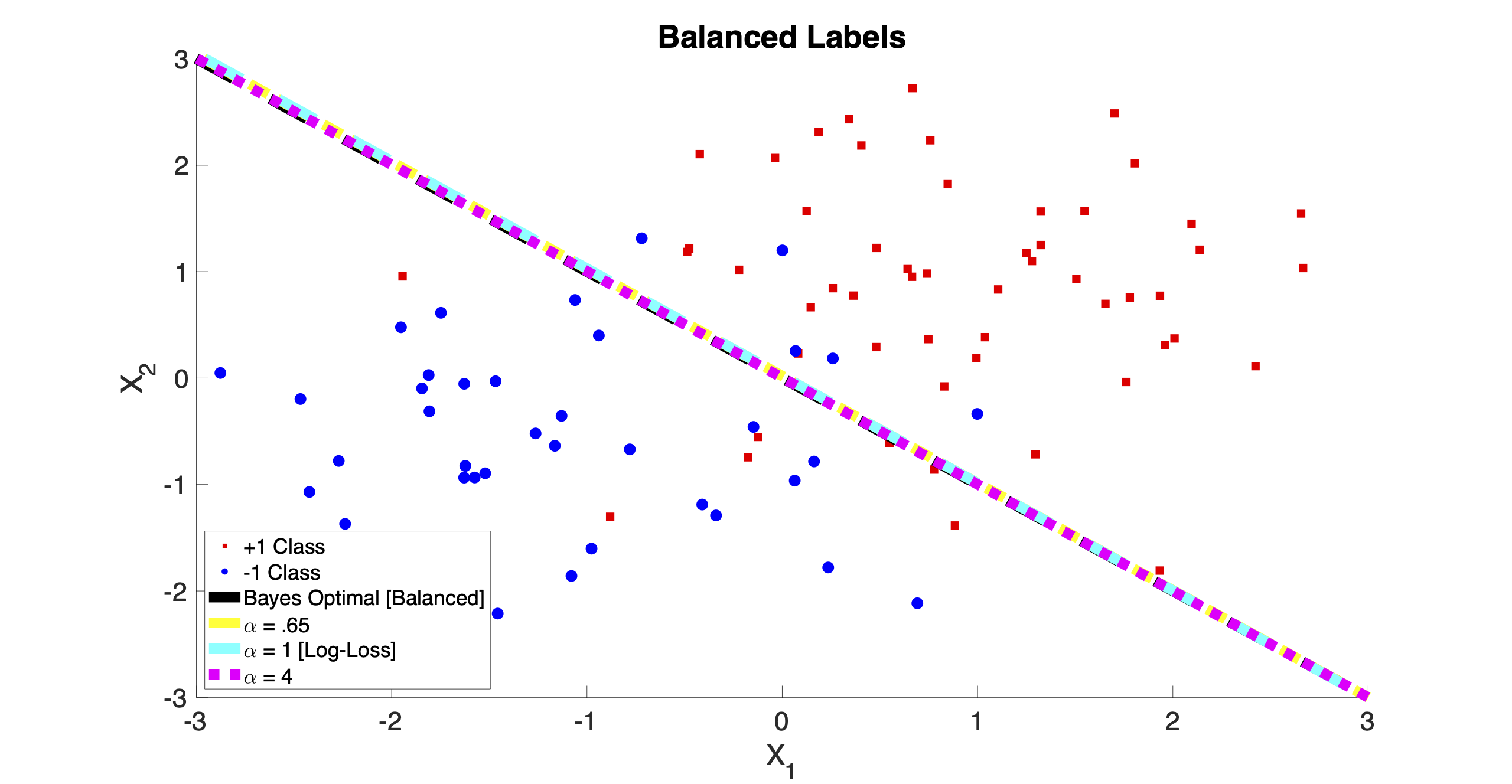}}
    \caption{A synthetic experiment highlighting the collapse in trained linear predictors of $\alpha$-loss for $\alpha \in \{0.65, 1, 4\}$ on clean, balanced data. Specifically, $\alpha$-loss is trained until convergence under the logistic model for a 2D-GMM with mixing probability $\mathbb{P}[Y=-1] = \mathbb{P}[Y=+1]$, symmetric means $\mu_{X|Y=-1} =[-1, -1] = -\mu_{X|Y=1}$, and shared covariance matrix $\Sigma = \mathbb{I}_{2}$. Averaged linear predictors generated by training of $\alpha$-loss averaged over 100 runs. Training data present in the figure is obtained from the last run in the experiment.}
    \label{fig:planeplotsclean}
\end{figure}

\begin{table}[]
\begin{center}
\begin{tabular}{cccccccccccccc}
             &          &     &            &           &           &   $\leftarrow$    &  $\alpha$'s   &   $\rightarrow$   &    &   &           & \\
             &          &     & .4         &  .5       & .65       & .8   & 1 & 2.5  & 4  & 8 & $10^{10}$ & $\infty$     \\ \clineB{2-13}{2}
             & 1        &     & 72.73      & 72.36     & 72.57    & 71.81 &  71.79  & 72.46    & 73.14   & 73.71     &  74.10 & 74.10        \\ \cline{2-13} 
             & 2        &     & 79.54      & 79.55     & 78.51         & 77.81 & 76.87   & 74.13    & 74.59   &  75.32    & 75.71 & 75.71   \\ \cline{2-13} 
            & 5        &     & 84.22      & 83.77     & 83.48         & 82.78 &  82.24  & 80.68    & 80.30   & 80.13     & 79.71 & 79.71   \\ \cline{2-13} 
$\uparrow$  & 10       &     & 87.86      & 87.54     & 87.55       & 87.30 & 87.09   &  85.59   &  85.36  & 85.08     &  84.99 & 84.99  \\ \cline{2-13} 
   Imb \%    & 15       &     & 89.01      & 88.98     & 88.74      & 88.66 & 88.63   &  88.32   & 88.09   &  88.14    &  87.97 & 87.97 \\ \cline{2-13}
$\downarrow$ & 20       &     & 90.09      & 90.11     & 89.96       & 89.88 & 89.79   & 89.61    &  89.59  &  89.73    & 89.60 & 89.60   \\ \cline{2-13}   
             & 30       &     & 91.55      & 91.36     & 91.30       & 91.27 & 91.24   & 91.16    & 91.10   & 90.90      & 90.75 & 90.75    \\ \cline{2-13}
             & 40       &     & 92.00      & 91.97     & 91.98  & 91.97 & 91.98 & 92.05 & 92.07   & 92.08   & 92.08 & 92.08      \\ \cline{2-13}
             & 50       &     & 92.08 & 92.09 & 92.08 & 92.08 & 92.08 & 92.08 & 92.07 & 92.06 & 92.06 & 92.06 \\ \clineB{2-13}{2}
\end{tabular}
\caption{Further quantitative results associated with Figure~\ref{fig:planeplots}(a) in Section~\ref{sec:syntheticexp} with exactly the same experimental setup. 
Values reported in the table are the test accuracy (in \%) of a linear predictive model tested on 1 million examples of clean, balanced synthetic test data. 
The linear model was learned by averaging models for 100 training examples over 100 runs. 
Such models were learned for different imbalance levels of the training data as shown in the table.
We found that the Bayes accuracy of this experiment was $92.14\%$.
In general, we find that $\alpha^{*} < 1$, which aligns with our theoretical intuition. 
This contrasts with the notable exception of $1\%$ imbalance, where $\alpha^* > 1$, which points towards the usefulness of \textit{class upweighting} in addition to employing $\alpha$-loss for such a highly imbalanced class.
Also of note, we find that smaller $\alpha$ is not always better (see <5\% imbalance), which hints at a trade-off between emphasizing the imbalanced class and computational infeasibility (e.g., exploding gradients) as discussed after Proposition~\ref{prop:SLQCpriortoevolution}.
Lastly, we note the closeness between $\alpha = 8$ and $10^{10}$ and $\infty$; this follows our theoretical intuition derived from the \textit{saturation effect} of $\alpha$-loss as depicted in~\eqref{eq:saturationeffect}.
}
\label{table:syntheticclassimbalanceaccuracy}
\end{center}
\end{table}

\begin{table}[]
\begin{center}
\begin{tabular}{cccccccccccccc}
             &          &     &            &           &           &   $\leftarrow$    &  $\alpha$'s   &   $\rightarrow$   &    &   &           & \\
             &          &     & .4         &  .5       & .65      & .8   & 1 & 2.5  & 4  & 8 & $10^{10}$ & $\infty$     \\ \clineB{2-13}{2}
             & 1        &     & 0.6261      & 0.6192     & 0.6231   & 0.6084 & 0.6081  & 0.6209 & 0.6338   & 0.6445    & 0.6517 & 0.6517  \\ \cline{2-13} 
             & 2        &     & 0.7446      & 0.7448     & 0.7280    & 0.7165 & 0.7007 & 0.6524 & 0.6607  & 0.6739 & 0.6807
 & 0.6807   \\ \cline{2-13} 
             & 5        &     & 0.8146  & 0.8083     & 0.8040    & 0.7938 & 0.7857  & 0.7619   & 0.7560   & 0.7534   & 0.7467
 & 0.7467  \\ \cline{2-13} 
$\uparrow$   & 10       &     & 0.8648    & 0.8605     & 0.8606    & 0.8573 & 0.8545   &  0.8341   &  0.8309  & 0.8270    & 0.8257 & 0.8257  \\ \cline{2-13} 
   Imb \%  & 15       &     & 0.8800     & 0.8797     & 0.8765    & 0.8755 & 0.8751  & 0.8710 & 0.8680 & 0.8687 & 0.8665 & 0.8665 \\ \cline{2-13}
$\downarrow$ & 20       &     & 0.8937   & 0.8940  & 0.8920 & 0.8910 & 0.8899 & 0.8876 & 0.8872 & 0.8892 & 0.8875 & 0.8875 \\ \cline{2-13}   
             & 30       &     & 0.9124      & 0.9100     & 0.9092 & 0.9089 & 0.9084  & 0.9074 & 0.9066 & 0.9040 & 0.9021 & 0.9021 \\ \cline{2-13}
             & 40       &     & 0.9187      & 0.9183 & 0.9184   & 0.9183 & 0.9183 & 0.9195 & 0.9199 & 0.9200    & 0.9201 & 0.9201 \\ \cline{2-13}
             & 50       &     & 0.9207  & 0.9207 & 0.9207  & 0.9208 & 0.9208 & 0.9208   & 0.9207 & 0.9206 & 0.9205 & 0.9205 \\ \clineB{2-13}{2}
\end{tabular}
\caption{A twin table of Table~\ref{table:syntheticclassimbalanceaccuracy}, except with $\text{F}_{1}$ scores reported. For a brief review of the $\text{F}_{1}$ score, see Appendix~\ref{sec:f1review}. Gains of $\alpha^{*} < 1$ over log-loss ($\alpha = 1$) are more exaggerated by the $\text{F}_{1}$ score, in particular see $2\%$ and $5\%$ imbalance.} 
\label{table:syntheticimbalancef1}
\end{center}
\end{table}

\begin{table}[]
\begin{center}
\begin{tabular}{cccccccccccccc}
             &          &     &            &           &           &   $\leftarrow$    &  $\alpha$'s   &   $\rightarrow$   &    &   &           & \\
             &          &     & .4         &  .5       & .65      & .8   & 1 & 2.5  & 4  & 8 & $10^{10}$ & $\infty$     \\ \clineB{2-13}{2}
             & 1        &     & 92.18      & 92.17     & 92.16    & 92.17 & 92.17  & 92.18    & 92.16   & 92.13    & 92.12 & 92.12  \\ \cline{2-13} 
             & 2        &     & 92.06      & 92.07     & 92.08    & 92.09 & 92.11 & 92.14 & 92.14   &  92.14 & 92.15 & 92.15   \\ \cline{2-13} 
             & 5        &     & 91.34  & 91.41     & 91.61    & 91.68 &  91.85  & 92.11    & 92.12   & 92.13    & 92.13 & 92.13   \\ \cline{2-13} 
$\uparrow$   & 10       &     & 90.41      & 90.34     & 90.53    & 90.89 & 91.29   &  92.01   &  92.04  & 92.05    & 92.06 & 92.06  \\ \cline{2-13} 
   Noise \%  & 15       &     & 88.45      & 88.72     & 89.03    & 89.53 & 90.14   & 91.95 & 92.02 & 92.02 & 92.03
 & 92.03 \\ \cline{2-13}
$\downarrow$ & 20       &     & 87.84      & 86.21     & 86.52    & 87.38 & 88.85   & 91.17    & 91.53  &  91.91   & 91.46 & 91.54   \\ \cline{2-13}   
             & 30       &     & 80.43      & 80.34     & 81.48    & 82.36 & 83.55   & 90.15    & 90.68   & 90.86 & 90.98 & 90.98 \\ \cline{2-13}
             & 40       &     & 75.02      & 75.20     & 75.11    & 75.38 & 75.89   & 83.00    & 84.51   & 85.59    & 85.82 & 85.82 \\ \cline{2-13}
             & 50       &     & 67.66      & 67.45     & 67.26    & 67.22 & 67.08   & 70.61    & 73.33   & 75.67    & 76.89 & 76.89   \\ \clineB{2-13}{2}
\end{tabular}
\caption{Further quantitative results associated with Figure~\ref{fig:planeplots}(b) in Section~\ref{sec:syntheticexp} with exactly the same experimental setup (training data with label noise). Values reported in the table are percent accuracy of averaged linear predictors, which were trained on noisy data, on 1 million examples of clean, balanced synthetic test data.
Similarly as in Table~\ref{table:syntheticclassimbalanceaccuracy}, we observe a saturation effect.
Further, note that $\alpha = \infty$ does not always outperform the smaller $\alpha$'s, in particular, see $20\%$ noise where $\alpha^{*} = 8$. 
This hints at a trade-off between $\alpha$ and computational feasibility in the large $\alpha$ regime ($\alpha 
> 1$), which also follows from our theoretical intuition as stated at the end of Section~\ref{sec:landscapelogisticmodel}.}
\label{table:syntheticnoisyaccuracy}
\end{center}
\end{table}

\subsubsection{Multiclass Symmetric Label Flip Experiments} \label{appen:multiclassexperiments}
In this section, we present multiclass symmetric noisy label experiments for the MNIST and FMNIST datasets.
Our goal is to evaluate the robustness of $\alpha$-loss over log-loss ($\alpha = 1$) to symmetric noisy labels in the training data.
We generate symmetric noisy labels in the multiclass training data as follows: 
\begin{enumerate}
    \item For each run of an experiment, we randomly select 0-40\% of the training data in increments of 10\%. 
    \item For each training sample in the selected group, we remove the true underlying label number from a list of the ten classes, then we roll a fair nine-sided die over the nine remaining classes in the list; once we have a new label, we replace the true label with the new randomly drawn label.
\end{enumerate}
Note that the test data is clean, i.e., we do not flip the labels of the test dataset. Thus, we consider the canonical scenario where the labels of the training data have been flipped, but the test data is clean.

The results of the multiclass symmetric noisy label experiments are presented in Tables~\ref{table:MNISTsymmetricnoisylabelsALL}~and~\ref{table:FMNISTsymmetricnoisylabelsALL}. 
Note that we use the same fixed learning rates as the binary symmetric noisy label experiments in Section~\ref{sec:experimentsbinarynoise}.
For the MNIST and FMNIST datasets with label flips, we find very strong gains in the test accuracy, which continue to improve as the percentage of label flips increases, through training $\alpha$-loss for $\alpha>1$ over log-loss ($\alpha = 1$). 
Once label flips are present in these two datasets, $\alpha^{*} = 7$ or $8$ for the CNN 2$+$2 architecture.

%
%

\begin{table}[H]
\begin{center}
\begin{tabular}{llccccc}
\hlineB{2}
\multicolumn{1}{c}{Dataset}                & \multicolumn{1}{c}{Architecture} & Label Flip \% & LL Acc & $\alpha$* Acc & $\alpha$*         & Rel Gain \% \\ \hlineB{2}
                                           &                                  & 0             & 99.16 & 99.16     & 0.9,0.99,1,1.1 & 0.00       \\ \cline{3-7} 
                                           &                                  & 10            & 94.15 & 99.00       & 8              & 5.15       \\ \cline{3-7} 
\multicolumn{1}{c}{\multirow{1}{*}{MNIST}} & \multicolumn{1}{c}{CNN 2$+$2}      & 20            & 85.90  & 98.84     & 8              & 15.06      \\ \cline{3-7} 
\multicolumn{1}{c}{}                       &                                  & 30            & 73.54 & 98.52     & 8              & 33.97      \\ \cline{3-7} 
                                           & \multicolumn{1}{c}{}             & 40            & 60.99 & 97.96     & 8              & 60.62      \\ \hlineB{2}
\end{tabular}
\caption{Multiclass symmetric noisy label experiment on \textbf{MNIST}.} 
\label{table:MNISTsymmetricnoisylabelsALL}
\end{center}
\end{table}

\begin{table}[H]
\begin{center}
\begin{tabular}{ccccccc}
\hlineB{2}
Dataset                  & Architecture & Label Flip \% & LL Acc & $\alpha$* Acc                  & $\alpha$* & Rel Gain \% \\ \hlineB{2}
                         &              & 0             & 90.45 & 90.45 & 1,1.1  & 0.00       \\ \cline{3-7} 
                         &              & 10            & 84.69 & 89.81                      & 8      & 6.05       \\ \cline{3-7} 
                         & CNN 2$+$2      & 20            & 77.51 & 89.27                      & 7      & 15.18      \\ \cline{3-7} 
\multirow{-3}{*}{FMNIST} &              & 30            & 67.94 & 88.10                       & 7      & 29.67      \\ \cline{3-7} 
                         &              & 40            & 68.28 & 88.20                      & 8      & 28.91      \\ \hlineB{2}
\end{tabular}
\caption{Multiclass symmetric noisy label experiment on \textbf{FMNIST}.}
\label{table:FMNISTsymmetricnoisylabelsALL}
\end{center}
\end{table}

\end{document}